\pdfoutput=1
\documentclass[10pt,letterpaper,twoside,twocolumn,journal,final]{IEEEtran}

\usepackage{amsmath,amsfonts}
\usepackage{algorithm}
\usepackage{array}
\usepackage[caption=false,justification=raggedright]{subfig}
\usepackage{textcomp}
\usepackage{stfloats}
\usepackage{url}
\usepackage{verbatim}
\usepackage{graphicx}
\usepackage{cite}
\usepackage{bm}
\usepackage{mathtools}
\usepackage{url}
\usepackage{multirow}
\usepackage{colortab}
\usepackage{colortbl}
\usepackage{arydshln}
\usepackage{algpseudocode}
\usepackage{tablefootnote}
\usepackage{amsthm}
\usepackage{amssymb}
\usepackage{placeins}
\definecolor{g}{rgb}{0.925, 0.957, 0.831}
\definecolor{p}{rgb}{0.980,0.910,0.922}
\newcommand{\bhline}[1]{\noalign{\hrule height #1}}
\usepackage{tikz}
\usetikzlibrary{positioning}
\usetikzlibrary{calc}
\usetikzlibrary{arrows.meta}
\usetikzlibrary{arrows.meta}
\usepackage{pgfplots}
\pgfplotsset{compat=1.18}

\usepgfplotslibrary{fillbetween,groupplots}
\newlength{\syntheticsurfacefullwidth}
\newsavebox{\syntheticsurfaceplotbox}
\newcount\syntheticrastriginplotform
\newcount\syntheticrosenbrockplotform
\newcount\syntheticcrossplotform
\newcount\syntheticstyblinskitangplotform

\newcommand{\syntheticrastriginplot}{\pdfrefxform\syntheticrastriginplotform}
\newcommand{\syntheticrosenbrockplot}{\pdfrefxform\syntheticrosenbrockplotform}
\newcommand{\syntheticcrossplot}{\pdfrefxform\syntheticcrossplotform}
\newcommand{\syntheticstyblinskitangplot}{\pdfrefxform\syntheticstyblinskitangplotform}

\newcommand{\storesyntheticsurfaceplot}[1]{
  \immediate\pdfxform\syntheticsurfaceplotbox
  \global#1=\pdflastxform
  \global\setbox\syntheticsurfaceplotbox=\hbox{}
}

\newif\ifsyntheticsurfaceplotsready
\syntheticsurfaceplotsreadyfalse

\newcommand{\preparesyntheticsurfaceplots}{
  \ifsyntheticsurfaceplotsready
  \else
    \global\setbox\syntheticsurfaceplotbox=\hbox{
      \begin{tikzpicture}
        \begin{axis}[
            width=0.22\syntheticsurfacefullwidth,
            height=0.165\syntheticsurfacefullwidth,
            xlabel={$x_1$}, ylabel={$x_2$}, zlabel={$f_1(x_1,x_2)$},
            view={60}{30},
            colormap/viridis,
            samples=30, samples y=30,
          ]
          \addplot3[
            surf,
            domain=-5.12:5.12,
            domain y=-5.12:5.12,
          ]
          {20 + x^2 - 10*cos(360*x) + y^2 - 10*cos(360*y)};
        \end{axis}
      \end{tikzpicture}
    }
    \storesyntheticsurfaceplot{\syntheticrastriginplotform}
    \global\setbox\syntheticsurfaceplotbox=\hbox{
      \begin{tikzpicture}[
          declare function={
            rosenbrock(\x,\y) = (1-\x)^2 + 100*(\y - \x^2)^2;
          }
        ]
        \begin{axis}[
            width=0.22\syntheticsurfacefullwidth,
            height=0.165\syntheticsurfacefullwidth,
            xlabel={$x_1$}, ylabel={$x_2$}, zlabel={$f_2(x_1,x_2)$},
            view={60}{30},
            colormap/viridis,
            xmin=-2, xmax=2,
            ymin=-2, ymax=2,
            zmode=log,
            samples=30, samples y=30,
          ]
          \addplot3[
            surf,
            domain=-2:2,
            domain y=-2:2,
            restrict z to domain=0.01:1e6,
          ]
          {rosenbrock(x,y)};
        \end{axis}
      \end{tikzpicture}
    }
    \storesyntheticsurfaceplot{\syntheticrosenbrockplotform}
    \global\setbox\syntheticsurfaceplotbox=\hbox{
      \begin{tikzpicture}[
          declare function={
            crossfunc(\x,\y) = max(
              exp(-10*\x*\x),
              exp(-50*\y*\y),
              1.25*exp(-5*(\x*\x + \y*\y))
            );
          }
        ]
        \begin{axis}[
            width=0.22\syntheticsurfacefullwidth,
            height=0.165\syntheticsurfacefullwidth,
            xlabel={$x_1$}, ylabel={$x_2$}, zlabel={$f_3(x_1,x_2)$},
            view={60}{30},
            colormap/viridis,
            xmin=-1, xmax=1,
            ymin=-1, ymax=1,
            samples=30, samples y=30,
          ]
          \addplot3[
            surf,
            domain=-1:1,
            domain y=-1:1,
          ]
          {crossfunc(x,y)};
        \end{axis}
      \end{tikzpicture}
    }
    \storesyntheticsurfaceplot{\syntheticcrossplotform}
    \global\setbox\syntheticsurfaceplotbox=\hbox{
      \begin{tikzpicture}[
          declare function={
            stybtang(\x,\y) = 0.5*((\x^4 - 16*\x^2 + 5*\x) + (\y^4 - 16*\y^2 + 5*\y));
          }
        ]
        \begin{axis}[
            width=0.22\syntheticsurfacefullwidth,
            height=0.165\syntheticsurfacefullwidth,
            xlabel={$x_1$}, ylabel={$x_2$}, zlabel={$f_4(x_1,x_2)$},
            view={60}{30},
            colormap/viridis,
            xmin=-5, xmax=5,
            ymin=-5, ymax=5,
            samples=30, samples y=30,
          ]
          \addplot3[
            surf,
            domain=-5:5,
            domain y=-5:5,
          ]
          {stybtang(x,y)};
        \end{axis}
      \end{tikzpicture}
    }
    \storesyntheticsurfaceplot{\syntheticstyblinskitangplotform}
    \global\syntheticsurfaceplotsreadytrue
  \fi
}

\definecolor{myblue}{RGB}{31,119,180}
\definecolor{myorange}{RGB}{255,127,14}

\usepackage[libertine]{newtxmath}

\newtheorem{theorem}{Theorem}
\newtheorem{lemma}{Lemma}

\newtheorem{assumption}{Assumption}

\usepackage{hyperref}
\usepackage{xcolor}
\hypersetup{
  colorlinks=true,
  linkcolor=[HTML]{B2001D},
  citecolor=[HTML]{27408B},
  urlcolor=[HTML]{27408B}
}
\usepackage{orcidlink}

\makeatletter
\let\arxiv@maketitlecmd\maketitle
\let\arxiv@maketitleinternal\@maketitle
\let\arxiv@thankscmd\thanks
\let\arxiv@maincite\cite
\newcommand{\supplementtableofcontents}{%
  \section*{Contents}%
  \@starttoc{stoc}%
}
\makeatother

\begin{document}
\title{Kolmogorov-Arnold Classifier Systems\\ as Universal Approximators}
\author{Hiroki Shiraishi$^{\orcidlink{0000-0001-8730-1276}}$,
  Hisao Ishibuchi$^{\orcidlink{0000-0001-9186-6472}}$,~\IEEEmembership{Fellow,~IEEE},
  and Masaya Nakata$^{\orcidlink{0000-0003-3428-7890}}$,~\IEEEmembership{Member,~IEEE}
  \thanks{
    Manuscript received 31 March 2026; revised 6 August 2026; accepted 11 September 2026.
  This work was supported by
  JSPS KAKENHI (Grant Nos. JP23KJ0993, JP25K03195),
   National Natural Science Foundation of China (Grant No. 62376115), and Guangdong Provincial Key Laboratory (Grant No. 2020B121201001). (\textit{Corresponding authors: Hisao Ishibuchi; Masaya Nakata.})}
  \thanks{Hiroki Shiraishi was with the Graduate School of Engineering Science, Yokohama National University, Yokohama 240-8501, Japan, at the time of manuscript submission on 31 March 2026. He is now with the Digital Healthcare Research Department, Hitachi, Ltd., Tokyo 185-8601, Japan (e-mail: hiroki.shiraishi.we@hitachi.com).}
  \thanks{Hisao Ishibuchi is with the Department of Computer Science and Engineering, Southern University of Science and Technology, Shenzhen 518055, China (e-mail: hisao@sustech.edu.cn).}
  \thanks{Masaya Nakata is with the Faculty of Engineering, Yokohama National University, Yokohama 240-8501, Japan (e-mail: nakata-masaya-tb@ynu.ac.jp).}
}

\markboth{IEEE TRANSACTIONS ON EVOLUTIONARY COMPUTATION, DOI: 10.1109/TEVC.2026.3736664, SEPTEMBER 2026}{
  SHIRAISHI \MakeLowercase{{\em et al.}}:~Kolmogorov-Arnold Classifier Systems as Universal Approximators
}
\maketitle

\begin{abstract}

  As the input dimension $n$ grows, rule-based machine learning, such as Learning Classifier Systems (LCSs), faces a fundamental scalability bottleneck for function approximation: both rule count and parameter count grow exponentially with $n$. {Traditional} LCSs partition the $n$-dimensional input space directly, requiring $\mathcal{O}(m^n)$ rules for adequate coverage, where $m$ is the per-variable resolution.\label{r3-2-1} This article breaks from this paradigm by reorganizing rules dimension-wise, guided by the Kolmogorov-Arnold representation theorem: any continuous $n$-dimensional function can be expressed as a finite superposition of one-dimensional functions. The proposed Kolmogorov-Arnold Classifier System (KACS) decomposes the target function into one-dimensional subproblems and assigns a dedicated ruleset to each, reducing the worst-case rule count from $\mathcal{O}(m^n)$ to $\mathcal{O}(mn^2)$ and replacing $n$-dimensional local models with one-dimensional models requiring only two parameters per rule, independent of $n$. We also provide the first constructive proof that an LCS, namely KACS, is a universal approximator for continuous functions on compact domains. Evaluated against a direct $n$-dimensional input space partitioning approach under otherwise identical conditions, KACS {achieves competitive accuracy in many settings} while using only 2\% to 40\% of the parameters.\label{r1-1-1} Our implementation is available at \url{https://github.com/YNU-NakataLab/KACS}.

\end{abstract}

\begin{IEEEkeywords}
  Learning classifier systems, evolutionary rule-based machine learning, Kolmogorov-Arnold representation theorem, universal approximation, dimension-wise decomposition
\end{IEEEkeywords}

\section{Introduction}
\label{sec:introduction}

\IEEEPARstart{F}{unction} approximation, also known as regression, is the problem of constructing an approximation $\hat{f}: \mathcal{X} \to \mathbb{R}$ from a set of training samples $\mathcal{D}_\mathrm{tr} = \{(\mathbf{x}_i, y_i)\}_{i=1}^{N_\mathrm{tr}}$, where $\mathbf{x}_i \in \mathcal{X} \subseteq \mathbb{R}^n$ and $y_i \in \mathbb{R}$. The goal is to {approximate an unknown target function $f$ such that $\hat{f}(\mathbf{x}) \approx f(\mathbf{x})$ for all $\mathbf{x}\in\mathcal{X}$.}\label{others-11} Ideally, $\hat{f}$ generalizes well to unseen inputs~\cite{bishop2006pattern,hastie2009elements}. Reliable function approximation is a core requirement in many real-world applications~\cite{goodfellow2016deep}, where a model supports not only data fitting but also decision making, monitoring, and control~\cite{sutton2018reinforcement,aastrom1995adaptive}. In such settings, approximation accuracy, robustness, and multi-dimensional scalability directly affect overall performance and safety~\cite{goodfellow2016deep}.

From an algorithmic viewpoint, function approximation becomes challenging when the target relationship $y=f(\mathbf{x})$ is complex. In particular, two difficulties are frequently observed:
\begin{itemize}
  \item \emph{Heterogeneity}: Different regions of the input space may exhibit different local behaviors, such as sharp changes, discontinuities, or region-specific nonlinearities \cite{vijayakumar2005incremental,jacobs1991adaptive}. In such cases, a single global model can be difficult to train and may require high complexity to represent all local patterns simultaneously~\cite{nguyen2018practical}.
  \item \emph{Multiple input variables}: Approximation becomes harder as the number of input variables increases (i.e., $n\geq 2$) \cite{bishop2006pattern}. Even when $f$ is continuous, multi-dimensional spaces require many samples for adequate coverage \cite{hinrichs2010curse}, and the model class must scale without an explosive increase in parameters \cite{hastie2009elements}.
\end{itemize}

To address heterogeneity, local modeling and divide-and-conquer strategies have been widely studied (e.g., regression trees~\cite{breiman2017classification}, Takagi-Sugeno fuzzy systems~\cite{takagi1985fuzzy}, and mixture-of-experts models~\cite{jacobs1991adaptive}). Among these approaches, Learning Classifier Systems (LCSs)~\cite{urbanowicz2017introduction} are a representative family. LCSs are an evolutionary rule-based machine learning paradigm that represents a model as a population of IF-THEN rules~\cite{siddique2024survey}. Each rule covers a local region (IF-part or \emph{antecedent}) and provides a local classification or approximation model (THEN-part or \emph{consequent}); the final output is obtained by aggregating the predictions of matching rules. Because of this rule-based, adaptive local modeling nature, LCSs have been used as automated tools for designing classifiers and approximators \cite{shiraishi2025xkan, preen2021autoencoding,dam2007neural}. For function approximation, to the best of our knowledge, all LCSs directly handle rules defined over the $n$-dimensional input space. For example, XCSF~\cite{wilson2002classifiers}, the most widely studied Michigan-style LCS, combines evolutionary search for rule-antecedents with gradient-based online updates of rule-consequents. SupRB~\cite{heider2023suprb}, a recently proposed Pittsburgh-style LCS, alternates between an evolutionary rule discovery phase and a batch supervised fitting of rule-consequents. Despite these architectural differences, both systems represent rules over the full $n$-dimensional input space.

While LCSs are effective in heterogeneous problems due to their local modeling nature \cite{urbanowicz2017introduction}, they still face a major scalability issue in multi-dimensional settings~\cite{debie2019implications,urbanowicz2009learning,heider2022metaheuristic,drugowitsch2008design}.\label{r2-1} From a partitioning perspective, if each of the $n$ input variables requires $m$ effective subdivisions to achieve good accuracy, the number of local regions (and thus rules) can grow as $\mathcal{O}(m^n)$ \cite{debie2019implications}. This growth makes both rule discovery and parameter learning increasingly costly \cite{stalph2012resource,debie2019implications}. Moreover, when the rule-consequent is parameterized in an $n$-dimensional form (e.g., linear regression), each rule must estimate multiple parameters, which can be difficult when each local region contains limited data \cite{bishop2006pattern}.

Based on these considerations, this article proposes a Kolmogorov-Arnold Classifier System (KACS), an LCS designed to address both heterogeneity and multi-dimensional scalability. KACS is guided by the Kolmogorov-Arnold (KA) representation theorem~\cite{kolmogorov1961representation,arnold2009functions}. The KA theorem states that any continuous function of $n$ variables can be represented as a finite superposition of one-dimensional functions~\cite{arnold2009functions,schmidt2021kolmogorov}:
\begin{equation}
  f(\mathbf{x}) = \sum_{q=1}^{2n+1} \Phi_q \left( \sum_{p=1}^{n} \psi_{q,p}(x_p) \right),
  \label{eq:ka_theorem}
\end{equation}
where $\psi_{q,p}$ are inner one-dimensional functions and $\Phi_q$ are outer one-dimensional functions.

KACS follows this KA structure and converts the original $n$-dimensional approximation task into learning a collection of one-dimensional functions: the inner functions $\{\psi_{q,p}\}$ and the outer functions $\{\Phi_q\}$. Instead of evolving an $n$-dimensional ruleset, KACS focuses on learning one-dimensional rulesets for these functions, which are then combined according to \eqref{eq:ka_theorem}. Each one-dimensional ruleset is used to represent an inner or outer function. Thus, the total number of required rulesets is $n(2n +1) + (2n + 1) = 2n^2 + 3n + 1$. If each ruleset has $m$ rules, the total number of rules is $2mn^2 + 3mn + m$. This design has two advantages. First, under a fixed per-variable resolution (i.e., roughly $m$ effective subdivisions per variable), the number of required rules is reduced from $\mathcal{O}(m^n)$ to $\mathcal{O}(mn^2)$ because KACS allocates rules to one-dimensional intervals rather than to $n$-dimensional regions (cf. Fig.~\ref{fig:xcsf_vs_kacs}) {\footnote{{Here, $m$ is an analytical measure of the effective resolution required by a problem, not a parameter set directly in XCSF or KACS. In practice, this resolution emerges from evolved rule coverage under the selected population budget. A larger effective $m$ can capture finer variation but increases the rule count exponentially in $n$ for XCSF and only linearly in $m$ for KACS.}}}.\label{r3-5} Second, KACS replaces $n$-dimensional rule-consequent models with one-dimensional models, reducing the number of parameters per local model and making parameter estimation more reliable from the same training data~\cite{bishop2006pattern}. Beyond these practical advantages, this article establishes the theoretical expressive power of KACS by proving that it is a universal approximator for continuous functions on compact domains (cf. Theorem \ref{thm:uat})~\cite{cybenko1989approximation}. Specifically, we prove that for any continuous target function $f$ and any desired accuracy level, there always exists a finite KACS model $\hat{f}$ that can approximate $f$ within that accuracy.

Our contributions are as follows:
\begin{enumerate}
  \item We propose KACS, a new paradigm for rule-based machine learning in which rules are organized dimension-wise rather than by partitioning only the input space (cf. Section~\ref{sec:kacs}). Unlike conventional LCSs (e.g., XCSF) that perform divide-and-conquer solely in the $n$-dimensional input space, KACS decomposes the problem into one-dimensional subproblems via the KA theorem, allocates dedicated rulesets to each subproblem, and composes them as a finite superposition. This reduces the worst-case rule count from $\mathcal{O}(m^n)$ to $\mathcal{O}(mn^2)$.\label{others-1}

  \item We prove that KACS is a universal approximator for continuous functions on compact domains.\label{others-2} To our knowledge, this is the first universal approximation proof for any LCS (cf. Section~\ref{sec:uat}).

  \item We empirically validate the proposed paradigm by holding all design choices constant except rule dimensionality. Specifically, we benchmark KACS against XCSF on artificial and real-world datasets (cf. Section~\ref{sec:experiments}). This controlled comparison {evaluates KA-based rule organization under otherwise comparable LCS settings,}\label{r1-2-1} and the results show consistent reductions in both approximation error and model complexity.
\end{enumerate}

{It should be noted that compactness and transparency are important practical properties of LCSs~\cite{heider2023suprb}. Predictive accuracy is not the sole objective of regression; when the target function is known, fidelity to its underlying structure may also be important. This article focuses on local modeling and scalability rather than on a comparative evaluation of global interpretability. Although individual KACS rules are one-dimensional, we do not claim that the composed KACS model is more globally interpretable than XCSF solely because it uses fewer rules or parameters. A systematic evaluation of global interpretability remains future work.}\label{r2-2}\label{r3-1-1}

The rest of this article is organized as follows. Section~\ref{sec:preliminaries} {describes XCSF and the KA theorem.}\label{others-3} Section~\ref{sec:related} reviews related work. Section~\ref{sec:kacs} presents KACS. Section~\ref{sec:uat} establishes the universal approximation theorem of KACS. Section~\ref{sec:experiments} reports experimental results, discussion, and analysis. Finally, Section~\ref{sec:conclusion} concludes this article.

\section{Preliminaries}
\label{sec:preliminaries}

\subsection{XCSF Classifier System}
\label{sec:xcsf}

XCSF is one of the most successful Michigan-style LCSs for function approximation. It incrementally optimizes a population $\mathcal{P}$ of IF-THEN rules online, updating and evolving rules as new samples arrive. These rules collectively partition the input space, and each rule covers a local region with a linear prediction model. We next provide a brief overview of XCSF. For further details on XCSF, see \cite{wilson2002classifiers,preen2021autoencoding}; {for complementary perspectives on LCSs, see \cite{drugowitsch2008design,heider2022metaheuristic,lanzi2025proposal}.}\label{r2-3}

\subsubsection{Rule Representation}

An $n$-dimensional XCSF rule\footnote{Due to historical background, rules in LCSs, including XCSF, are called \textit{classifiers}, though they work as approximators. To avoid this potential confusion, we use the term \textit{rules} throughout this article.} $cl_k$ is written as
\begin{equation}
  cl_k:
  \textbf{ IF } \mathbf{x}\in C_k= \prod_{i=1}^{n}\left[l_{k,i}, u_{k,i}\right]
  \textbf{ THEN } P_k(\mathbf{x}) = \mathbf{w}_k^\top \mathbf{x}',
  \label{eq:xcsf_rule}
\end{equation}
where the antecedent $C_k$ is an $n$-dimensional hyperrectangle and the consequent $P_k$ is a linear model with weight vector $\mathbf{w}_k \in \mathbb{R}^{n+1}$ and the augmented input vector $\mathbf{x}' = (1, x_1, \ldots, x_n)^\top$~\cite{wilson2002classifiers}. {Alternative antecedent representations, including hyperellipsoids, are discussed in Section~\ref{sec:related_scalability}.}\label{r3-7} Each rule $cl_k$ maintains seven bookkeeping parameters: (i) a fitness $F_k \in (0,1]$, which represents the relative usefulness of the rule; (ii) a prediction error $\epsilon_k \ge 0$, which measures the absolute prediction error of the rule; (iii) an accuracy $\kappa_k\in(0,1]$, which is computed from $\epsilon_k$ and quantifies the rule accuracy; (iv) a match set size estimate $\mathrm{ms}_k\ge 1$, which estimates the average size of the match set in which the rule appears; (v) an experience $\mathrm{exp}_k\in\mathbb{N}_0$, which counts how many updates the rule has received; (vi) a numerosity $\mathrm{num}_k\in\mathbb{N}$, which records how many copies of the rule are represented by this rule; and (vii) a time stamp $\mathrm{ts}_k\in\mathbb{N}_0$, which records when a steady-state genetic algorithm (GA) was last applied to the rule~\cite{wilson1995xcs,butz2002algorithmic}.

\subsubsection{Algorithm}

XCSF processes one training sample $(\mathbf{x}_i, y_i)$ per iteration. First, it forms the \emph{match set} $\mathcal{M} = \{cl_k \in \mathcal{P} \mid \mathbf{x}_i \in C_k\}$. \footnote{{In the LCS literature, the population and match sets are conventionally denoted by $[P]$ and $[M]$, respectively~\cite{urbanowicz2017introduction}; this article instead uses $\mathcal{P}$ and $\mathcal{M}$ to follow standard mathematical set notation.}}\label{r3-8} If $\mathcal{M}$ is empty, a \emph{covering} operator generates a new rule whose hyperrectangular antecedent is centered near $\mathbf{x}_i$ with random width controlled by a half-width parameter $r_0$; each dimension is independently set to the full domain $[0,1]$ with probability $P_\#$~\cite{nakata2020learning}.

Each rule $cl_k \in \mathcal{M}$ then computes its prediction $\hat{y}_k = \mathbf{w}_k^\top \mathbf{x}'$ and updates its weight vector using a gradient-based method (e.g., the Widrow-Hoff delta rule~\cite{widrow1960adaptive}, Adam~\cite{kingma2015adam}), based on the rule-specific error $y_i - \hat{y}_k$. We use Adam in XCSF to enable a fair comparison with KACS, which also uses Adam for weight updates. The prediction error $\epsilon_k$ is updated as a running mean of $|y_i - \hat{y}_k|$, and the accuracy $\kappa_k$ is computed from $\epsilon_k$ relative to a threshold $\epsilon_0$. The fitness $F_k$ is updated within $\mathcal{M}$ in proportion to accuracy and numerosity, so accurate rules accumulate high fitness over time~\cite{wilson1995xcs}. The system output is the fitness-weighted average $\hat{f}(\mathbf{x}_i) = \sum_{k \in \mathcal{M}} P_k(\mathbf{x}_i) F_k \,/\, \sum_{k \in \mathcal{M}} F_k$~\cite{wilson2002classifiers}.

The GA is applied to $\mathcal{M}$ when the average time since the last GA application across rules in $\mathcal{M}$ exceeds $\theta_\mathrm{GA}$~\cite{butz2002algorithmic}. Two parents (i.e., two existing rules in $\mathcal{M}$) are chosen by tournament selection; two offspring (i.e., two new rules) are created by uniform antecedent crossover and bound mutation. The created offspring are added to the current population $\mathcal{P}$ unless a parent \emph{subsumes} them: a parent subsumes an offspring when it is both more general (antecedent contains the offspring's) and sufficiently accurate and experienced~\cite{wilson1998generalization}. When the total numerosity exceeds the population size limit $N$, rules with low fitness-adjusted deletion votes are deleted by roulette-wheel selection~\cite{wilson1995xcs,butz2002algorithmic}.

\subsection{Kolmogorov-Arnold Representation Theorem}
\label{sec:ka_theorem}

The Kolmogorov-Arnold (KA) representation theorem~\cite{kolmogorov1961representation,arnold2009functions} gives a fundamental result about the structure of multivariate continuous functions.

\begin{theorem}[Kolmogorov-Arnold Representation Theorem]
  \label{thm:ka}
  Let $n\ge 2$ and $\mathcal{K} = [0,1]^n$. For any continuous function $f \in C(\mathcal{K})$, there exist continuous one-dimensional functions $\psi_{q,p} : [0,1] \to \mathbb{R}$ and $\Phi_q : \mathbb{R} \to \mathbb{R}$ such that
  \begin{equation}
    f(\mathbf{x}) = \sum_{q=1}^{2n+1}
    \Phi_q\!\left(\sum_{p=1}^{n} \psi_{q,p}(x_p)\right),
    \tag{\ref{eq:ka_theorem}}
  \end{equation}
  where $\psi_{q,p}$ are called \emph{inner functions} and $\Phi_q$ are called \emph{outer functions}.
\end{theorem}

Although Theorem~\ref{thm:ka} formally defines each outer function as $\Phi_q : \mathbb{R} \to \mathbb{R}$, in practice $\Phi_q$ is evaluated only at values taken by $z_q(\mathbf{x}) = \sum_{p=1}^{n} \psi_{q,p}(x_p)$ as $\mathbf{x}$ ranges over $\mathcal{K}$. Hence, $\Phi_q$ is only required on $\mathcal{K}^{(q)}_z \coloneqq z_q(\mathcal{K}) \subseteq \mathbb{R}$. This image is a compact interval: it is compact because the continuous image of a compact set is compact~\cite{rudin1976principles}, and it is an interval because the continuous image of the connected set $\mathcal{K}$ is connected~\cite{rudin1976principles}, and every compact connected subset of $\mathbb{R}$ is a closed bounded interval~\cite{rudin1976principles}.

Theorem~\ref{thm:ka} states that any $n$-dimensional continuous function can be represented exactly using only one-dimensional functions. Specifically, it requires $n(2n+1)$ inner functions and $2n+1$ outer functions, giving $2n^2 + 3n + 1$ one-dimensional functions in total. This is quadratic in $n$, a much more favorable scaling than the exponential growth that appears in direct $n$-dimensional partitioning~\cite{debie2019implications}.

Two remarks are worth noting for the use of this theorem in KACS. First, the domain of $f$ can always be mapped to $[0,1]^n$ by a linear rescaling without loss of generality. Second, the theorem guarantees the \emph{existence} of the representation, but does not specify a unique set of $\psi_{q,p}$ and $\Phi_q$~\cite{schmidt2021kolmogorov}; in practice, these functions need to be learned from data \cite{liu2025kan}.

\section{Related Work}
\label{sec:related}

\subsection{Multi-Dimensional Scalability in LCSs}
\label{sec:related_scalability}

The multi-dimensional scalability problem in LCSs, including XCSF~\cite{wilson2002classifiers}, is well known. To approximate a target function with $m$ effective subdivisions per variable in $n$ dimensions, XCSF needs to maintain $\mathcal{O}(m^n)$ rules in the worst case~\cite{debie2019implications}. Several directions have been explored to ease this burden.

A first direction focuses on the \emph{antecedent representation} (for a recent comprehensive review, cf. \cite{shiraishi2026adapting}). The standard hyperrectangular representation{\cite{wilson1999xcsr,wilson2000mining}} is axis-aligned, which forces the system to use many small rules when the true subproblem boundaries are tilted or curved~\cite{shiraishi2022can}.\label{r2-4} Replacing hyperrectangles with hyperellipsoids via kernel functions reduces the rule count for curved boundaries~\cite{butz2008function}. Neural network-based conditions can represent even more flexible boundaries but raise the cost of rule evaluation and evolutionary search~\cite{bull2002accuracy}. Code-fragment-based conditions, inspired by genetic programming, handle multi-dimensional problems where only a few variables are relevant~\cite{arif2017solving}, {thereby yielding shorter and more interpretable conditions. Such representations are practically effective when real-world data have low effective dimensionality. The worst case instead corresponds to a target with strong joint dependence on all variables and fine-scale variation throughout the input space; for such problems, achieving a fixed per-variable resolution can still require exponentially many local models.}\label{r3-2-2}

A second direction applies \emph{dimensionality reduction} before passing the input to an LCS. Behdad et al.~\cite{behdad2011pca} showed that combining an LCS with principal component analysis (PCA) substantially reduces the required population size while maintaining accuracy on multi-dimensional classification tasks. Several approaches~\cite{shiraishi2021increasing,yatsu2023exploring,schonberner2025dimensionality} have also been proposed to reduce the dimensionality of multi-dimensional inputs using autoencoders. A shared limitation is that preserving high accuracy in the compressed representation is difficult, and any compression error can propagate into an LCS.

A third direction uses \emph{feature selection} to reduce the effective input dimension before or during learning. Urbanowicz and Moore~\cite{urbanowicz2015exstracs} introduced ExSTraCS 2.0, which improves scalability by incorporating expert knowledge scores based on TuRF~\cite{moore2007tuning} into the covering and mutation operators so that rules are biased toward more relevant features. ExSTraCS 2.0 was motivated in part by biomedical classification settings, where a multi-dimensional input may contain many irrelevant genetic attributes with only a small subset being predictive. This approach works well when many input variables are genuinely irrelevant to the target output. However, in function approximation, variables may interact and contribute jointly to the target output~\cite{friedman2008predictive}. In such cases, feature selection cannot reduce the exponential rule count increase~\cite{bellman1966dynamic,debie2019implications}, and the multi-dimensional scalability problem remains.

{Unlike} these approaches, KACS reorganizes rules according to the KA theorem (i.e., Theorem~\ref{thm:ka}), reducing the worst-case rule count from $\mathcal{O}(m^n)$ to $\mathcal{O}(mn^2)$ without preprocessing.\label{others-4}

\subsection{Kolmogorov-Arnold Networks and Variants}
\label{sec:related_kan}

A representative recent use of the KA theorem in machine learning is the Kolmogorov-Arnold Network (KAN)~\cite{liu2025kan}. KAN brings the KA theorem into the neural network paradigm by replacing the fixed activation functions of a multi-layer perceptron (MLP) \cite{cybenko1989approximation} with learnable univariate B-spline functions on the edges (i.e., connections between neurons).

Since then, several other KA-inspired variants have been proposed. Wav-KAN~\cite{bozorgasl2024wav} replaces B-spline activation functions with wavelets to enable efficient capture of both high-frequency and low-frequency components through multiresolution analysis. Ensemble-KAN~\cite{de2024ensemble} builds multiple KANs, each using a different subset of input features, and aggregates their predictions. Chebyshev KAN~\cite{ss2024chebyshev} uses Chebyshev polynomials for improved parameter efficiency. Temporal-KAN~\cite{genet2025tkan} and Federated-KAN~\cite{zeydan2025f} extend the framework to sequential and federated learning settings, respectively.

Within the LCS community, X-KAN~\cite{shiraishi2025xkan} exploits the divide-and-conquer nature of LCSs to optimize multiple local KAN models through evolutionary search. By placing a KAN model in the consequent of each $n$-dimensional XCSF rule, X-KAN achieves effective approximation of discontinuous functions and functions with high local complexity. These two types of functions are difficult for a single global KAN.

As explained in the brief literature review above, KAN and its variants apply the KA theorem to neural architecture design, and X-KAN uses it to design local models in rule-consequents. Our KACS in this article takes a different approach: it first applies the KA theorem to the rule organization itself within a rule-learning framework. That is, our approach in KACS and the above-mentioned approaches in existing studies operate at structurally different levels of KA application (i.e., neural architecture, rule-consequents, and rule organization).

\subsection{Universal Approximation in Rule-Based Systems}
\label{sec:related_uat}

Universal approximation has been studied across a range of model classes. For neural networks, Cybenko~\cite{cybenko1989approximation} and Hornik et al.~\cite{hornik1989multilayer} showed that a single hidden layer of sigmoidal units is sufficient to approximate any continuous function on a compact set. Wang and Mendel~\cite{wang1992fuzzy} proved an analogous result for fuzzy systems using the Stone-Weierstrass theorem~\cite{stone1948generalized}. Ying~\cite{ying1994sufficient} extended this to a broader class of fuzzy inference systems covering product, minimum, and other common t-norms. Kosko~\cite{kosko2002fuzzy} gave a separate proof for additive fuzzy systems, and Castro~\cite{castro2002fuzzy} further showed that fuzzy logic controllers with arbitrary membership functions and a broad class of inference engines are universal approximators.

A practically important question is how the required number of rules scales with dimension $n$. Standard fuzzy systems with a uniform grid still need $\mathcal{O}(m^n)$ rules, the same exponential cost as XCSF~\cite{bellman1966dynamic,debie2019implications}. Wang~\cite{wang1998universal} showed that a hierarchical fuzzy system (HFS), which consists of several low-dimensional fuzzy systems connected in sequence, is a universal approximator whose rule count grows only linearly with $n$. This result is important because it shows that the exponential rule count is not an essential requirement for universal approximation.

KACS reaches a similar conclusion via the KA decomposition rather than a manually designed hierarchy, such as HFS; it achieves $\mathcal{O}(mn^2)$ rules in a mathematically grounded two-layer structure. Furthermore, to our knowledge, no prior LCS has been proved to be a universal approximator. In contrast, this article provides the first constructive universal approximation proof for an LCS, namely KACS.

\section{Kolmogorov-Arnold Classifier System}
\label{sec:kacs}

{KACS has two main properties. First, it organizes the rule population into one-dimensional inner and outer submodels according to the KA structure. Second, it learns rule consequents through system-level backpropagation across the KA structure. The following subsections detail these characteristics.}\label{r3-12}

\usetikzlibrary{positioning, arrows.meta, backgrounds, fit}

\subsection{Overview}
\label{sec:kacs_overview}
\begin{figure*}[t]
  \centering
  \begin{tikzpicture}[
      font=\small,
      >=Stealth,
      inode/.style ={circle, draw=gray!60, fill=gray!12, thick,
      minimum size=0.62cm, inner sep=0pt},
      inbox/.style ={draw=blue!65!black, fill=blue!10, thick, rounded corners=2pt,
        minimum width=1.42cm, minimum height=0.46cm,
      inner sep=2pt, font=\scriptsize, text=blue!75!black},
      outbox/.style={draw=orange!70!black, fill=orange!15, thick, rounded corners=2pt,
        minimum width=1.42cm, minimum height=0.46cm,
      inner sep=2pt, font=\scriptsize, text=orange!82!black},
      snode/.style ={circle, draw=gray!60, fill=gray!12, thick,
      minimum size=0.52cm, inner sep=0pt, font=\scriptsize},
      garr/.style={->, thick, gray!55},
      barr/.style={->, thick, blue!55!black},
      oarr/.style={->, thick, orange!75!black},
      arr/.style ={->, thick},
    ]

    \begin{scope}

      \fill[red!3,  rounded corners=0pt] (-0.25, 0.90) rectangle (6.6, -6.6);
      \draw[red!20, rounded corners=0pt, thick] (-0.25, 0.90) rectangle (6.6, -6.6);

      \node[font=\large\bfseries, text=red!65!black]  at (3.2, 0.55) {XCSF};
      \node[font=\footnotesize\itshape, text=red!50!black]
      at (3.2, 0.10) {$n$-dimensional rules\quad$(n=2)$};

      \def\gL{1.3}\def\gB{-4.80}\def\gS{3.5}
      \fill[white] (\gL,\gB) rectangle (\gL+\gS,\gB+\gS);

      \fill[red!18,fill opacity=0.72]
      (\gL+0.03*\gS,\gB+0.07*\gS) rectangle (\gL+0.62*\gS,\gB+0.65*\gS);
      \draw[red!55!black,thick]
      (\gL+0.03*\gS,\gB+0.07*\gS) rectangle (\gL+0.62*\gS,\gB+0.65*\gS);
      \node[font=\scriptsize\bfseries,text=red!65!black]
      at (\gL+0.18*\gS,\gB+0.20*\gS) {$cl_1$};
      \fill[red!22,fill opacity=0.72]
      (\gL+0.35*\gS,\gB+0.04*\gS) rectangle (\gL+0.96*\gS,\gB+0.62*\gS);
      \draw[red!55!black,thick]
      (\gL+0.35*\gS,\gB+0.04*\gS) rectangle (\gL+0.96*\gS,\gB+0.62*\gS);
      \node[font=\scriptsize\bfseries,text=red!65!black]
      at (\gL+0.80*\gS,\gB+0.18*\gS) {$cl_2$};
      \fill[red!18,fill opacity=0.72]
      (\gL+0.05*\gS,\gB+0.40*\gS) rectangle (\gL+0.65*\gS,\gB+0.97*\gS);
      \draw[red!55!black,thick]
      (\gL+0.05*\gS,\gB+0.40*\gS) rectangle (\gL+0.65*\gS,\gB+0.97*\gS);
      \node[font=\scriptsize\bfseries,text=red!65!black]
      at (\gL+0.20*\gS,\gB+0.82*\gS) {$cl_3$};
      \fill[red!22,fill opacity=0.72]
      (\gL+0.42*\gS,\gB+0.38*\gS) rectangle (\gL+0.97*\gS,\gB+0.96*\gS);
      \draw[red!55!black,thick]
      (\gL+0.42*\gS,\gB+0.38*\gS) rectangle (\gL+0.97*\gS,\gB+0.96*\gS);
      \node[font=\scriptsize\bfseries,text=red!65!black]
      at (\gL+0.80*\gS,\gB+0.82*\gS) {$cl_4$};
      \fill[red!28,fill opacity=0.65]
      (\gL+0.23*\gS,\gB+0.23*\gS) rectangle (\gL+0.77*\gS,\gB+0.77*\gS);
      \draw[red!55!black,thick]
      (\gL+0.23*\gS,\gB+0.23*\gS) rectangle (\gL+0.77*\gS,\gB+0.77*\gS);
      \node[font=\scriptsize\bfseries,text=red!65!black]
      at (\gL+0.50*\gS,\gB+0.50*\gS) {$cl_5$};
      \draw[gray!45,thick] (\gL,\gB) rectangle (\gL+\gS,\gB+\gS);

      \draw[garr] (\gL,\gB-0.22) -- (\gL+\gS,\gB-0.22)
      node[right,font=\scriptsize,text=gray!70!black] {$x_1$};
      \draw[garr] (\gL-0.22,\gB) -- (\gL-0.22,\gB+\gS)
      node[above,font=\scriptsize,text=gray!70!black] {$x_2$};
      \node[font=\footnotesize\bfseries,text=red!60!black]
      at (\gL+0.5*\gS,\gB+\gS+0.30) {Population $\mathcal{P}$\;(2-D rules)};

      \node[inode] (xin) at (0.38,\gB+0.5*\gS) {$\mathbf{x}$};
      \node[font=\scriptsize,below=1pt of xin,text=gray!70!black] {$(x_1,x_2)$};
      \draw[barr] (xin) -- (\gL,\gB+0.5*\gS);
      \node[inode] (youtL) at (5.7,\gB+0.5*\gS) {$\hat{y}$};
      \draw[barr] (\gL+\gS,\gB+0.5*\gS) -- (youtL);

      \node[font=\scriptsize\normalshape,text=gray!65!black,align=center]
      at (3.2,-6.00)
      {Rule $cl_k:$ IF $\mathbf{x}\in[\mathbf{l},\mathbf{u}]\subseteq\mathbb{R}^{2}$
      \;THEN $P_k(\mathbf{x})=\mathbf{w}^{\!\top}\!\mathbf{x}'$};
      \node[font=\scriptsize\bfseries,text=red!65!black]
      at (3.2,-6.34) {\#Rules:\;$\mathcal{O}(m^n)$\quad \#Parameters per rule:\;$n+1=3$};

    \end{scope}

    \begin{scope}[xshift=7.3cm]

      \fill[blue!2,  rounded corners=0pt] (-0.25, 0.90) rectangle (9.8, -6.6);
      \draw[blue!20, rounded corners=0pt, thick] (-0.25, 0.90) rectangle (9.8, -6.6);

      \node[font=\large\bfseries, text=blue!65!black] at (4.8, 0.55) {KACS};
      \node[font=\footnotesize\itshape, text=blue!50!black]
      at (4.8, 0.10)
      {$1$-dimensional rules\quad$(n=2,\;2n{+}1=5$ channels$)$};

      \node[font=\footnotesize\bfseries, text=blue!65!black, anchor=west]
      at (1.8,-0.33) {Inner submodels \!$\left\{\mathcal{P}_{q,p}^\psi\right\}$};
      \node[font=\scriptsize\itshape, text=blue!52!black, anchor=west]
      at (1.8,-0.63) {domain\;$[0,1]$};
      \node[font=\scriptsize\itshape, text=blue!52!black, anchor=west]
      at (1.8,-0.92) {total\;$n(2n+1)=10$};

      \node[font=\footnotesize\bfseries, text=orange!72!black, anchor=west]
      at (4.95,-0.33) {Outer submodels \!$\left\{\mathcal{P}_{q}^\Phi\right\}$};
      \node[font=\scriptsize\itshape, text=orange!62!black, anchor=west]
      at (4.95,-0.63) {domain\;$\mathcal{K}_z^{(q)}$};
      \node[font=\scriptsize\itshape, text=orange!62!black, anchor=west]
      at (4.95,-0.92) {total\;$2n+1=5$};

      \node[inode] (x1) at (0.48,-1.83) {$x_1$};
      \node[inode] (x2) at (0.48,-5.00) {$x_2$};

      \fill[gray!7, rounded corners=3pt]  (1.60,-1.2) rectangle (6.78,-2.47);
      \draw[gray!25,rounded corners=3pt,thick] (1.60,-1.2) rectangle (6.78,-2.47);
      \fill[gray!7, rounded corners=3pt]  (1.60,-2.74) rectangle (6.78,-4.01);
      \draw[gray!25,rounded corners=3pt,thick] (1.60,-2.74) rectangle (6.78,-4.01);
      \fill[gray!7, rounded corners=3pt]  (1.60,-4.37) rectangle (6.78,-5.64);
      \draw[gray!25,rounded corners=3pt,thick] (1.60,-4.37) rectangle (6.78,-5.64);

      \node[font=\scriptsize\itshape,text=gray!55!black] at (1.92,-1.43) {$q{=}1$};
      \node[inbox]  (p11) at (3.18,-1.55) {$\mathcal{P}^{\psi}_{1,1}$};
      \node[inbox]  (p12) at (3.18,-2.11) {$\mathcal{P}^{\psi}_{1,2}$};
      \node[snode]  (z1)  at (4.72,-1.83) {$\Sigma$};
      \node[font=\scriptsize,below=-1pt of z1] {$\hat{z}_1$};
      \node[outbox] (f1)  at (5.95,-1.83) {$\mathcal{P}^{\Phi}_{1}$};

      \node[font=\scriptsize\itshape,text=gray!55!black] at (1.92,-3.37) {$q{=}2$};
      \node[inbox]  (p21) at (3.18,-3.09) {$\mathcal{P}^{\psi}_{2,1}$};
      \node[inbox]  (p22) at (3.18,-3.65) {$\mathcal{P}^{\psi}_{2,2}$};
      \node[snode]  (z2)  at (4.72,-3.37) {$\Sigma$};
      \node[font=\scriptsize,below=-1pt of z2] {$\hat{z}_2$};
      \node[outbox] (f2)  at (5.95,-3.37) {$\mathcal{P}^{\Phi}_{2}$};

      \foreach \xx in {3.18, 4.72, 5.95}{
        \node[text=gray!55!black,font=\large] at (\xx,-4.20) {$\vdots$};
      }
      \node[font=\scriptsize\itshape,text=gray!40!black] at (1.92,-4.20) {$\vdots$};

      \node[font=\scriptsize\itshape,text=gray!55!black] at (1.92,-4.95) {$q{=}5$};
      \node[inbox]  (p51) at (3.18,-4.72) {$\mathcal{P}^{\psi}_{5,1}$};
      \node[inbox]  (p52) at (3.18,-5.28) {$\mathcal{P}^{\psi}_{5,2}$};
      \node[snode]  (z5)  at (4.72,-5.00) {$\Sigma$};
      \node[font=\scriptsize,below=0pt of z5] {$\hat{z}_5$};
      \node[outbox] (f5)  at (5.95,-5.00) {$\mathcal{P}^{\Phi}_{5}$};

      \node[snode]  (ysum) at (7.90,-3.37) {$\Sigma$};
      \node[inode]  (yout) at (9.20,-3.37) {$\hat{y}$};

      \draw[barr] (x1) -- (p11);
      \draw[barr] (x1) to[out=-52, in=175] (p21);
      \draw[barr] (x1) to[out=-68, in=175] (p51);

      \draw[barr] (x2) to[out=78,  in=175] (p12);
      \draw[barr] (x2) to[out=78,  in=175] (p22);
      \draw[barr] (x2) -- (p52);

      \draw[barr] (p11) -- (z1);
      \draw[barr] (p12) -- (z1);
      \draw[barr] (p21) -- (z2);
      \draw[barr] (p22) -- (z2);
      \draw[barr] (p51) -- (z5);
      \draw[barr] (p52) -- (z5);

      \draw[oarr] (z1) -- (f1);
      \draw[oarr] (z2) -- (f2);
      \draw[oarr] (z5) -- (f5);

      \draw[oarr] (f1.east) to[out=0, in=105] (ysum);
      \draw[oarr] (f2.east) to[out=0, in=165] (ysum);
      \draw[oarr] (f5.east) to[out=0, in=-105] (ysum);

      \draw[arr] (ysum) -- (yout);

      \node[font=\scriptsize\normalshape, text=gray!65!black, align=center]
      at (4.8,-6.00)
      {Rule $cl_k$: IF $z\in[a,b]\subseteq\mathbb{R}$ \;THEN $P_k(z)=w_0+w_1 z$};
      \node[font=\scriptsize\bfseries, text=blue!65!black]
      at (4.8,-6.34)
      {\#Rules:\;$\mathcal{O}(mn^{2})$\quad \#Parameters per rule:\;$2$\;(independent of\;$n$)};

    \end{scope}

  \end{tikzpicture}
  \caption{{Conceptual} comparison of XCSF and KACS for $n=2$
    ($2n+1=5$ channels). {The left panel illustrates 2-D rule
      regions used by XCSF, where each rule covers a rectangle in the
    $(x_1,x_2)$ plane and maps} $\mathbf{x}$ directly to $\hat{y}$,
    {yielding} $\mathcal{O}(m^n)$ rules with $n+1$ consequent
    parameters per rule. {The right panel illustrates the KACS
    architecture, where prediction is decomposed into} one-dimensional
    inner submodels $\psi_{q,p}(x_p)$, intermediate sums $\hat{z}_q$,
    {and} one-dimensional outer submodels $\Phi_q(\hat{z}_q)$,
    {which are finally summed to obtain $\hat{y}$, yielding}
  $\mathcal{O}(mn^2)$ rules with two consequent parameters per rule.}
  \label{r3-4}
  \label{fig:xcsf_vs_kacs}
\end{figure*}
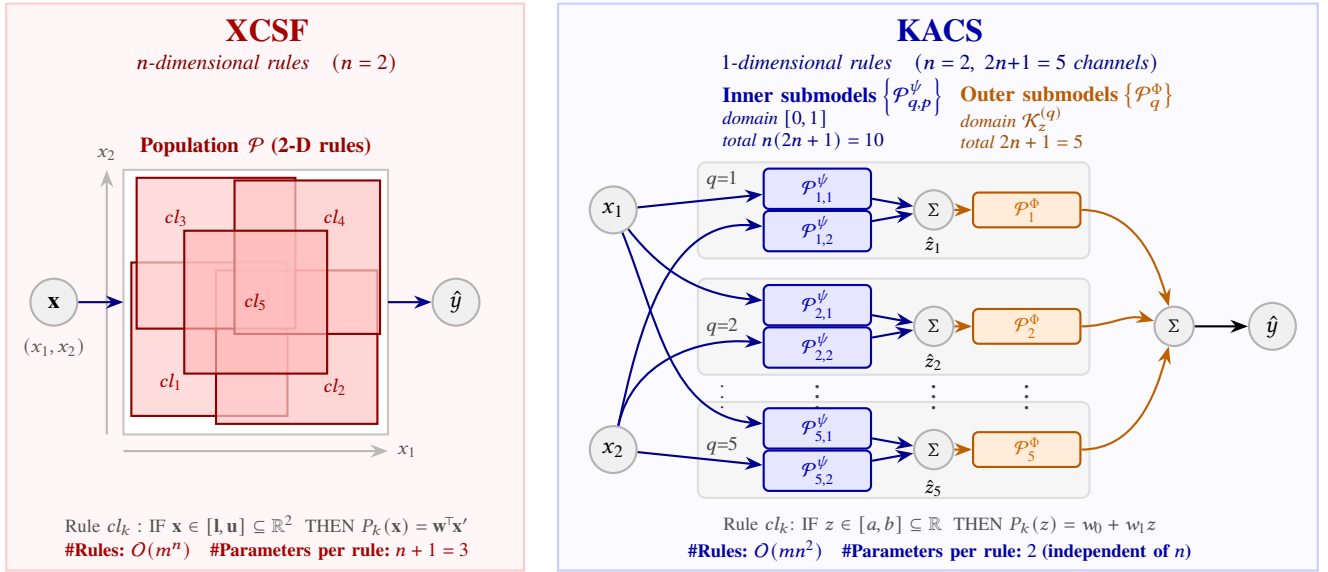

Fig.~\ref{fig:xcsf_vs_kacs} schematically illustrates XCSF and KACS. KACS is a Michigan-style LCS that inherits the structural framework of XCSF, including the covering mechanism, fitness-based prediction aggregation, GA, subsumption, and hyperparameter settings. By organizing rules via the KA decomposition (Theorem~\ref{thm:ka}), KACS not only suppresses the exponential growth in the number of rules with dimension $n$, but also reduces the difficulty of optimizing consequents by learning one-dimensional local models with only two parameters per rule, independent of $n$. Like XCSF, KACS optimizes its rules online, updating and evolving them when a new sample arrives.

{As visualized in Fig.~\ref{fig:xcsf_vs_kacs}, the KA decomposition in \eqref{eq:ka_theorem} reduces learning $f$ to learning $n(2n+1)$ inner functions $\{\psi_{q,p}: [0,1] \to \mathbb{R}\}$ and $(2n+1)$ outer functions $\{\Phi_q: \mathcal{K}^{(q)}_z \to \mathbb{R}\}$, all one-dimensional.}\label{r3-13} KACS maintains a dedicated ruleset, called a \emph{submodel}, for each one-dimensional function. All submodels live in a single population $\mathcal{P}$, partitioned into:
\begin{itemize}
  \item $n(2n+1)$ \emph{inner submodels} $\mathcal{P}^\psi_{q,p}$, each approximating $\psi_{q,p}$ using rules over the input domain $[0,1]$;
  \item $(2n+1)$ \emph{outer submodels} $\mathcal{P}^\Phi_q$, each approximating $\Phi_q$ using rules over the intermediate domain $\mathcal{K}^{(q)}_z$.
\end{itemize}
Given an input $\mathbf{x}$, KACS computes $\hat{y}$ through a two-stage feedforward pass. It then backpropagates the system-level loss through this structure to update the weights of all active rules jointly. This is a key departure from XCSF, which updates each rule independently using its own prediction error.

{Although each inner submodel processes one variable, KACS captures epistasis through the potentially nonlinear outer function $\Phi_q$ applied to $z_q=\sum_p\psi_{q,p}(x_p)$, which couples the dimensions. This is not bi-level optimization: active inner and outer consequents are jointly updated from the system-level loss, allowing gradients through $\Phi_q$ to coordinate the inner submodels. Theorem~\ref{thm:uat} guarantees representation of continuous interactions in principle, although finite capacity and imperfect optimization may limit their learning.}\label{r3-6-1} Algorithm~\ref{alg:kacs} summarizes one complete training iteration of KACS; the following subsections describe each component in detail.

\begin{algorithm}[t]
  \footnotesize
  \caption{One training iteration of KACS}
  \label{alg:kacs}
  \begin{algorithmic}[1]
    \Require Sample $(\mathbf{x}, y)$, population $\mathcal{P}$, iteration $t$
    \Ensure  Updated $\mathcal{P}$
    \Statex \texttt{Step 1: Feedforward}\hfill $\triangleright$ Section \ref{sec:kacs_feedforward}
    \For{$q = 1$ \textbf{to} $2n+1$}
    \For{$p = 1$ \textbf{to} $n$}
    \State Form $\mathcal{M}^\psi_{q,p}$
    \eqref{eq:inner_match};
    cover if empty \eqref{eq:cover_inner}
    \State Compute $\hat{\psi}_{q,p}(x_p)$
    \eqref{eq:kacs_inner}
    \EndFor
    \State Compute $\hat{z}_q$ as $\hat{z}_q \leftarrow
    \sum_{p} \hat{\psi}_{q,p}(x_p)$ \eqref{eq:kacs_zq}
    \State Form $\mathcal{M}^\Phi_q$
    \eqref{eq:outer_match};
    cover if empty \eqref{eq:cover_outer}
    \State Compute $\hat{\Phi}_q(\hat{z}_q)$
    \eqref{eq:kacs_outer}
    \EndFor
    \State Compute $\hat{y}$ as $\hat{y} \leftarrow
    \sum_{q} \hat{\Phi}_q(\hat{z}_q)$ \eqref{eq:kacs_output}
    \Statex
    \Statex \texttt{Step 2: Backpropagation}\hfill $\triangleright$ Section \ref{sec:kacs_backprop}
    \State Form $\mathcal{M}_\mathrm{act}$ as $\mathcal{M}_\mathrm{act}
    = (\bigcup_{q,p} \mathcal{M}^\psi_{q,p})
    \cup (\bigcup_{q} \mathcal{M}^\Phi_q)$ \eqref{eq:mact}
    \For{each $cl_k \in \mathcal{M}_\mathrm{act}$}
    \State Compute $\mathbf{g}_k$ via
    \eqref{eq:grad_outer} or \eqref{eq:grad_inner}
    \State Update $\mathbf{w}_k$ via Adam
    \eqref{eq:adam_m}--\eqref{eq:adam_w}
    \EndFor
    \Statex
    \Statex \texttt{Step 3: Parameter update} \hfill $\triangleright$ Section \ref{sec:kacs_params}
    \For{each $cl_k \in \mathcal{M}_\mathrm{act}$}
    \State Update $\mathrm{exp}_k$ as $\mathrm{exp}_k \leftarrow \mathrm{exp}_k + 1$
    \State Update $\epsilon_k$,
    $\kappa_k$,
    $F_k$,
    $\mathrm{ms}_k$
    within submodel $\mathcal{M}$ via \eqref{eq:kacs_err}--\eqref{eq:kacs_ms}
    \EndFor
    \Statex
    \Statex \texttt{Step 4: GA and subsumption}\hfill $\triangleright$ Section \ref{sec:kacs_ea}
    \For{each $\mathcal{M} \in
    \{\mathcal{M}^\psi_{q,p}\} \cup \{\mathcal{M}^\Phi_q\}$}
    \If{$t -
      {\sum_{cl_k \in \mathcal{M}}
      \mathrm{num}_k \cdot \mathrm{ts}_k}/
      {\sum_{cl_k \in \mathcal{M}} \mathrm{num}_k}
    > \theta_\mathrm{GA}$}
    \State Update $\mathrm{ts}_k$ as $\mathrm{ts}_k \leftarrow t$
    for $cl_k \in \mathcal{M}$
    \State Select $cl_{p_1}$, $cl_{p_2}$ by
    tournament
    \State Copy to $cl_{o_1}$, $cl_{o_2}$;
    apply crossover and mutation
    \State Initialize offspring parameters
    \State Attempt subsumption by parents;
    add survivors to $\mathcal{P}$
    \While{$\sum_k \mathrm{num}_k > N$}
    \State Delete by roulette-wheel on
    \eqref{eq:kacs_del}
    \EndWhile
    \EndIf
    \EndFor
  \end{algorithmic}
\end{algorithm}

\subsection{Rule Representation and Population}
\label{sec:kacs_rule}

Every rule $cl_k \in \mathcal{P}$, regardless of which submodel it belongs to, has the same one-dimensional format:
\begin{equation}
  cl_k:
  \textbf{ IF } z \in C_k = [l_k,\, u_k]
  \textbf{ THEN } P_k(z) = w_{k,0} + w_{k,1}\, z,
  \label{eq:kacs_rule}
\end{equation}
where $C_k \subset \mathbb{R}$ is a compact interval and $\mathbf{w}_k = (w_{k,0},\, w_{k,1})^\top \in \mathbb{R}^2$. For inner rules $cl_k \in \mathcal{P}^\psi_{q,p}$, $z = x_p \in [0,1]$. For outer rules $cl_k \in \mathcal{P}^\Phi_q$, $z = \hat{z}_q \in \mathcal{K}_z^{(q)}$.

{A finite set of such local linear rules can represent a piecewise-linear function. This property is central to the universal approximation proof in Section~\ref{sec:uat_1d}, but does not constrain KACS operation, permitting overlapping rules.}\label{r3-15-1}

The weight vector always has exactly two elements regardless of $n$, in contrast to $n+1$ elements per rule in XCSF. Each rule $cl_k$ maintains the same seven bookkeeping parameters as in XCSF: fitness $F_k$, prediction error $\epsilon_k$, accuracy $\kappa_k$, match set size estimate $\mathrm{ms}_k$, experience $\mathrm{exp}_k$, numerosity $\mathrm{num}_k$, and GA time stamp $\mathrm{ts}_k$.

A rule $cl_k$ belongs to exactly one submodel, identified by its type ($\psi$ or $\Phi$) and index: $(q,p)$ for inner submodels (channel index $q$ and input-dimension index $p$), or $q$ for outer submodels (channel index). The population size limit $N$ constrains the total numerosity, i.e., $\sum_{cl_k \in \mathcal{P}} \mathrm{num}_k \leq N$.

\subsection{Feedforward Prediction and Covering}
\label{sec:kacs_feedforward}

\subsubsection{Stage 1: Inner Function Evaluation}

For each pair $(q, p)$, where $q$ is the channel index ($q \in \{1,\ldots,2n+1\}$) and $p$ is the input-dimension index ($p \in \{1,\ldots,n\}$), form the inner match set for a given input $\mathbf{x}$:
\begin{equation}
  \mathcal{M}^\psi_{q,p}
  = \bigl\{cl_k \in \mathcal{P}^\psi_{q,p}
  \mid x_p \in C_k\bigr\}.
  \label{eq:inner_match}
\end{equation}
If $\mathcal{M}^\psi_{q,p} = \emptyset$, generate a covering rule according to \eqref{eq:cover_inner}, as described later in Section~\ref{sec:kacs_covering}. Compute the fitness-weighted prediction of $\psi_{q,p}$:
\begin{equation}
  \hat{\psi}_{q,p}(x_p)
  = \frac{\sum_{cl_k \in \mathcal{M}^\psi_{q,p}}
  P_k(x_p) \cdot F_k}
  {\sum_{cl_k \in \mathcal{M}^\psi_{q,p}}
  F_k}.
  \label{eq:kacs_inner}
\end{equation}
Sum over $p$ to obtain the intermediate value for channel $q$:
\begin{equation}
  \hat{z}_q = \sum_{p=1}^{n} \hat{\psi}_{q,p}(x_p),
  \qquad q = 1, \ldots, 2n+1.
  \label{eq:kacs_zq}
\end{equation}

\subsubsection{Stage 2: Outer Function Evaluation}

For each channel index $q$, form the outer match set for the obtained intermediate value $\hat{z}_q$:
\begin{equation}
  \mathcal{M}^\Phi_q
  = \bigl\{cl_k \in \mathcal{P}^\Phi_q
  \mid \hat{z}_q \in C_k\bigr\}.
  \label{eq:outer_match}
\end{equation}
If $\mathcal{M}^\Phi_q = \emptyset$, generate a covering rule according to \eqref{eq:cover_outer}, as described later in Section~\ref{sec:kacs_covering}. Compute the fitness-weighted prediction of $\Phi_q$:
\begin{equation}
  \hat{\Phi}_q(\hat{z}_q)
  = \frac{\sum_{cl_k \in \mathcal{M}^\Phi_q}
  P_k(\hat{z}_q) \cdot F_k}
  {\sum_{cl_k \in \mathcal{M}^\Phi_q}
  F_k}.
  \label{eq:kacs_outer}
\end{equation}
Sum over $q$ to obtain the final system prediction:
\begin{equation}
  \hat{y} = \hat{f}(\mathbf{x})
  = \sum_{q=1}^{2n+1} \hat{\Phi}_q(\hat{z}_q).
  \label{eq:kacs_output}
\end{equation}
Stage~2 requires the intermediate values $\hat{z}_q$ in \eqref{eq:kacs_zq}. Thus, KACS first completes Stage~1 for all $(q,p)$ pairs and then proceeds to Stage~2. Within each stage, the submodels are independent and can be evaluated in any order (and also in parallel).\footnote{Under an idealized rule-allocation assumption with $m$ rules per local region or KACS submodel, the worst-case inference complexities of XCSF and KACS scale as $\mathcal{O}(nm^n)$ and $\mathcal{O}(mn^2)$, respectively. For further details, see Section~\ref{sec:supp_inference_complexity} of the supplementary material.\label{r2-20-1}}

\subsubsection{Covering}
\label{sec:kacs_covering}

When a submodel match set is empty, a new rule $cl_\mathrm{cov}$ is generated to ensure every input is always matched.

For inner submodel $\mathcal{P}^\psi_{q,p}$ with uncovered $x_p \in [0,1]$:
\begin{equation}
  (l_\mathrm{cov},\; u_\mathrm{cov}) =
  \bigl(\max(0,\, x_p - \mathcal{U}(0,r_0]),\;
  \min(1,\, x_p + \mathcal{U}(0,r_0])\bigr),
  \label{eq:cover_inner}
\end{equation}
where $r_0 \in (0,1]$ is the maximum covering half-width. With probability $P_\# \in [0,1]$, a maximally general rule is generated instead, by setting $(l_\mathrm{cov},u_\mathrm{cov})=(0,1)$ as in XCSF.

For outer submodel $\mathcal{P}^\Phi_q$ with uncovered $\hat{z}_q \in \mathcal{K}_z^{(q)} \subseteq \mathbb{R}$, the maximally general option is inapplicable because $\mathcal{K}_z^{(q)}$ is not known a priori. The antecedent is set to:
\begin{equation}
  (l_\mathrm{cov},\; u_\mathrm{cov})
  = \bigl(\hat{z}_q - \mathcal{U}(0,r_0],\;
  \hat{z}_q + \mathcal{U}(0,r_0]\bigr).
  \label{eq:cover_outer}
\end{equation}
In both cases, the new rule is initialized as follows: the consequent weight vector $\mathbf{w}_\mathrm{cov}$ is sampled i.i.d.\ from $\mathcal{U}[-1,1]$ for each element, $\epsilon_\mathrm{cov} = 0$, $F_\mathrm{cov} = 0.01$, $\mathrm{ms}_\mathrm{cov} = 1$, $\mathrm{exp}_\mathrm{cov} = 0$, $\mathrm{num}_\mathrm{cov} = 1$, with its time stamp $\mathrm{ts}_\mathrm{cov}$ set to the current iteration $t$, and is added to both its submodel and $\mathcal{P}$.

\subsection{Consequent Update by Backpropagation}
\label{sec:kacs_backprop}

After computing $\hat{y}$, KACS updates the weight vectors of all active rules by backpropagating the system-level loss through \eqref{eq:kacs_output}. Define the loss as $\mathcal{L} = \frac{1}{2}(y - \hat{y})^2$ and the set of all active rules as
\begin{equation}
  \mathcal{M}_\mathrm{act}
  = \Bigl(\bigcup_{q,p} \mathcal{M}^\psi_{q,p}\Bigr)
  \cup \Bigl(\bigcup_{q} \mathcal{M}^\Phi_q\Bigr).
  \label{eq:mact}
\end{equation}

\subsubsection{Gradient for Outer Rules}

For $cl_k \in \mathcal{M}^\Phi_q$, differentiating $\mathcal{L}$ through \eqref{eq:kacs_output} and \eqref{eq:kacs_outer} gives:
\begin{align}
  \frac{\partial \mathcal{L}}{\partial \mathbf{w}_k}
  &= \frac{\partial \mathcal{L}}{\partial \hat{y}}
  \cdot \frac{\partial \hat{y}}{\partial \hat{\Phi}_q(\hat{z}_q)}
  \cdot \frac{\partial \hat{\Phi}_q(\hat{z}_q)}{\partial P_k(\hat{z}_q)}
  \cdot \frac{\partial P_k(\hat{z}_q)}{\partial \mathbf{w}_k}
  \notag \\
  &= -(y-\hat{y})
  \cdot \frac{F_k}{\sum_{cl_{k'} \in \mathcal{M}^\Phi_q} F_{k'}}
  \cdot \hat{\mathbf{z}}'_q,
  \label{eq:grad_outer}
\end{align}
where $\hat{\mathbf{z}}'_q = (1,\, \hat{z}_q)^\top$.

\subsubsection{Gradient for Inner Rules}

For $cl_k \in \mathcal{M}^\psi_{q,p}$, the gradient must pass through the outer stage because $\hat{z}_q$ enters \eqref{eq:kacs_outer}. Applying the chain rule through \eqref{eq:kacs_output}$\to$\eqref{eq:kacs_outer}$\to$ \eqref{eq:kacs_zq}$\to$\eqref{eq:kacs_inner}:
\begin{align}
  \frac{\partial \mathcal{L}}{\partial \mathbf{w}_k}
  &= \frac{\partial \mathcal{L}}{\partial \hat{y}}
  \cdot \frac{\partial \hat{y}}{\partial \hat{\Phi}_q(\hat{z}_q)}
  \cdot \frac{\partial \hat{\Phi}_q(\hat{z}_q)}{\partial \hat{z}_q} \notag \\
  &\quad \cdot \frac{\partial \hat{z}_q}{\partial \hat{\psi}_{q,p}(x_p)}
  \cdot \frac{\partial \hat{\psi}_{q,p}(x_p)}{\partial P_k(x_p)}
  \cdot \frac{\partial P_k(x_p)}{\partial \mathbf{w}_k}
  \notag \\
  &= -(y-\hat{y})
  \cdot \frac{\sum_{cl_j \in \mathcal{M}^\Phi_q}
  w_{j,1} \cdot F_j}
  {\sum_{cl_j \in \mathcal{M}^\Phi_q}
  F_j}
  \cdot \frac{F_k}{\sum_{cl_{k'} \in \mathcal{M}^\psi_{q,p}} F_{k'}}
  \cdot \mathbf{x}'_p,
  \label{eq:grad_inner}
\end{align}
where $\mathbf{x}'_p = (1,\, x_p)^\top$.

\subsubsection{Adam Update}

Each rule $cl_k \in \mathcal{M}_\mathrm{act}$ maintains its own first and second moment vectors $\mathbf{m}_k, \mathbf{v}_k \in \mathbb{R}^2$, both initialized to $\mathbf{0}$ when the rule is first created. Here, $t_k$ denotes a \emph{rule-local} Adam time step, defined as $t_k=\mathrm{exp}_k+1$ at the moment of updating $cl_k$. We use $+1$ because newly created rules (by covering or the GA) have $\mathrm{exp}_k=0$ before their first gradient update. At iteration $t$, after computing $\mathbf{g}_k = \partial \mathcal{L} / \partial \mathbf{w}_k$:
\begin{align}
  \mathbf{m}_k &\leftarrow
  \beta_1 \mathbf{m}_k + (1 - \beta_1)\,\mathbf{g}_k,
  \;
  \mathbf{v}_k \leftarrow
  \beta_2 \mathbf{v}_k
  + (1 - \beta_2)\,\mathbf{g}_k^2,
  \label{eq:adam_m} \\
  \mathbf{w}_k &\leftarrow
  \mathbf{w}_k
  - \eta\,
  \frac{\mathbf{m}_k / (1 - \beta_1^{t_k})}
  {\sqrt{\mathbf{v}_k / (1 - \beta_2^{t_k})}
  + \varepsilon_\mathrm{Adam}},
  \label{eq:adam_w}
\end{align}
where squaring and the square root in \eqref{eq:adam_m}--\eqref{eq:adam_w} are elementwise. {We use the default values recommended in \cite{kingma2015adam}: $\eta=0.001$, $\beta_1=0.9$, $\beta_2=0.999$, and $\varepsilon_\mathrm{Adam}=10^{-8}$.}\label{r3-14}

\subsection{Parameter Update}
\label{sec:kacs_params}

After the gradient update, the following bookkeeping parameters are updated for all $cl_k \in \mathcal{M}_\mathrm{act}$.

First, the experience is updated as $\mathrm{exp}_k \leftarrow \mathrm{exp}_k + 1$.

Next, the prediction error is updated. Unlike XCSF, which uses each rule's individual error $|y - \hat{y}_k|$, KACS uses the system-level error for all rules because each rule contributes to $\hat{y}$ as part of a larger structure, and its individual output is not directly comparable to $y$:
\begin{equation}
  \epsilon_k \leftarrow
  \epsilon_k + \beta\,(|y - \hat{y}| - \epsilon_k),
  \label{eq:kacs_err}
\end{equation}
with learning rate $\beta\in(0,1]$.

Then, the accuracy is calculated as:
\begin{equation}
  \kappa_k =
  \begin{cases}
    1, & \text{if } \epsilon_k < \epsilon_0, \\
    \alpha\,(\epsilon_k / \epsilon_0)^{-\nu},
    & \text{otherwise,}
  \end{cases}
  \label{eq:kacs_acc}
\end{equation}
with error threshold $\epsilon_0>0$, accuracy fall-off rate $\alpha \in [0,1]$, accuracy exponent $\nu > 0$.

After that, the fitness is updated within the \emph{submodel} match set $\mathcal{M}\in \{\mathcal{M}^\psi_{q,p}\} \cup \{\mathcal{M}^\Phi_q\}$, not across all active rules. For $cl_k\in\mathcal{M}$:
\begin{equation}
  F_k \leftarrow F_k + \beta\left(
    \frac{\kappa_k \cdot \mathrm{num}_k}
    {\sum_{cl_{k'} \in \mathcal{M}}
    \kappa_{k'} \cdot \mathrm{num}_{k'}}
  - F_k \right).
  \label{eq:kacs_fit}
\end{equation}
This per-submodel update ensures that the fitness of a rule reflects its accuracy relative to other rules competing to approximate the same component function.

Finally, the match set size estimate is updated as:
\begin{equation}
  \mathrm{ms}_k \leftarrow \mathrm{ms}_k
  + \beta\left(
    \textstyle\sum_{cl_{k'} \in \mathcal{M}} \mathrm{num}_{k'}
  - \mathrm{ms}_k\right).
  \label{eq:kacs_ms}
\end{equation}

\subsection{Genetic Algorithm and Subsumption}
\label{sec:kacs_ea}

For each active submodel match set $\mathcal{M}$, the GA is activated when $t - {\sum_{cl_k \in \mathcal{M}} \mathrm{num}_k \cdot \mathrm{ts}_k}/ {\sum_{cl_k \in \mathcal{M}} \mathrm{num}_k} > \theta_\mathrm{GA}$. When triggered, the time stamp is updated as $\mathrm{ts}_k \leftarrow t$ for all $cl_k \in \mathcal{M}$. In the GA, two parents $cl_{p_1}$ and $cl_{p_2}$ are selected from $\mathcal{M}$ by tournament selection with tournament size $\tau$. Two offspring $cl_{o_1}$ and $cl_{o_2}$ are initialized as copies of $cl_{p_1}$ and $cl_{p_2}$. With probability $\chi$, crossover is applied: $l_{o_1}$ and $l_{o_2}$ are swapped with probability $0.5$, and the same is done for $u_{o_1}$ and $u_{o_2}$. If crossover yields an invalid interval (i.e., $l_{o_i} > u_{o_i}$ for $i \in \{1,2\}$), then $l_{o_i}$ and $u_{o_i}$ are swapped so that $l_{o_i} \le u_{o_i}$. With probability $\mu$ per bound, mutation is applied: add $\mathcal{U}[-m_0, m_0]$ to each of $l_{o_i}$ and $u_{o_i}$ for $i \in \{1,2\}$, where $m_0>0$ is the mutation magnitude. For inner rules, both bounds are clipped to $[0,1]$. For outer rules, no clipping is applied. After crossover and mutation, for $i \in \{1,2\}$, set $\epsilon_{o_i} \leftarrow \tfrac{1}{2}(\epsilon_{p_1} + \epsilon_{p_2})$, $F_{o_i} \leftarrow 0.1 \times \tfrac{1}{2}(F_{p_1} + F_{p_2})$, $\mathrm{exp}_{o_i} = 0$, and $\mathrm{num}_{o_i} = 1$.

Before adding an offspring $cl_o\in\{cl_{o_1},cl_{o_2}\}$ to $\mathcal{P}$, each parent $cl_p$ attempts to subsume it. Subsumption occurs when all the following three conditions are satisfied: (i) $[l_o, u_o] \subseteq [l_p, u_p]$; (ii) $\epsilon_p < \epsilon_0$; and (iii) $\mathrm{exp}_p > \theta_\mathrm{sub}$. If subsumed, $\mathrm{num}_p \leftarrow \mathrm{num}_p + \mathrm{num}_o$ and $cl_o$ is discarded.

After inserting surviving offspring, while $\sum_{cl_k \in \mathcal{P}} \mathrm{num}_k > N$, one rule is deleted by roulette-wheel selection with vote
\begin{equation}
  d_k =
  \begin{cases}
    \mathrm{ms}_k \cdot \mathrm{num}_k
    \cdot \bar{F} / F_k,
    & \mathrm{exp}_k > \theta_\mathrm{del}
    \text{ and }
    F_k < \delta\bar{F}, \\
    \mathrm{ms}_k \cdot \mathrm{num}_k,
    & \text{otherwise,}
  \end{cases}
  \label{eq:kacs_del}
\end{equation}
where $\bar{F} = \sum_{k \in \mathcal{P}} F_k / N$, $\theta_\mathrm{del}>0$ is the deletion threshold, and $\delta\in(0,1]$. The numerosity of the selected rule is decreased by one; any rule with $\mathrm{num}_k = 0$ is removed from $\mathcal{P}$.

\section{Universal Approximation Theorem of KACS}
\label{sec:uat}

\subsection{Assumptions and Main Theorem}
\label{sec:uat_main}

The proof in this section concerns the \emph{expressive power} of KACS, not the behavior of its learning algorithm. The following two assumptions make this scope precise.

\begin{assumption}[Target function]
  \label{asm:target}
  The target function $f$ is a real-valued continuous function defined on the compact set $\mathcal{K} = [0,1]^n$, i.e., $f \in C(\mathcal{K})$.
\end{assumption}

\begin{assumption}[Existence (not learnability)]
  \label{asm:expressive}
  This assumption concerns representational existence rather than algorithmic learnability. The proof is an existence theorem for the expressive power of KACS. We assume that KACS can represent and store any rule $cl_k$ with any compact interval $C_k = [l_k, u_k]$ as its antecedent, any linear function $P_k(z) = w_{k,0} + w_{k,1}z$ as its consequent, and any fitness value $F_k \in (0,1]$. This is the same standard assumption adopted in universal approximation proofs for neural networks~\cite{cybenko1989approximation} and fuzzy systems~\cite{ying1994sufficient}, where arbitrary weights or membership functions are assumed to be expressible. {It does not guarantee that the online updates and GA find such rules.}\label{r2-6}
\end{assumption}

Under these assumptions, the main theorem is stated as follows.

\begin{theorem}[Universal Approximation of KACS]
  \label{thm:uat}
  Let $\mathcal{F}_\mathrm{KACS}$ be the set of all functions $\hat{f}$ expressible by KACS models on $\mathcal{K} = [0,1]^n$ as defined in Section~\ref{sec:kacs}. Then $\mathcal{F}_\mathrm{KACS}$ is dense in $C(\mathcal{K})$. That is, for any $f \in C(\mathcal{K})$ and any $\varepsilon > 0$, there exists $\hat{f} \in \mathcal{F}_\mathrm{KACS}$ such that
  \begin{equation}
    \sup_{\mathbf{x} \in \mathcal{K}}
    |f(\mathbf{x}) - \hat{f}(\mathbf{x})| < \varepsilon.
    \label{eq:uat}
  \end{equation}
\end{theorem}

The proof uses the following two-step strategy.
\begin{enumerate}
  \item \text{Step 1} (Section~\ref{sec:uat_1d}): We prove that a KACS ruleset operating on a one-dimensional compact interval, referred to as a \emph{1D KACS submodel}, is itself a universal approximator for $C([a,b])$.
  \item \text{Step 2} (Section~\ref{sec:uat_nd}): We apply the KA representation theorem (Theorem~\ref{thm:ka}) to extend the one-dimensional result to $n$ dimensions, thereby proving Theorem~\ref{thm:uat}.
\end{enumerate}

\subsection{Step 1: 1D KACS Submodel as a Universal Approximator}
\label{sec:uat_1d}

A \emph{1D KACS submodel} $\hat{g}$ is a KACS ruleset defined on a one-dimensional compact interval $\mathcal{I} = [a, b] \subset \mathbb{R}$. Each rule $cl_k \in \mathcal{P}$ has the form given in \eqref{eq:kacs_rule}, and the output for a given $z \in \mathcal{I}$ is
\begin{equation}
  \hat{g}(z) =
  \frac{\sum_{cl_k \in \mathcal{M}(z)}
  P_k(z) \cdot F_k}
  {\sum_{cl_k \in \mathcal{M}(z)} F_k},
  \quad
  \mathcal{M}(z) = \{cl_k \in \mathcal{P} \mid z \in C_k\},
  \label{eq:1d_output}
\end{equation}
which matches \eqref{eq:kacs_inner} and \eqref{eq:kacs_outer} in the one-dimensional case. Let $\mathcal{G}^\mathrm{1D}_\mathrm{KACS}$ denote the set of all functions expressible by 1D KACS submodels on $\mathcal{I}$.

\begin{theorem}[1D Universal Approximation]
  \label{thm:uat_1d}
  $\mathcal{G}^\mathrm{1D}_\mathrm{KACS}$ is dense in $C(\mathcal{I})$. That is, for any $g \in C(\mathcal{I})$ and any $\varepsilon > 0$, there exists $\hat{g} \in \mathcal{G}^\mathrm{1D}_\mathrm{KACS}$ such that $\sup_{z \in \mathcal{I}} |g(z) - \hat{g}(z)| < \varepsilon$.
\end{theorem}

{As noted in Section~\ref{sec:kacs_rule}, the proof uses piecewise-linear approximation as a bridge from local linear rules to arbitrary continuous functions.}\label{r3-15-2}

\begin{theorem}[Piecewise Linear Approximation~\cite{folland1999real,rudin1976principles}]
  \label{thm:pla}
  For any $g \in C(\mathcal{I})$ and any $\varepsilon > 0$, there exists a continuous piecewise linear function $h \in \mathrm{CPL}(\mathcal{I})$ such that $\sup_{z \in \mathcal{I}} |g(z) - h(z)| < \varepsilon$, where $\mathrm{CPL}(\mathcal{I})$ is the set of continuous functions on $\mathcal{I}$ that are linear on each subinterval of a finite partition of $\mathcal{I}$.
\end{theorem}

Theorem~\ref{thm:pla} reduces the proof of Theorem~\ref{thm:uat_1d} to the following lemma.

\begin{lemma}
  \label{lem:cpl}
  For any $h \in \mathrm{CPL}(\mathcal{I})$ and any $\varepsilon > 0$, there exists $\hat{g} \in \mathcal{G}^\mathrm{1D}_\mathrm{KACS}$ such that $\sup_{z \in \mathcal{I}} |h(z) - \hat{g}(z)| < \varepsilon / 2$.
\end{lemma}

\begin{proof}
  Since $h \in \mathrm{CPL}(\mathcal{I})$, there is a finite partition $a = z_0 < z_1 < \cdots < z_M = b$ of $\mathcal{I}$ into $M$ subintervals $R_j = [z_{j-1},\, z_j]$ such that $h$ restricts to a linear function $h_j(z) = a_j z + b_j$ on each $R_j$.

  For each subinterval $R_j$ ($j = 1, \ldots, M$), we place a \emph{reference rule} $cl^*_j$ in $\mathcal{P}$ with antecedent $C^*_j = R_j$ and consequent $P^*_j(z) = h_j(z)$. Since $h_j$ is linear, this satisfies \eqref{eq:kacs_rule} under Assumption~\ref{asm:expressive}. The ruleset $\mathcal{P}$ may additionally contain any finite number of rules $\{cl_k\}_{k \neq j}$ with arbitrary consequents, provided their antecedents are compact intervals and their fitnesses satisfy $F_k \in (0,1]$.

  Fix any $z \in \mathcal{I}$. Since $z \in R_j$ for some $j$, the reference rule $cl^*_j$ is guaranteed to match. If $z$ lies on a breakpoint $z_j$, continuity of $h$ gives $h_j(z_j) = h_{j+1}(z_j)$, so either reference rule yields the same value and the argument is unaffected. Applying the triangle inequality:
  \begin{equation}
    |h(z) - \hat{g}(z)|
    \leq
    \underbrace{|h_j(z) - P^*_j(z)|}_{\text{(A)}}
    +\;
    \underbrace{|P^*_j(z) - \hat{g}(z)|}_{\text{(B)}}.
    \label{eq:err_decomp}
  \end{equation}
  Here, (A) denotes the \textit{reference rule error}, which measures the local approximation error of the selected reference rule $cl_j^*$ on the interval $R_j$; (B) denotes the \textit{aggregated error}, which captures the additional error induced by the fitness-weighted aggregation with the other matching rules $\{cl_k\}$. We show that both (A) and (B) can be made less than $\varepsilon/4$, so their sum is less than $\varepsilon/2$.

  \emph{(A) Reference Rule Error.} By construction $P^*_j(z) = h_j(z)$, so
  \begin{equation}
    \text{(A)}=|h_j(z) - P^*_j(z)| = 0 < \frac{\varepsilon}{4}.
    \label{eq:ref_err}
  \end{equation}

  \emph{(B) Aggregated Error.} Substituting \eqref{eq:1d_output} and separating the reference rule $cl^*_j$:
  \begin{align}
    |P^*_j(z) - \hat{g}(z)|
    &=\left| \frac{\sum_{cl_k \in \mathcal{M}(z)} (P_j^*(z) - P_k(z)) \cdot F_k}{\sum_{cl_k \in \mathcal{M}(z)} F_k} \right| \notag \\
    &\leq
    \frac{
      \sum_{cl_k \in \mathcal{M}(z) \setminus \{cl^*_j\}}
    |P^*_j(z) - P_k(z)|\cdot F_k}
    {F^*_j +
    \sum_{cl_k \in \mathcal{M}(z) \setminus \{cl^*_j\}} F_k}.
    \label{eq:agg_err}
  \end{align}
  Since all consequents are linear on the compact set $\mathcal{I}$, the quantity $\Lambda = \max_{1 \le j \le M}\max_{1 \le k \le |\mathcal{P}|} \sup_{z \in \mathcal{I}} |P^*_j(z) - P_k(z)|$ is finite. Set $F^*_j = F_{\mathrm{ref}} \in (0,1]$ and $F_k = \delta \in (0,1]$ for all $k \neq j$. Since $|\mathcal{M}(z) \setminus \{cl^*_j\}| \leq |\mathcal{P}| - 1$, \eqref{eq:agg_err} gives
  \begin{equation}
    |P^*_j(z) - \hat{g}(z)|
    \leq
    \Lambda \cdot
    \frac{(|\mathcal{P}|-1)\,\delta}
    {F_{\mathrm{ref}} + (|\mathcal{P}|-1)\,\delta}
    = \sigma(\gamma),
    \label{eq:sigma}
  \end{equation}
  where $\gamma = \delta / F_{\mathrm{ref}}$ and $\sigma(\gamma) = \Lambda \cdot \frac{(|\mathcal{P}|-1)\gamma}{1+(|\mathcal{P}|-1)\gamma}$. Since $\lim_{\gamma \to 0^+} \sigma(\gamma) = 0$, the $\varepsilon$-$\delta$ definition of the limit guarantees that, for any $E>0$, there exists $\Delta>0$ such that $0 < \gamma < \Delta \Rightarrow \sigma(\gamma) < E$. In particular, setting $E = \varepsilon/4$ yields $0 < \gamma < \Delta \Rightarrow \sigma(\gamma) < \varepsilon/4$. Choosing $\delta > 0$ sufficiently small and $F_{\mathrm{ref}}$ sufficiently close to $1$ so that $\gamma = \delta/F_{\mathrm{ref}} < \Delta$ yields
  \begin{equation}
    \text{(B)}=|P^*_j(z) - \hat{g}(z)| \leq \sigma(\gamma) < \frac{\varepsilon}{4}.
    \label{eq:agg_err_bound}
  \end{equation}

  \emph{Combining (A) and (B).} Substituting \eqref{eq:ref_err} and \eqref{eq:agg_err_bound} into \eqref{eq:err_decomp} gives
  \begin{equation}
    \sup_{z \in \mathcal{I}} |h(z) - \hat{g}(z)|
    < \frac{\varepsilon}{4} + \frac{\varepsilon}{4}
    = \frac{\varepsilon}{2}.
  \end{equation}
  This completes the proof of Lemma~\ref{lem:cpl}.
\end{proof}

\begin{proof}[Proof of Theorem~\ref{thm:uat_1d}]
  Let $g \in C(\mathcal{I})$ and $\varepsilon > 0$. By Theorem~\ref{thm:pla} with $\varepsilon/2$, there exists $h \in \mathrm{CPL}(\mathcal{I})$ with $\sup_{z \in \mathcal{I}} |g(z) - h(z)| < \varepsilon/2$. By Lemma~\ref{lem:cpl} with $\varepsilon$, there exists $\hat{g} \in \mathcal{G}^\mathrm{1D}_\mathrm{KACS}$ with $\sup_{z \in \mathcal{I}} |h(z) - \hat{g}(z)| < \varepsilon/2$. The triangle inequality gives
  \begin{align}
    \sup_{z \in \mathcal{I}} |g(z) - \hat{g}(z)|
    &\leq \sup_{z \in \mathcal{I}} |g(z) - h(z)|
    + \sup_{z \in \mathcal{I}} |h(z) - \hat{g}(z)| \notag\\
    &< \frac{\varepsilon}{2} + \frac{\varepsilon}{2}
    = \varepsilon.
  \end{align}
  This completes the proof of Theorem~\ref{thm:uat_1d}.
\end{proof}

\subsection{Step 2: Extension to $n$ Dimensions}
\label{sec:uat_nd}

\begin{proof}[Proof of Theorem~\ref{thm:uat}]
  Let $f \in C(\mathcal{K})$ and $\varepsilon > 0$. By Theorem~\ref{thm:ka}, there exist continuous functions $\psi_{q,p}: [0,1] \to \mathbb{R}$ and $\Phi_q: \mathcal{K}^{(q)}_z \to \mathbb{R}$ such that $f(\mathbf{x}) = \sum_{q=1}^{2n+1} \Phi_q(z_q)$, where $z_q = \sum_{p=1}^{n} \psi_{q,p}(x_p)$ and $\mathcal{K}^{(q)}_z = \bigl\{\sum_{p=1}^n \psi_{q,p}(x_p) \mid \mathbf{x} \in \mathcal{K}\bigr\} \subset \mathbb{R}$ is a compact interval.

  Define $\hat{f}(\mathbf{x}) = \sum_{q=1}^{2n+1} \hat{\Phi}_q(\hat{z}_q)$, where $\hat{z}_q = \sum_p \hat{\psi}_{q,p}(x_p)$. For all $\mathbf{x} \in \mathcal{K}$,
  \begin{align}
    |f(\mathbf{x}) - \hat{f}(\mathbf{x})|
    &= \left|\sum_{q=1}^{2n+1} \Phi_q(z_q)
    - \sum_{q=1}^{2n+1} \hat{\Phi}_q(\hat{z}_q)\right| \notag \\
    &\leq \sum_{q=1}^{2n+1}
    \Bigl[
      \underbrace{|\Phi_q(z_q) - \Phi_q(\hat{z}_q)|}_{\text{(C)}}
      +
      \underbrace{|\Phi_q(\hat{z}_q) - \hat{\Phi}_q(\hat{z}_q)|}_{\text{(D)}}
    \Bigr].
    \label{eq:total_err}
  \end{align}
  Here, (C) is the \textit{inner approximation error}: the error in $\Phi_q$'s input caused by replacing the true intermediate value $z_q$ with its approximation $\hat{z}_q$; (D) is the \textit{outer approximation error}: the error from approximating $\Phi_q$ itself by the 1D KACS submodel $\hat{\Phi}_q$. We show that both (C) and (D) are bounded by $\varepsilon/(2(2n+1))$ for each $q$, so the total is at most $\varepsilon$.

  \emph{(C) Inner Approximation Error.} KACS computes intermediate values $z_q = \sum_p \psi_{q,p}(x_p)$ and feeds them into the outer function $\Phi_q$. Because each $\psi_{q,p}$ is replaced by an approximation $\hat{\psi}_{q,p}$, the resulting intermediate value $\hat{z}_q = \sum_p \hat{\psi}_{q,p}(x_p)$ may deviate from $z_q$ and could fall outside the range $\mathcal{K}^{(q)}_z$ on which $\Phi_q$ was originally analyzed. Error~(C) quantifies the impact of this deviation: $|\Phi_q(z_q) - \Phi_q(\hat{z}_q)|$. We now bound this error.

  {Fix any distance $r > 0$ as the enlargement width around $\mathcal{K}^{(q)}_z$, and define the \emph{enlarged} compact interval}\label{r3-16}
  \begin{equation}
    \tilde{\mathcal{K}}^{(q)}_z
    = \bigl\{t \in \mathbb{R}
    \mid \mathrm{dist}(t,\,\mathcal{K}^{(q)}_z) \leq r\bigr\},
    \label{eq:K_tilde}
  \end{equation}
  where $\mathrm{dist}(t, S) = \inf_{s \in S}|t - s|$ denotes the distance from the point $t$ to the set $S$~\cite{rudin1976principles,folland1999real}. Concretely, if $\mathcal{K}^{(q)}_z = [a, b]$, then $\tilde{\mathcal{K}}^{(q)}_z = [a-r,\, b+r]$, which is again a compact interval~\cite{rudin1976principles}. This enlarged set acts as a \emph{safety buffer}: any $\hat{z}_q$ that stays within distance $r$ of the true value $z_q$ is guaranteed to remain inside $\tilde{\mathcal{K}}^{(q)}_z$, where $\Phi_q$ is still well-defined.

  Since $\Phi_q$ is continuous on the compact set $\tilde{\mathcal{K}}^{(q)}_z$, it is \emph{uniformly continuous} there by the Heine-Cantor theorem~\cite{rudin1976principles}. Precisely, for the target accuracy $\delta_\Phi = \varepsilon/(2(2n+1))$, there exists $\delta^*_z>0$ such that
  \begin{equation}
    t_1, t_2 \in \tilde{\mathcal{K}}^{(q)}_z,\quad
    |t_1 - t_2| < \delta^*_z
    \;\Longrightarrow\;
    |\Phi_q(t_1) - \Phi_q(t_2)| < \delta_\Phi.
    \label{eq:phi_uc}
  \end{equation}

  Define $\delta_z = \min(\delta^*_z,\, r)$. Because $\delta_z \leq \delta^*_z$, condition~\eqref{eq:phi_uc} remains valid for any pair with $|t_1-t_2| < \delta_z$. Because $\delta_z \leq r$, any deviation of $\hat{z}_q$ from $z_q$ smaller than $\delta_z$ is also smaller than $r$, keeping $\hat{z}_q$ inside $\tilde{\mathcal{K}}^{(q)}_z$ (verified in the next paragraph). Distributing the budget equally among $n$ inner functions, set $\delta_\psi = \delta_z/n$. By Theorem~\ref{thm:uat_1d}, for each pair $(q,p)$ there exists a 1D submodel $\hat{\psi}_{q,p}$ with
  \begin{equation}
    \sup_{x_p \in [0,1]}
    |\psi_{q,p}(x_p)-\hat{\psi}_{q,p}(x_p)| < \delta_\psi.
  \end{equation}
  Applying the triangle inequality gives
  \begin{equation}
    |z_q - \hat{z}_q|
    \leq \sum_{p=1}^n
    |\psi_{q,p}(x_p) - \hat{\psi}_{q,p}(x_p)|
    < n \cdot \delta_\psi
    = \delta_z
    \leq r.
    \label{eq:zq_err}
  \end{equation}
  Since $z_q \in \mathcal{K}^{(q)}_z \subseteq \tilde{\mathcal{K}}^{(q)}_z$ and the deviation is less than $r$, it follows that $\hat{z}_q \in \tilde{\mathcal{K}}^{(q)}_z$.

  Both $z_q$ and $\hat{z}_q$ lie in $\tilde{\mathcal{K}}^{(q)}_z$, and $|z_q-\hat{z}_q| < \delta_z \leq \delta^*_z$. Applying~\eqref{eq:phi_uc} with $t_1 = z_q$ and $t_2 = \hat{z}_q$ yields
  \begin{equation}
    \text{(C)}
    = |\Phi_q(z_q)-\Phi_q(\hat{z}_q)|
    < \delta_\Phi
    = \frac{\varepsilon}{2(2n+1)}.
  \end{equation}

  \emph{(D) Outer Approximation Error.} Since $\Phi_q \in C(\tilde{\mathcal{K}}^{(q)}_z)$, Theorem~\ref{thm:uat_1d} applied on the compact interval $\tilde{\mathcal{K}}^{(q)}_z$ gives a 1D submodel $\hat{\Phi}_q$ such that
  \begin{equation}
    \sup_{t \in \tilde{\mathcal{K}}^{(q)}_z}
    |\Phi_q(t) - \hat{\Phi}_q(t)|
    < \delta_\Phi = \frac{\varepsilon}{2(2n+1)}.
  \end{equation}
  Since $\hat{z}_q(\mathbf{x}) \in \tilde{\mathcal{K}}^{(q)}_z$ for all $\mathbf{x} \in \mathcal{K}$ (established above), evaluating at $t = \hat{z}_q(\mathbf{x})$ yields
  \begin{equation}
    \text{(D)}=|\Phi_q(\hat{z}_q) - \hat{\Phi}_q(\hat{z}_q)|
    < \delta_\Phi = \frac{\varepsilon}{2(2n+1)}.
  \end{equation}

  \emph{Total Error.} Substituting into \eqref{eq:total_err},
  \begin{align}
    \sup_{\mathbf{x} \in \mathcal{K}}
    |f(\mathbf{x}) - \hat{f}(\mathbf{x})|
    &< \sum_{q=1}^{2n+1}
    \left[\frac{\varepsilon}{2(2n+1)}
    + \frac{\varepsilon}{2(2n+1)}\right] \notag \\
    &= \sum_{q=1}^{2n+1} \frac{\varepsilon}{2n+1}
    = \varepsilon.
  \end{align}
  This completes the proof of Theorem~\ref{thm:uat}.\footnote{{Note that Theorem~\ref{thm:uat} does not imply that a KACS model with fixed capacity can achieve arbitrary accuracy. Although the construction uses finitely many rules for any $\varepsilon>0$, the required rule count may increase as $\varepsilon$ decreases. Thus, in practice, the population limit $N$ imposes a trade-off among approximation accuracy, model compactness, and transparency.}\label{r3-11}}
\end{proof}

Theorem~\ref{thm:uat} also holds for consequent models that represent linear functions exactly or approximate them arbitrarily well; see Section~\ref{sec:uat_discussion_sup} of the supplementary material.

\section{Experiments}
\label{sec:experiments}

\subsection{Experimental Setup}
\label{subsec:exp_setup}

\subsubsection{Methods, Metrics, and Objective}
We compare two LCSs: (i) XCSF, which uses $n$-dimensional rules with linear consequents, and (ii) KACS, which uses one-dimensional rules with linear consequents via KA-based decomposition.

We evaluate each method using three metrics: model accuracy, model complexity, and the accuracy-complexity trade-off.
\begin{itemize}
  \item \emph{Model Accuracy.} We track MAE (mean absolute error) on training and testing data.\label{others-5}
  \item \emph{Model Complexity.} We count the number of {macro-}rules $|\mathcal{P}|$ and the total number of parameters $k$ in $\mathcal{P}$.
  \item \emph{Accuracy-Complexity Trade-off.} We use AIC (Akaike Information Criterion) to capture the balance between accuracy and complexity. Under the Gaussian error assumption, AIC is computed as \cite{burnham2002model}:
    \begin{equation}
      \mathrm{AIC} = N_\text{tr} \ln\!\left(
        \sum_{i=1}^{N_\text{tr}} (y_i-\hat{y}_i)^2 \big/ N_\text{tr}
      \right) + 2(k + 1),
      \label{eq:aic}
    \end{equation}
    where $N_\text{tr}$ is the number of training samples, $k$ is the number of model parameters, and the extra $+1$ accounts for the estimated noise variance $\hat{\sigma}^{2}$.
\end{itemize}
The parameter count $k$, also used in AIC, is defined as the total number of consequent weight parameters across all rules, i.e., $k = |\mathcal{P}| \times (d+1)$, where $d$ is the input dimension of each rule ($d = n$ for XCSF; $d = 1$ for KACS). {Hence, for fixed $n$, $k$ scales linearly with $|\mathcal{P}|$ for both methods.}\label{r2-14}

The objective of this comparison is to {evaluate KA-based rule organization, which replaces $n$-dimensional XCSF rules with one-dimensional KACS rulesets} while keeping the surrounding LCS framework as comparable as possible.\label{r1-2-2}

For this reason, XCSF is used as the primary baseline: it is the most established LCS for function approximation and enables a controlled comparison in which the key difference is rule dimensionality rather than unrelated architectural choices. Accordingly, we do not include KAN~\cite{liu2025kan} in the main comparison, because it is a neural network with B-spline activations and is structurally distinct from LCSs. We also do not include X-KAN~\cite{shiraishi2025xkan} in the main comparison, because its KAN-based consequents substantially change the consequent model class and tuning space. We therefore interpret the experimental results as evidence for the effect of dimensional reorganization within the LCS framework, not as a universal ranking against all KA-inspired models. A systematic comparison with KAN and X-KAN is left for future work. For a comparison with SupRB~\cite{heider2023suprb}, see Section~\ref{sec:suprb_comparison_sup} of the supplementary material.
\label{r2-8-1}\label{r2-10-1}\label{r2-15-1}\label{r2-19-1}

\subsubsection{Benchmark Problems}
Table~\ref{tb:dataset} summarizes four synthetic test functions and four real-world regression datasets used.

\begin{table}[t]
  \centering
  \caption{Benchmark Problems Used}
  \label{tb:dataset}
  \footnotesize
  \begin{tabular}{c l c c}
    \bhline{1pt}
    Abbr. & Name & \#Samples & $n$ \\
    \bhline{1pt}
    $f_1$ & Rastrigin Function & 1000 & 10 \\
    $f_2$ & Rosenbrock Function & 1000 & 10 \\
    $f_3$ & Cross Function & 1000 & 10 \\
    $f_4$ & Styblinski-Tang Function & 1000 & 10 \\
    ASN  & Airfoil Self-Noise & 1503 & 5 \\
    CCPP & Combined Cycle Power Plant & 9568 & 4 \\
    CS   & Concrete Strength & 1030 & 8 \\
    EEC  & Energy Efficiency Cooling & 768 & 8 \\
    \bhline{1pt}
  \end{tabular}
\end{table}

The synthetic functions are the multimodal Rastrigin function~\cite{muhlenbein1991parallel}, the valley-shaped Rosenbrock function~\cite{jamil2013literature}, the ridge-like Cross function~\cite{stein2018interpolation}, and the nonconvex Styblinski-Tang function~\cite{stein2018interpolation}, all evaluated in $n=10$ dimensions with 1000 uniformly sampled points over $[0,1]^n$. Fig.~\ref{fig:all_functions} illustrates each function for $n=2$, and the corresponding mathematical formulations are provided in Section~\ref{sec:Formulations of Synthetic Functions sup} of the supplementary material. The real-world datasets are taken from \cite{heider2023suprb}, and include two highly nonlinear problems (ASN and CS) and two nearly linear problems (CCPP and EEC).

\preparesyntheticsurfaceplots
\begin{figure*}[!t]
  \centering
  \subfloat[\shortstack{$f_1$: Rastrigin Function}]{
    \makebox[0.22\textwidth][c]{\syntheticrastriginplot}
    \label{fig:rastrigin}
  }\hfill
  \subfloat[\shortstack{$f_2$: Rosenbrock Function}]{
    \makebox[0.22\textwidth][c]{\syntheticrosenbrockplot}
    \label{fig:rosenbrock}
  }\hfill
  \subfloat[\shortstack{$f_3$: Cross Function}]{
    \makebox[0.22\textwidth][c]{\syntheticcrossplot}
    \label{fig:cross_function}
  }\hfill
  \subfloat[\shortstack{$f_4$: Styblinski-Tang Function}]{
    \makebox[0.22\textwidth][c]{\syntheticstyblinskitangplot}
    \label{fig:styblinski_tang}
  }
  \caption{Surface plots of the four synthetic test functions ($n=2$) shown for visualization; throughout the experiments we fix $n=10$.}
  \label{fig:all_functions}
  \vspace{0.5\baselineskip}

  \begingroup
  \makeatletter
  \def\@captype{table}
  \makeatother
  \centering
  \caption{Summary of Main Results. Green Indicates the Best Value. Rank Is the Average Rank Across Problems. {For Each Problem, Symbols $+/-/\sim$ Denote Significantly Better/Worse/Similar Performance Compared With KACS, Based on a Paired Wilcoxon Signed-Rank Test Over 30 Runs Using Identical Data Splits. The $p$-Values in the Final Row Are From Across-Problem Paired Wilcoxon Signed-Rank Tests Applied to the 30-Run Mean Values for the Eight Benchmark Problems.} Arrows $\uparrow/\downarrow$ Indicate Rank Improvement/Decline Relative to KACS. Statistical Significance {Is} at $\alpha=0.05$ ($\dag$)}\label{r2-11-1}
  \label{tb:result}
  \renewcommand{\arraystretch}{1.2}
  \resizebox{\width}{!}{
    \footnotesize
    \begin{tabular}{c|cc|cc|cc|cc|cc}
      \bhline{1pt}
      &\multicolumn{4}{c|}{\textsc{Model Accuracy}}
      &\multicolumn{4}{c|}{\textsc{Model Complexity}}
      &\multicolumn{2}{c}{\textsc{Trade-Off}}\\[-0.2ex]\rule{0pt}{2.6ex}
      &\multicolumn{2}{c|}{\text{Training MAE}}
      &\multicolumn{2}{c|}{\text{Testing MAE}}
      &\multicolumn{2}{c|}{\text{\#Rules} $(|\mathcal{P}|)$}
      &\multicolumn{2}{c|}{\text{\#Parameters} $(k)$}
      &\multicolumn{2}{c}{\text{AIC}}\\
      & XCSF & KACS & XCSF & KACS & XCSF & KACS & XCSF & KACS & XCSF & KACS\\
      \bhline{1pt}
      $f_1$ & \cellcolor{g}0.2033 $+$ & 0.2194 & 0.2642 $-$ & \cellcolor{g}0.2351 & 4162 $-$ & \cellcolor{g}1341 & 45780 $-$ & \cellcolor{g}2681 & 89210 $-$ & \cellcolor{g}3024 \\
      $f_2$ & \cellcolor{g}0.1124 $\sim$ & 0.1157 & 0.1564 $-$ & \cellcolor{g}0.1253 & 4087 $-$ & \cellcolor{g}1462 & 44960 $-$ & \cellcolor{g}2924 & 86580 $-$ & \cellcolor{g}2342 \\
      $f_3$ & 0.2510 $-$ & \cellcolor{g}0.2097 & 0.3246 $-$ & \cellcolor{g}0.2329 & 4130 $-$ & \cellcolor{g}1429 & 45430 $-$ & \cellcolor{g}2857 & 88890 $-$ & \cellcolor{g}3324 \\
      $f_4$ & 0.1680 $-$ & \cellcolor{g}0.1334 & 0.2303 $-$ & \cellcolor{g}0.1430 & 4170 $-$ & \cellcolor{g}1414 & 45870 $-$ & \cellcolor{g}2829 & 89080 $-$ & \cellcolor{g}2412 \\
      ASN & 0.1434 $-$ & \cellcolor{g}0.1094 & 0.1478 $-$ & \cellcolor{g}0.1144 & 2224 $-$ & \cellcolor{g}1417 & 13350 $-$ & \cellcolor{g}2834 & 22270 $-$ & \cellcolor{g}371.1 \\
      CCPP & 0.0930 $\sim$ & \cellcolor{g}0.0918 & 0.0929 $\sim$ & \cellcolor{g}0.0926 & 1953 $\sim$ & \cellcolor{g}1924 & 9764 $-$ & \cellcolor{g}3848 & -17360 $-$ & \cellcolor{g}-29070 \\
      CS & 0.1299 $\sim$ & \cellcolor{g}0.1213 & 0.1350 $\sim$ & \cellcolor{g}0.1288 & 2985 $-$ & \cellcolor{g}1425 & 26870 $-$ & \cellcolor{g}2851 & 50450 $-$ & \cellcolor{g}2237 \\
      EEC & \cellcolor{g}0.0820 $\sim$ & 0.0967 & \cellcolor{g}0.0847 $\sim$ & 0.1020 & 2925 $-$ & \cellcolor{g}1012 & 26330 $-$ & \cellcolor{g}2023 & 49850 $-$ & \cellcolor{g}1125 \\
      \bhline{1pt}
      Rank & \textit{1.62}$\downarrow$ & \cellcolor{g}\textit{1.38} & \textit{1.88}$\downarrow^{\dag}$ & \cellcolor{g}\textit{1.12} & \textit{2.00}$\downarrow^{\dag}$ & \cellcolor{g}\textit{1.00} & \textit{2.00}$\downarrow^{\dag}$ & \cellcolor{g}\textit{1.00} & \textit{2.00}$\downarrow^{\dag}$ & \cellcolor{g}\textit{1.00} \\
      $+/-/\sim$ & 1/3/4 & - & 0/5/3 & - & 0/7/1 & - & 0/8/0 & - & 0/8/0 & - \\
      $p$-value & 0.383 & - & 0.0391 & - & 0.00781 & - & 0.00781 & - & 0.00781 & - \\
      \bhline{1pt}
  \end{tabular}}
  \endgroup

  \centering
  \subfloat[$f_1$: Rastrigin Function]{
    \includegraphics[width=0.24\textwidth]{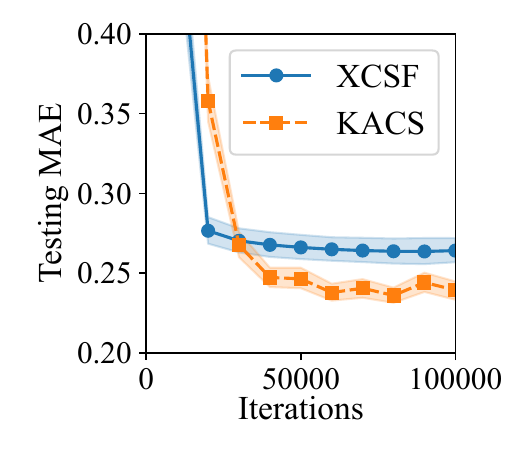}
    \label{fig:f1_mae}
  }
  \subfloat[$f_2$: Rosenbrock Function]{\includegraphics[width=0.24\textwidth]{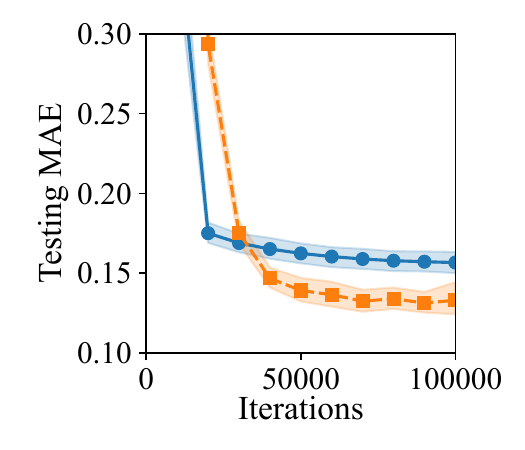} \label{fig:f2_mae}}
  \subfloat[$f_3$: Cross Function]{\includegraphics[width=0.24\textwidth]{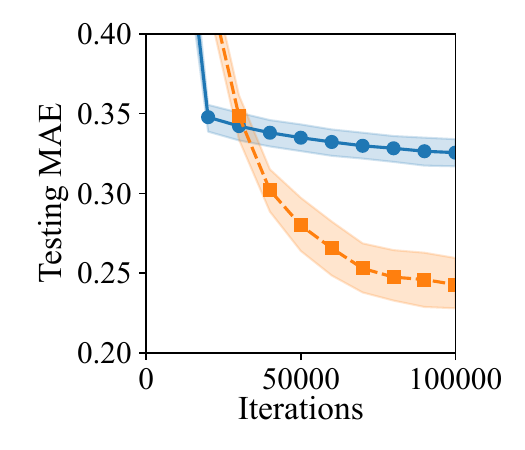} \label{fig:f3_mae}}
  \subfloat[$f_4$: Styblinski-Tang Function]{\includegraphics[width=0.24\textwidth]{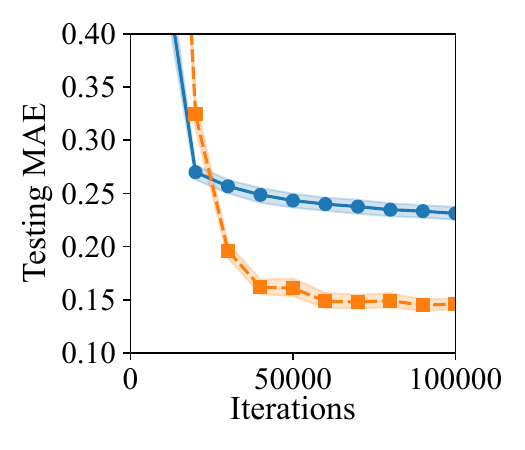} \label{fig:f4_mae}}
  \caption{Testing MAE learning curves for all {four synthetic functions}. Curves show the mean over 30 runs, and shaded regions denote 95\% confidence intervals.}
  \label{fig:all_mae}
  \label{others-8}
\end{figure*}

\subsubsection{Hyperparameter Settings and Protocol}
\label{sss: Hyperparameter Settings and Protocol}
We use the same hyperparameter settings for both XCSF and KACS. A detailed description of the hyperparameters is provided in Section~\ref{sec:Descriptions of Hyperparameters sup} of the supplementary material. Unless otherwise stated, the main hyperparameters are set as follows: $N=6400$, $\epsilon_0=0.01$, $\beta=0.2$, $\alpha=1$, $\nu=1$, $\delta=0.1$, $m_0=0.1$, $r_0=1.0$, $\{\theta_{\rm del},\theta_{\rm sub},\theta_{\rm GA}\}=50$, $\chi=0.8$, $\mu=0.04$, $\tau=0.4$, where most values are typical defaults (e.g.,~\cite{stein2018interpolation,preen2021autoencoding,shiraishi2025xkan}). For synthetic test functions, we set $P_\# = 0$ \cite{stein2018interpolation} so that covering always generates local rules, which quickly tile the uniformly sampled input space. For real-world datasets, we set $P_\# = 0.8$ \cite{nakata2020learning} so that covering starts from globally-covering rules. This setting is more stable under highly non-uniform input distributions, which are frequently observed in real-world situations.

We run each method for 100,000 learning iterations and report results averaged over 30 independent runs. For each run, we use Monte Carlo cross-validation with 90\% training data and 10\% testing data, following \cite{preen2021autoencoding}. {Unlike online function-approximation protocols that continue learning during testing after a warm-up, we evaluate unseen test samples with learning disabled.}\label{r3-18} {For each split, min--max scaling parameters are estimated from the training data only and then applied to the test data, mapping inputs to $[0,1]^n$ and targets to $[-1,1]$. Thus, the reported MAEs are on the same normalized target scale across all problems.}\label{r2-9} We use the Wilcoxon signed-rank test with significance level $\alpha=0.05$. {No post-training rule-compaction \cite{shiraishi2025xkan} or condensation \cite{butz2008function} was applied.}\label{r2-16} All experiments were conducted using our Julia implementation, which is publicly available at \url{https://github.com/YNU-NakataLab/KACS}.\label{r3-17}

\subsection{Main Results and Discussion}
\label{subsec:exp_results}
Table~\ref{tb:result} reports the averages over 30 runs of model accuracy (training MAE and testing MAE), model complexity (number of rules and parameters), and their trade-off (AIC) at the end of learning for XCSF and KACS. {Section~\ref{sec:supplementary-statistics} of the supplementary material provides the corresponding standard deviations, paired Wilcoxon $p$-values, and effect sizes for every benchmark and metric.}\label{r2-11-2}\label{r1-3-1}

From Table~\ref{tb:result}, we can see that there is no statistically significant difference in training MAE between XCSF and KACS ($p=0.383$). In contrast, KACS is significantly better than XCSF in testing MAE, number of rules, number of parameters, and AIC ($p<0.05$).

{To assess the sensitivity of our conclusions to the evaluation protocol, we additionally compared XCSF and KACS using a protocol based on Heider's evaluation design~\cite{heider2023suprb}, with target standardization, MSE, a 75/25 train--test split, and eight Monte Carlo splits combined with eight random seeds per split. Under this alternative protocol, KACS achieved significantly lower training and testing MSE than XCSF on all eight benchmarks, while using fewer parameters and yielding lower AIC in every case. The details are provided in Section~\ref{sec:alternative_protocol_sup} of the supplementary material.}
\label{r2-9-2}

\subsubsection{Model Accuracy}
\label{sss:model_accuracy}
Fig.~\ref{fig:all_mae} shows the testing MAE learning curves across all {four synthetic functions. The training and testing MAE learning curves across all eight problems are provided in Sections~\ref{ss:Training MAE Learning Curves sup} and \ref{ss:Testing MAE Learning Curves sup} of the supplementary material, respectively.}\label{others-9} From Table~\ref{tb:result} and Fig.~\ref{fig:all_mae}, KACS tends to achieve lower testing MAE than XCSF as training progresses, except on near-linear problems where both methods perform comparably. This indicates that reducing the problem to a set of one-dimensional subproblems allows KACS to improve prediction accuracy more efficiently than XCSF, which optimizes rules directly in the full $n$-dimensional space.\label{others-10}

The magnitude of this performance gap depends on the problem characteristics, which are summarized below:
\begin{itemize}
  \item All four artificial functions $f_1$--$f_4$ are strongly nonlinear 10-dimensional problems. $f_1$ has strong local complexity in the form of multimodality. $f_2$ shares this local complexity in the form of a steep curved valley and also exhibits interactions between adjacent dimensions. $f_3$ has interactions across all dimensions. $f_4$ is challenging because function values rise steeply near the edges as dimensionality increases \cite{stein2018interpolation}.
  \item Among the real-world datasets, ASN is the hardest to predict due to its strong nonlinearity \cite{heider2023suprb}. CS is also nonlinear, but somewhat easier to predict \cite{heider2023suprb}. CCPP and EEC are relatively straightforward problems with stronger linearity, where even simple linear models can achieve adequate prediction \cite{heider2023suprb}.
\end{itemize}

On all four artificial functions, KACS achieved significantly lower testing MAE than XCSF. $f_1$ is a notable case: XCSF recorded significantly lower training MAE, yet worse testing MAE than KACS. Since XCSF used roughly 17 times more parameters on this problem (45780 for XCSF vs.\ 2681 for KACS), this pattern points to overfitting driven by multi-dimensional rules on a multimodal function. A similar pattern appears on $f_2$: the gap between training and testing MAE is much wider for XCSF ($0.1124 \to 0.1564$, $\Delta = 0.044$) than for KACS ($0.1157 \to 0.1253$, $\Delta = 0.010$), suggesting that steep curved valleys cause mild overfitting for $n$-dimensional rules. On $f_3$, where interactions span all dimensions simultaneously, KACS outperformed XCSF in both training and testing MAE. XCSF must use $n$-dimensional rules to capture such interactions implicitly, whereas KACS {represents them through nonlinear outer functions that aggregate all dimensions.}\label{r3-6-2} On $f_4$, the difficulty lies in approximating steep edge regions in 10-dimensional space, and KACS again showed a clear advantage.

Among the real-world datasets, KACS achieved significantly lower testing MAE than XCSF only on ASN. For CS and CCPP, no statistically significant differences were found. For EEC, KACS showed higher testing MAE than XCSF (0.1020 vs. 0.0847), but the difference was not statistically significant. For CCPP and EEC, $n$-dimensional linear rules are already sufficient (testing MAE below 0.1). In such near-linear problems, the KA decomposition introduces unnecessary structural overhead: KACS must coordinate $n(2n+1)$ inner submodels and compose their outputs through $2n+1$ outer submodels to recover a function that XCSF can represent directly with multiple linear rules. This additional indirection accumulates approximation error in the outer stage, which outweighs the benefits of dimensional decomposition.

In Fig.~\ref{fig:all_mae}, KACS started with higher testing MAE than XCSF in the early phase of training (0--20000 iterations) across all problems. This is because KACS needs to initialize and tune many one-dimensional submodels at once, which produces inaccurate predictions early on and temporarily inflates the overall error. The GA-driven search is also distributed across all {$2n^2+3n+1$} submodels at the start, which slows convergence of the final output compared to XCSF.\label{others-13}

The above findings can be summarized as follows:
\begin{itemize}
  \item The accuracy advantage of KACS grows with problem nonlinearity, particularly when cross-dimensional interactions or sharp local structures are present, and fades when the problem is well-captured by linear approximations.
  \item In all cases, a warm-up period is needed at the start of KACS training, as coordinating many one-dimensional submodels from scratch inevitably raises early-stage error.
\end{itemize}

Finally, {we evaluate rotated, translated, and scaled variants of $f_3$ to test sensitivity to ridge orientation, location, and scale. As reported in Section~\ref{sec:rotation_sup} of the supplementary material, KACS achieves significantly lower testing MAE than XCSF for all three variants, indicating that its advantage is not limited to the original axis-aligned structure of $f_3$.}\label{r2-13-1}

\subsubsection{Model Complexity and Accuracy-Complexity Trade-off}

One of the most significant results in this article is the large difference in model complexity. As shown in Table~\ref{tb:result}, KACS achieved comparable or higher prediction accuracy using only roughly 6--40\% of the parameters required by XCSF across all eight problems. The parameter-count learning curves are provided in Section~\ref{ss:Parameter-Count Learning Curves sup} of the supplementary material.

This gap reflects a fundamental difference in how each method constructs its rules. XCSF covers the input space with local linear models, each defined over the full $n$-dimensional space. As $n$ grows, defining meaningful local regions and fitting reliable linear models within them both become harder, which tends to produce redundant or poorly fitted rules, driving up model size and the risk of overfitting (cf. Section~\ref{sss:model_accuracy}). KACS avoids this by decomposing the problem into one-dimensional subproblems, allowing it to represent multivariate functions compactly with far fewer parameters.

KACS achieved significantly lower AIC than XCSF on all eight problems, indicating that its models are statistically more efficient. On CS, where testing MAE is similar for both methods, the AIC advantage of KACS comes entirely from its smaller parameter count. On EEC, where KACS records slightly higher testing MAE than XCSF, KACS still achieves far better AIC, showing that its reduction in model complexity more than compensates for the small loss in accuracy.

\subsection{Scalability Analysis}
\label{subsec:scalability}

This subsection analyzes the scalability of XCSF and KACS. We use $f_3$ (Cross function) as the representative benchmark because it is the only synthetic function in our set that involves interactions across all dimensions, making it directly sensitive to changes in $n$.\label{others-6}

\subsubsection{Experimental Design}
\label{sss:scalability_design}
Unless otherwise noted, hyperparameter settings and training iterations follow the setup in Section~\ref{subsec:exp_setup}. We report testing MAE and two complexity measures, $|\mathcal{P}|$ and $k$, as means with 95\% confidence intervals.

\subsubsection{Results on Sample-Size Scaling}
\label{sss:scalability_ns}

We vary $N_S\in\{125,250,500,1000,2000,4000,8000\}$ while fixing $n = 10$. Fig.~\ref{fig:scaling_ns} summarizes how both methods respond to more data. The detailed tabular results are provided in Section~\ref{ss:Sample-Size Scaling sup} of the supplementary material.

Increasing $N_S$ improves accuracy for both methods, but the gain is larger for KACS. XCSF testing MAE decreases from 0.416 ($N_S=125$) to 0.312 ($N_S=8000$), while KACS decreases from 0.380 to 0.216 and stabilizes after $N_S \approx 2000$. Complexity grows much more slowly for KACS: XCSF rules and parameters increase from 3164 to 4724 and from 34800 to 51960, respectively, whereas KACS stays around 1242--1430 rules and 2484--2860 parameters. Thus, more data mainly improves KACS accuracy without inducing comparable complexity growth.

\subsubsection{Results on Dimensional Scaling}
\label{sss:scalability_n}
\begin{figure*}[!t]
  \centering
  \pgfplotslegendfromname{sharedlegend}
  \\[-10pt]
  \subfloat[Testing MAE vs. $N_S$]{
    \begin{tikzpicture}
      \begin{axis}[
          width=5.8cm, height=3.0cm,
          xlabel={\#Samples $N_S$},
          ylabel={Testing MAE},
          ytick={0.2,0.3,0.4,0.5},
          xmode=log,
          xmin=100, xmax=10000,
          ymin=0.2, ymax=0.5,
          xticklabel style={font=\small},
          yticklabel style={font=\small},
          xlabel style={font=\small},
          ylabel style={font=\small},
          legend style={
            font=\small,
            at={(0.98,0.98)},
            anchor=north east,
            draw=black,
            fill=white,
          },
        ]

        \addplot[name path=s1_upper, draw=none, forget plot]
        table[x=trials, y=upper, col sep=comma]{fig/xcsf_sam_test.csv};

        \addplot[name path=s1_lower, draw=none, forget plot]
        table[x=trials, y=lower, col sep=comma]{fig/xcsf_sam_test.csv};

        \addplot[myblue!30, forget plot]
        fill between[of=s1_upper and s1_lower];

        \addplot[
          myblue, thick,
          mark=*, mark size=2pt,
          mark options={fill=myblue},
        ]
        table[x=trials, y=mean, col sep=comma]{fig/xcsf_sam_test.csv};

        \addplot[name path=s2_upper, draw=none, forget plot]
        table[x=trials, y=upper, col sep=comma]{fig/kacs_sam_test.csv};

        \addplot[name path=s2_lower, draw=none, forget plot]
        table[x=trials, y=lower, col sep=comma]{fig/kacs_sam_test.csv};

        \addplot[myorange!30, forget plot]
        fill between[of=s2_upper and s2_lower];

        \addplot[
          myorange, thick, dashed,
          mark=square*, mark size=2.6pt,
          mark options={solid, draw=myorange, fill=myorange, line width=0.5pt},
        ]
        table[x=trials, y=mean, col sep=comma]{fig/kacs_sam_test.csv};

      \end{axis}
    \end{tikzpicture}
  }
  \subfloat[\#Rules vs. $N_S$]{
    \begin{tikzpicture}
      \begin{axis}[
          width=5.8cm, height=3.0cm,
          xlabel={\#Samples $N_S$},
          ylabel={$\#$Rules $|\mathcal{P}|$},
          xmode=log,
          xmin=100, xmax=10000,
          ymin=1000, ymax=7000,
          ytick={1000,3000,5000,7000},
          xticklabel style={font=\small},
          yticklabel style={font=\small},
          xlabel style={font=\small},
          ylabel style={font=\small},
        ]

        \addplot[name path=s1_upper, draw=none, forget plot]
        table[x=trials, y=upper, col sep=comma]{fig/xcsf_sam_pop.csv};

        \addplot[name path=s1_lower, draw=none, forget plot]
        table[x=trials, y=lower, col sep=comma]{fig/xcsf_sam_pop.csv};

        \addplot[myblue!30, forget plot]
        fill between[of=s1_upper and s1_lower];

        \addplot[
          myblue, thick,
          mark=*, mark size=2pt,
          mark options={fill=myblue},
          forget plot,
        ]
        table[x=trials, y=mean, col sep=comma]{fig/xcsf_sam_pop.csv};

        \addplot[name path=s2_upper, draw=none, forget plot]
        table[x=trials, y=upper, col sep=comma]{fig/kacs_sam_pop.csv};

        \addplot[name path=s2_lower, draw=none, forget plot]
        table[x=trials, y=lower, col sep=comma]{fig/kacs_sam_pop.csv};

        \addplot[myorange!30, forget plot]
        fill between[of=s2_upper and s2_lower];

        \addplot[
          myorange, thick, dashed,
          mark=square*, mark size=2.6pt,
          mark options={solid, draw=myorange, fill=myorange, line width=0.5pt},
          forget plot,
        ]
        table[x=trials, y=mean, col sep=comma]{fig/kacs_sam_pop.csv};

        \addplot[black, dashed, thick, forget plot]
        coordinates {(100,6400)(10000,6400)};

        \node[anchor=north west, font=\small]
        at (axis cs:100, 6400) {$N=6,400$};

      \end{axis}
    \end{tikzpicture}
  }
  \subfloat[\#Parameters vs. $N_S$]{
    \begin{tikzpicture}
      \begin{axis}[
          width=5.8cm, height=3.0cm,
          xlabel={\#Samples $N_S$},
          ylabel={$\#$Parameters $k$},
          xmode=log,
          xmin=100, xmax=10000,
          ymode=log,
          ymin=1000, ymax=100000,
          xticklabel style={font=\small},
          yticklabel style={font=\small},
          xlabel style={font=\small},
          ylabel style={font=\small},
          legend entries={},
        ]

        \addplot[name path=s1_upper, draw=none, forget plot]
        table[x=trials, y=upper, col sep=comma]{fig/xcsf_sam_nop.csv};

        \addplot[name path=s1_lower, draw=none, forget plot]
        table[x=trials, y=lower, col sep=comma]{fig/xcsf_sam_nop.csv};

        \addplot[myblue!30, forget plot]
        fill between[of=s1_upper and s1_lower];

        \addplot[
          myblue, thick,
          mark=*, mark size=2pt,
          mark options={fill=myblue},
        ]
        table[x=trials, y=mean, col sep=comma]{fig/xcsf_sam_nop.csv};

        \addplot[name path=s2_upper, draw=none, forget plot]
        table[x=trials, y=upper, col sep=comma]{fig/kacs_sam_nop.csv};

        \addplot[name path=s2_lower, draw=none, forget plot]
        table[x=trials, y=lower, col sep=comma]{fig/kacs_sam_nop.csv};

        \addplot[myorange!30, forget plot]
        fill between[of=s2_upper and s2_lower];

        \addplot[
          myorange, thick, dashed,
          mark=square*, mark size=2.6pt,
          mark options={solid, draw=myorange, fill=myorange, line width=0.5pt},
        ]
        table[x=trials, y=mean, col sep=comma]{fig/kacs_sam_nop.csv};

      \end{axis}
    \end{tikzpicture}
  }
  \caption{Sample-size scaling results for XCSF and KACS on $f_3$, where $N_S\in\{125,250,500,1000,2000,4000,8000\}$ and $n=10$.}
  \label{fig:scaling_ns}
\end{figure*}
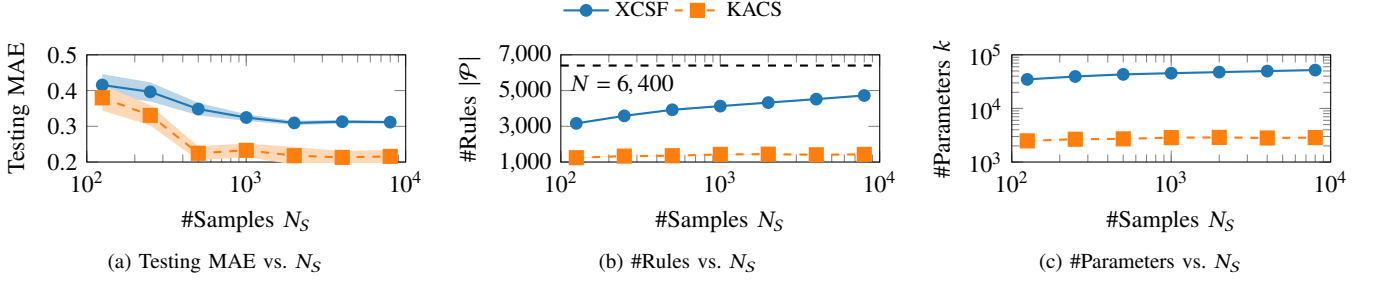
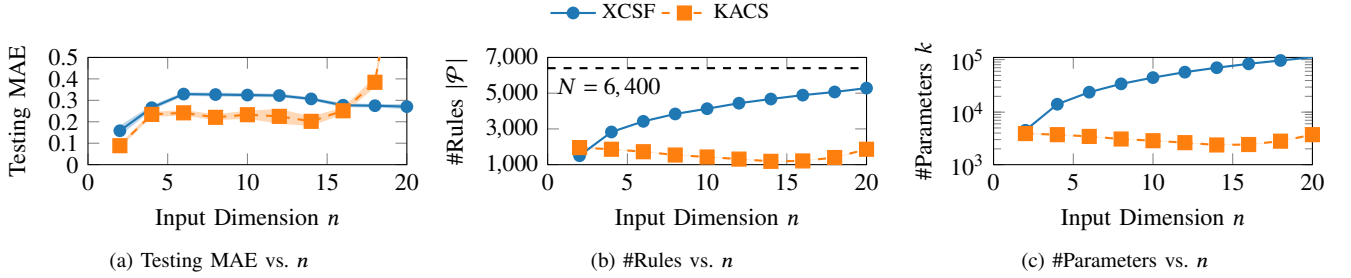
\begin{figure*}[!t]
  \centering
  \pgfplotslegendfromname{sharedlegend}
  \\[-10pt]
  \subfloat[Testing MAE vs. $n$]{
    \begin{tikzpicture}
      \begin{axis}[
          width=5.8cm, height=3.0cm,
          xlabel={Input Dimension $n$},
          ylabel={Testing MAE},
          ytick={0,0.1,0.2,0.3,0.4,0.5},
          xmin=0, xmax=20,
          ymin=0.0, ymax=0.5,
          xticklabel style={font=\small},
          yticklabel style={font=\small},
          xlabel style={font=\small},
          ylabel style={font=\small},
          legend style={
            font=\small,
            at={(0.98,0.02)},
            anchor=south east,
            draw=black,
            fill=white,
          },
        ]

        \addplot[name path=s1_upper, draw=none, forget plot]
        table[x=trials, y=upper, col sep=comma]{fig/xcsf_dim_test.csv};

        \addplot[name path=s1_lower, draw=none, forget plot]
        table[x=trials, y=lower, col sep=comma]{fig/xcsf_dim_test.csv};

        \addplot[myblue!30, forget plot]
        fill between[of=s1_upper and s1_lower];

        \addplot[
          myblue, thick,
          mark=*, mark size=2pt,
          mark options={fill=myblue},
        ]
        table[x=trials, y=mean, col sep=comma]{fig/xcsf_dim_test.csv};

        \addplot[name path=s2_upper, draw=none, forget plot]
        table[x=trials, y=upper, col sep=comma]{fig/kacs_dim_test.csv};

        \addplot[name path=s2_lower, draw=none, forget plot]
        table[x=trials, y=lower, col sep=comma]{fig/kacs_dim_test.csv};

        \addplot[myorange!30, forget plot]
        fill between[of=s2_upper and s2_lower];

        \addplot[
          myorange, thick,dashed,
          mark=square*, mark size=2.6pt,
          mark options={solid, draw=myorange, fill=myorange, line width=0.5pt},
        ]
        table[x=trials, y=mean, col sep=comma]{fig/kacs_dim_test.csv};

      \end{axis}
  \end{tikzpicture}}
  \subfloat[\#Rules vs. $n$]{
    \begin{tikzpicture}
      \begin{axis}[
          width=5.8cm, height=3.0cm,
          xlabel={Input Dimension $n$},
          ylabel={$\#$Rules $|\mathcal{P}|$},
          xmin=0, xmax=20,
          ymin=1000, ymax=7000,
          xtick={0,5,10,15,20},
          ytick={1000,3000,5000,7000},
          xticklabel style={font=\small},
          yticklabel style={font=\small},
          xlabel style={font=\small},
          ylabel style={font=\small},
        ]

        \addplot[name path=s1_upper, draw=none, forget plot]
        table[x=trials, y=upper, col sep=comma]{fig/xcsf_dim_pop.csv};

        \addplot[name path=s1_lower, draw=none, forget plot]
        table[x=trials, y=lower, col sep=comma]{fig/xcsf_dim_pop.csv};

        \addplot[myblue!30, forget plot]
        fill between[of=s1_upper and s1_lower];

        \addplot[
          myblue, thick,
          mark=*, mark size=2pt,
          mark options={fill=myblue},
          forget plot,
        ]
        table[x=trials, y=mean, col sep=comma]{fig/xcsf_dim_pop.csv};

        \addplot[name path=s2_upper, draw=none, forget plot]
        table[x=trials, y=upper, col sep=comma]{fig/kacs_dim_pop.csv};

        \addplot[name path=s2_lower, draw=none, forget plot]
        table[x=trials, y=lower, col sep=comma]{fig/kacs_dim_pop.csv};

        \addplot[myorange!30, forget plot]
        fill between[of=s2_upper and s2_lower];

        \addplot[
          myorange, thick, dashed,
          mark=square*, mark size=2.6pt,
          mark options={solid, draw=myorange, fill=myorange, line width=0.5pt},
          forget plot,
        ]
        table[x=trials, y=mean, col sep=comma]{fig/kacs_dim_pop.csv};

        \addplot[black, dashed, thick, forget plot]
        coordinates {(0,6400)(20,6400)};

        \node[anchor=north west, font=\small]
        at (axis cs:0, 6400) {$N=6,400$};

      \end{axis}
    \end{tikzpicture}
  }
  \subfloat[\#Parameters vs. $n$]{
    \begin{tikzpicture}
      \begin{axis}[
          width=5.8cm, height=3.0cm,
          xlabel={Input Dimension $n$},
          ylabel={$\#$Parameters $k$},
          xmin=0, xmax=20,
          ymode=log,
          ymin=1000, ymax=110000,
          xticklabel style={font=\small},
          yticklabel style={font=\small},
          xlabel style={font=\small},
          ylabel style={font=\small},
          legend entries={},
        ]

        \addplot[name path=s1_upper, draw=none, forget plot]
        table[x=trials, y=upper, col sep=comma]{fig/xcsf_dim_nop.csv};

        \addplot[name path=s1_lower, draw=none, forget plot]
        table[x=trials, y=lower, col sep=comma]{fig/xcsf_dim_nop.csv};

        \addplot[myblue!30, forget plot]
        fill between[of=s1_upper and s1_lower];

        \addplot[
          myblue, thick,
          mark=*, mark size=2pt,
          mark options={fill=myblue},
        ]
        table[x=trials, y=mean, col sep=comma]{fig/xcsf_dim_nop.csv};

        \addplot[name path=s2_upper, draw=none, forget plot]
        table[x=trials, y=upper, col sep=comma]{fig/kacs_dim_nop.csv};

        \addplot[name path=s2_lower, draw=none, forget plot]
        table[x=trials, y=lower, col sep=comma]{fig/kacs_dim_nop.csv};

        \addplot[myorange!30, forget plot]
        fill between[of=s2_upper and s2_lower];

        \addplot[
          myorange, thick, dashed,
          mark=square*, mark size=2.6pt,
          mark options={solid, draw=myorange, fill=myorange, line width=0.5pt},
        ]
        table[x=trials, y=mean, col sep=comma]{fig/kacs_dim_nop.csv};

      \end{axis}
  \end{tikzpicture}}

  \caption{Dimensional scaling results for XCSF and KACS on $f_3$, where $n\in\{2,4,6,8,10,12,14,16,18,20\}$ and $N_S=1000$.}
  \label{fig:scaling_n}
\end{figure*}

We vary $n\in\{2,4,6,8,10,12,14,16,18,20\}$ while fixing $N_S = 1000$. Fig.~\ref{fig:scaling_n} summarizes how both methods respond to increasing dimensionality. The detailed tabular results are provided in Section~\ref{ss:Dimensional Scaling sup} of the supplementary material.

As $n$ increases, XCSF testing MAE rises from 0.158 at $n=2$ to about 0.27--0.33 at $n=4$--20; KACS significantly outperforms XCSF from $n=2$ to $n=16$ (e.g., 0.233 vs.\ 0.325 at $n=10$), but degrades sharply at higher dimensions (0.384 at $n=18$, 1.25 at $n=20$). Complexity trends are clearer: XCSF rules grow monotonically (1501 to 5283), and its parameter count grows from 4504 to 110900. KACS remains compact, with 1186--1955 rules and 2372--3909 parameters across most settings, with only a notable increase at $n \ge 18$.

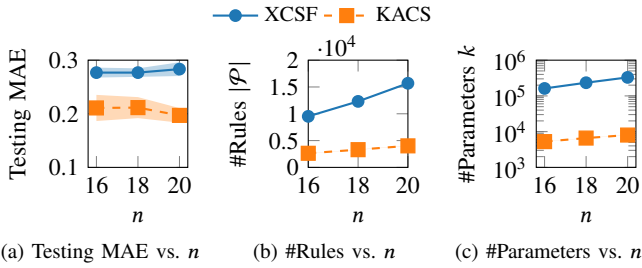
\begin{figure}[t]
  \centering
  \pgfplotslegendfromname{sharedlegend}
  \\[-10pt]

  \subfloat[Testing MAE vs. $n$]{
    \begin{tikzpicture}
      \begin{axis}[
          width=2.9cm, height=3.0cm,
          xlabel={$n$},
          ylabel={Testing MAE},
          ytick={0.1,0.2,0.3},
          ymin=0.1, ymax=0.3,
          xticklabel style={font=\small},
          yticklabel style={font=\small},
          xlabel style={font=\small},
          ylabel style={font=\small},
        ]
        \addplot[name path=s1_upper, draw=none, forget plot]
        table[x=trials, y=upper, col sep=comma]{fig/xcsf_guideline_test.csv};
        \addplot[name path=s1_lower, draw=none, forget plot]
        table[x=trials, y=lower, col sep=comma]{fig/xcsf_guideline_test.csv};
        \addplot[myblue!30, forget plot]
        fill between[of=s1_upper and s1_lower];
        \addplot[myblue, thick, mark=*, mark size=2pt, mark options={fill=myblue}]
        table[x=trials, y=mean, col sep=comma]{fig/xcsf_guideline_test.csv};

        \addplot[name path=s2_upper, draw=none, forget plot]
        table[x=trials, y=upper, col sep=comma]{fig/kacs_guideline_test.csv};
        \addplot[name path=s2_lower, draw=none, forget plot]
        table[x=trials, y=lower, col sep=comma]{fig/kacs_guideline_test.csv};
        \addplot[myorange!30, forget plot]
        fill between[of=s2_upper and s2_lower];
        \addplot[myorange, thick, dashed, mark=square*, mark size=2.6pt, mark options={solid, draw=myorange, fill=myorange, line width=0.5pt}]
        table[x=trials, y=mean, col sep=comma]{fig/kacs_guideline_test.csv};
      \end{axis}
    \end{tikzpicture}
  }
  \hfil
  \subfloat[\#Rules vs. $n$]{
    \begin{tikzpicture}
      \begin{axis}[
          width=2.9cm, height=3.0cm,
          xlabel={$n$},
          ylabel={\#Rules $|\mathcal{P}|$},
          ymin=0, ymax=20000,
          xmin=16, xmax=20,
          xtick={16,18,20},
          xticklabel style={font=\small},
          yticklabel style={font=\small},
          xlabel style={font=\small},
          ylabel style={font=\small},
          legend to name=sharedlegend,
          legend style={
            draw=none,
            font=\footnotesize,
            legend columns=-1,
            /tikz/every even column/.append style={xshift=0.5em},
          },
        ]
        \addplot[name path=s1_upper, draw=none, forget plot]
        table[x=trials, y=upper, col sep=comma]{fig/xcsf_guideline_pop.csv};
        \addplot[name path=s1_lower, draw=none, forget plot]
        table[x=trials, y=lower, col sep=comma]{fig/xcsf_guideline_pop.csv};
        \addplot[myblue!30, forget plot]
        fill between[of=s1_upper and s1_lower];
        \addplot[myblue, thick, mark=*, mark size=2pt, mark options={fill=myblue}]
        table[x=trials, y=mean, col sep=comma]{fig/xcsf_guideline_pop.csv};
        \addlegendentry{XCSF}

        \addplot[name path=s2_upper, draw=none, forget plot]
        table[x=trials, y=upper, col sep=comma]{fig/kacs_guideline_pop.csv};
        \addplot[name path=s2_lower, draw=none, forget plot]
        table[x=trials, y=lower, col sep=comma]{fig/kacs_guideline_pop.csv};
        \addplot[myorange!30, forget plot]
        fill between[of=s2_upper and s2_lower];
        \addplot[myorange, thick, dashed, mark=square*, mark size=2.6pt, mark options={solid, draw=myorange, fill=myorange, line width=0.5pt}]
        table[x=trials, y=mean, col sep=comma]{fig/kacs_guideline_pop.csv};
        \addlegendentry{KACS}
      \end{axis}
    \end{tikzpicture}
  }
  \hfil
  \subfloat[\#Parameters vs. $n$]{
    \begin{tikzpicture}
      \begin{axis}[
          width=2.9cm, height=3.0cm,
          xlabel={$n$},
          ylabel={\#Parameters $k$},
          ymode=log,
          ymin=1000, ymax=1000000,
          xtick={16,18,20},
          xticklabel style={font=\small},
          yticklabel style={font=\small},
          xlabel style={font=\small},
          ylabel style={font=\small},
        ]
        \addplot[name path=s1_upper, draw=none, forget plot]
        table[x=trials, y=upper, col sep=comma]{fig/xcsf_guideline_nop.csv};
        \addplot[name path=s1_lower, draw=none, forget plot]
        table[x=trials, y=lower, col sep=comma]{fig/xcsf_guideline_nop.csv};
        \addplot[myblue!30, forget plot]
        fill between[of=s1_upper and s1_lower];
        \addplot[myblue, thick, mark=*, mark size=2pt, mark options={fill=myblue}]
        table[x=trials, y=mean, col sep=comma]{fig/xcsf_guideline_nop.csv};

        \addplot[name path=s2_upper, draw=none, forget plot]
        table[x=trials, y=upper, col sep=comma]{fig/kacs_guideline_nop.csv};
        \addplot[name path=s2_lower, draw=none, forget plot]
        table[x=trials, y=lower, col sep=comma]{fig/kacs_guideline_nop.csv};
        \addplot[myorange!30, forget plot]
        fill between[of=s2_upper and s2_lower];
        \addplot[myorange, thick, dashed, mark=square*, mark size=2.6pt, mark options={solid, draw=myorange, fill=myorange, line width=0.5pt}]
        table[x=trials, y=mean, col sep=comma]{fig/kacs_guideline_nop.csv};
      \end{axis}
    \end{tikzpicture}
  }

  \caption{Dimensional scaling results for XCSF and KACS on $f_3$, where $n\in\{16,18,20\}$, $N\in\{12672,15984,19680\}$, and $N_S=1000$.}
  \label{fig:scaling_n_guideline}
\end{figure}

\subsubsection{Discussion}
\label{sss:scalability_discussion}

Figs.~\ref{fig:scaling_ns} and~\ref{fig:scaling_n} together support two claims. First, KACS provides a better accuracy-complexity trade-off in the practical regime ($n \le 14$ and all tested sample sizes), consistently using far fewer rules and parameters than XCSF. Second, under the default $N=6400$, KACS shows a clear performance drop in higher dimensions ($n=16$--20): accuracy deteriorates sharply even though the model remains compact.

The failure at higher $n$ stems from two related capacity bottlenecks. First, KACS decomposes the $n$-dimensional problem into $(2n+1)(n+1) = 2n^2+3n+1$ one-dimensional subproblems, so the number of macro-rules per subproblem shrinks as $n$ grows. At $n=20$, for instance, this gives 861 subproblems, leaving at most $\lfloor 6400/861 \rfloor \approx 7$ macro-rules per subproblem on average, and typically fewer because individual rules can carry numerosity greater than one. Second, and more critically, the inner and outer submodels have fundamentally different input domains. The inner submodel input domain is fixed at $[0,1]$ regardless of $n$, whereas the outer submodel input $\hat{z}_q = \sum_{p=1}^{n}\hat{\psi}_{q,p}(x_p)$ is a sum of $n$ terms. Since inner rule weights are initialized from $\mathcal{U}[-1,1]$, the output of each $\hat{\psi}_{q,p}$ can be negative, and the outer input domain can be written as $\hat{\mathcal{K}}_z^{(q)} := \{\hat{z}_q \mid \hat{z}_q = \sum_{p=1}^{n}\hat{\psi}_{q,p}(x_p),\ x_p \in [0,1]\} \approx [-n,n]$, so its domain width scales as $|\hat{\mathcal{K}}_z^{(q)}| \approx 2n$. Consequently, outer submodels require $\mathcal{O}(n)$ times more rules than inner submodels to achieve the same coverage density. When $N$ is set without accounting for this difference, outer submodels are severely under-resourced at large $n$, triggering a \textit{cover-delete cycle}\footnote{{A detailed analysis of covering dynamics and cover-delete behavior is provided in Section \ref{sec:covering_dynamics_sup} of the supplementary material.}\label{r2-12-1}} in which newly generated covering rules are immediately deleted to maintain the population within $N$ \cite{butz2004toward}.

A guideline for $N$ should therefore distinguish between inner and outer coverage requirements. Let $k_{\min}$ denote the minimum number of macro-rules per unit of input domain width. Inner submodels, each with domain $[0,1]$ of width 1, require $k_{\min}$ rules. Outer submodels, with domain $\approx [-n,n]$ of width $2n$, require $k_{\min} \cdot 2n$ rules. With $n(2n+1)$ inner submodels and $(2n+1)$ outer submodels, the total rule budget must satisfy
\begin{equation}
  N \;\ge\; k_{\min} \cdot n(2n+1) + k_{\min} \cdot 2n \cdot (2n+1)
  \;=\; 3k_{\min} \cdot n(2n+1).
  \label{eq:N_guideline}
\end{equation}
We estimate $k_{\min}$ from the last stable dimension $n=14$:
\begin{equation}
  k_{\min} \;\approx\; \frac{N}{3 \cdot n(2n+1)}\bigg|_{n=14}
  = \frac{6400}{3 \times 14 \times 29} \approx 5.3.
  \nonumber
\end{equation}
Rounding up and adding a modest safety margin to account for problem-to-problem variation in the outer function range, we adopt $k_{\min} = 8$, which gives $N \ge 3 \times 8 \times n(2n+1) = 24n(2n+1)$, predicting $N \ge 12672, 15984, 19680$ for $n = 16, 18, 20$, respectively. To verify this, we ran additional experiments with these $N$ values at the corresponding dimensions. As shown in Fig.~\ref{fig:scaling_n_guideline}, under this scaled budget, the severe degradation at $n\in\{16,18,20\}$ no longer appears. KACS attains significantly lower testing MAE than XCSF at all three dimensions (16: $0.2109$ vs. $0.2768$, 18: $0.2115$ vs. $0.2768$, 20: $0.1970$ vs. $0.2833$). KACS also remains much more compact, using significantly fewer macro-rules ($|\mathcal{P}|$: 2633, 3306, 4017 vs. 9520, 12310, 15690) and significantly fewer parameters ($k$: 5266, 6612, 8033 vs. 161800, 233900, 329600). In particular, at $n=20$, KACS achieves lower testing MAE than XCSF using only about $2\%$ as many parameters. This confirms that the higher-dimensional failure is a capacity issue attributable to under-resourced outer submodels, rather than an inherent limitation of KACS.

\label{others-12}
{Note} that $k_{\min}$ is not a universal constant: its value depends on the actual range of the outer submodel inputs, which varies with the problem and the learned inner functions. The estimate $k_{\min} \approx 5.3$ was derived empirically from a single benchmark ($f_3$), so \eqref{eq:N_guideline} should be treated as a practical starting point rather than a tight bound; {its transferability to other functions and real-world data remains to be validated.}\label{r2-18}

\section{Conclusion}
\label{sec:conclusion}

In this article, we proposed KACS, the first LCS that reorganizes rules dimension-wise according to the Kolmogorov-Arnold (KA) representation theorem. Instead of evolving rules directly in the $n$-dimensional input space, KACS decomposes the target function into one-dimensional subproblems and assigns a dedicated ruleset to each, optimized through evolutionary algorithms and gradient-based methods. This reduces the worst-case rule count from $\mathcal{O}(m^n)$ to $\mathcal{O}(mn^2)$, where $m$ is the per-variable resolution. We also proved that KACS is a universal approximator for continuous functions on compact domains, which is the first such proof for any LCS. Experiments showed that { KACS provides a parameter-efficient alternative for function approximation, with competitive accuracy across many tested settings.}\label{r1-1-2}

These results suggest that KA-guided dimension-wise decomposition is a broadly applicable design principle. Any local modeling framework that partitions the input space directly, such as fuzzy systems and other rule-based systems, may benefit from reorganizing its components into one-dimensional subproblems. KACS provides a concrete example of this principle and shows that such reorganization improves both accuracy and model compactness without sacrificing expressive power, as guaranteed by the universal approximation proof.

Several limitations remain. First, initializing many submodels simultaneously requires a warm-up period, during which prediction error temporarily exceeds that of XCSF. Second, reliable coverage requires a population budget of $N = \mathcal{O}(n^2)$; a smaller budget triggers a cover-delete cycle that degrades performance. Third, the universal approximation proof guarantees only the \textit{existence} of an accurate ruleset, not that the learning algorithm will converge to it in practice.

Future work should address these issues in order of practical impact. An adaptive mechanism that monitors submodel coverage and adjusts $N$ dynamically would resolve both the warm-up and budget problems. On the theoretical side, deriving convergence rates for the joint GA and gradient-based optimization would connect the existence proof to algorithmic behavior. At a broader level, extending dimension-wise decomposition to other local modeling frameworks and identifying the conditions under which it outperforms direct partitioning is a direction that reaches beyond LCSs, and KACS provides a concrete foundation for that investigation. Within this direction, relaxing the strict KA structure may further improve approximation efficiency. For example, one can vary the number of inner and outer models, or stack multiple KA layers as in KAN~\cite{liu2025kan}. Whether universal approximation is preserved under these modifications remains an open question. {Future studies will also visualize the learned rules and submodels, evaluate structural fidelity by comparing learned models with known low-dimensional target functions, and compare KACS with KAN-based models~\cite{shiraishi2025xkan}, genetic programming variants~\cite{koza1994genetic}, and neural- and code-fragment-based LCSs~\cite{bull2002accuracy,arif2017solving} across broader problem suites.}\label{r2-5}\label{r3-21}\label{r3-3}\label{r3-1-2}


\begin{IEEEbiography}
  [{\includegraphics[width=1in,height=1.25in,clip,keepaspectratio]{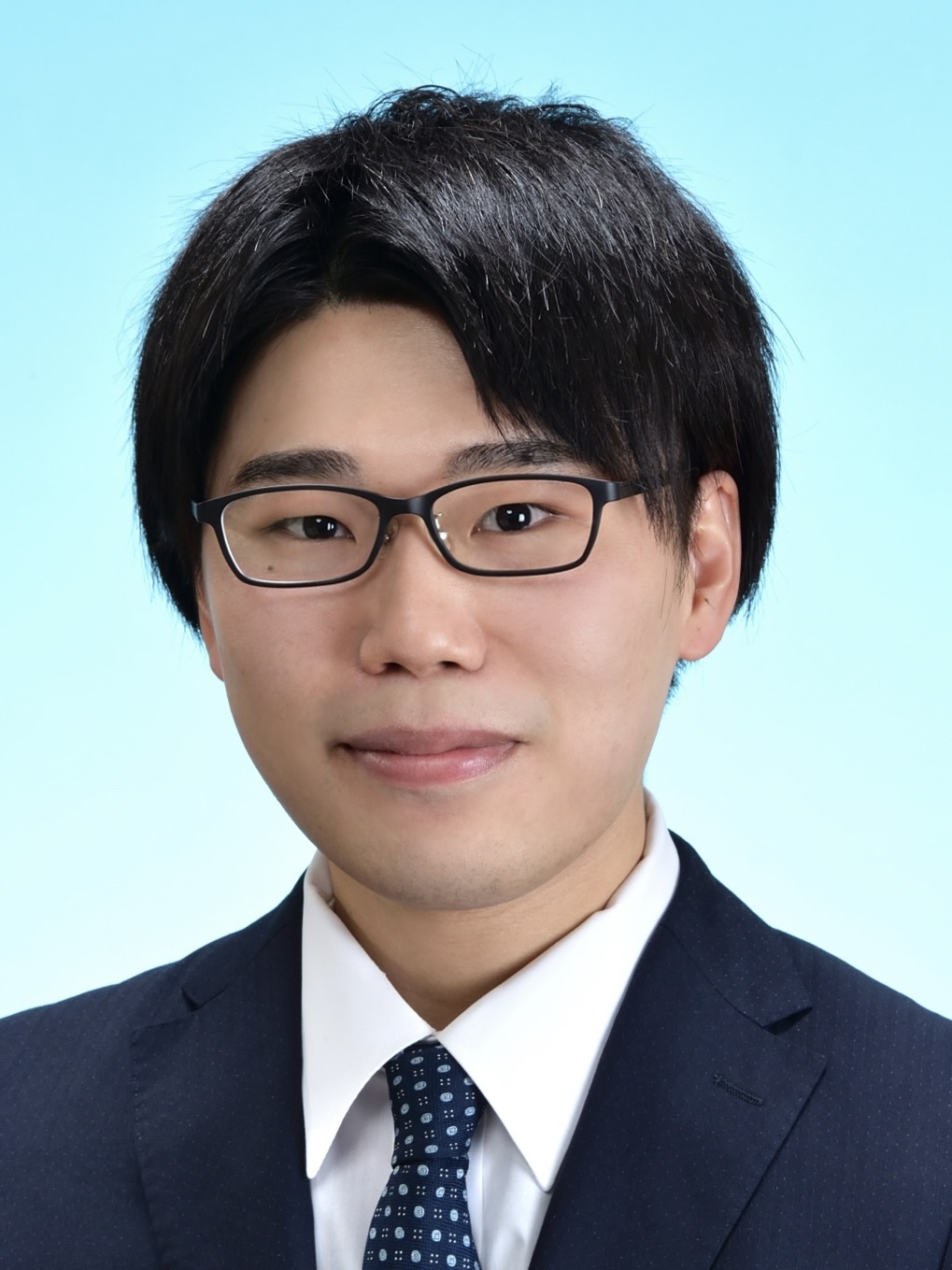}}]
  {Hiroki Shiraishi} received the B.E. and M.E. degrees in informatics from the University of Electro-Communications, Tokyo, Japan, in 2021 and 2023, respectively, and the Ph.D. degree in information systems from Yokohama National University, Yokohama, Japan, in 2026.

  Since 2026, he has been a Researcher with the Digital Healthcare Research Department, Hitachi, Ltd., Tokyo, Japan. His research interests include fuzzy systems, neuroevolution,
  learning classifier systems, and their clinical applications in pathology.

  Dr. Shiraishi received the Best Paper Award at GECCO 2022 and was nominated for the Best Paper Award at GECCO 2023. He served as Chair of the International Workshop on Evolutionary Rule-Based Machine Learning at GECCO from 2024 to 2026.
\end{IEEEbiography}
\begin{IEEEbiography}
  [{\includegraphics[width=1in,height=1.25in,clip,keepaspectratio]{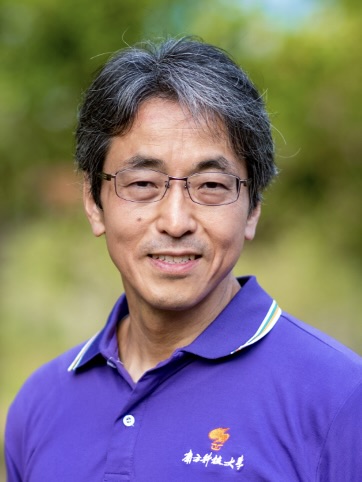}}]
  {Hisao Ishibuchi} (Fellow, IEEE) received the B.S. and M.S. degrees from Kyoto University in 1985 and 1987, respectively, and the Ph.D. degree from Osaka Prefecture University in 1992.

  He was with Osaka Prefecture University from 1987 to 2017.
Since 2017, he has been a Chair Professor at Southern University
of Science and Technology, China.
  His research interests include the design, analysis, comparison and application of evolutionary multi-objective optimization algorithms.

  Dr. Ishibuchi was the IEEE CIS Vice-President for Technical Activities in 2010-2013, the Editor-in-Chief of the IEEE \textsc{Computational Intelligence Magazine} in 2014-2019, and an AdCom member of the IEEE CIS in 2014-2019 and 2021-2026. Currently, he is General Chair of EMO 2027.

\end{IEEEbiography}
\begin{IEEEbiography}
  [{\includegraphics[width=1in,height=1.25in,clip,keepaspectratio]{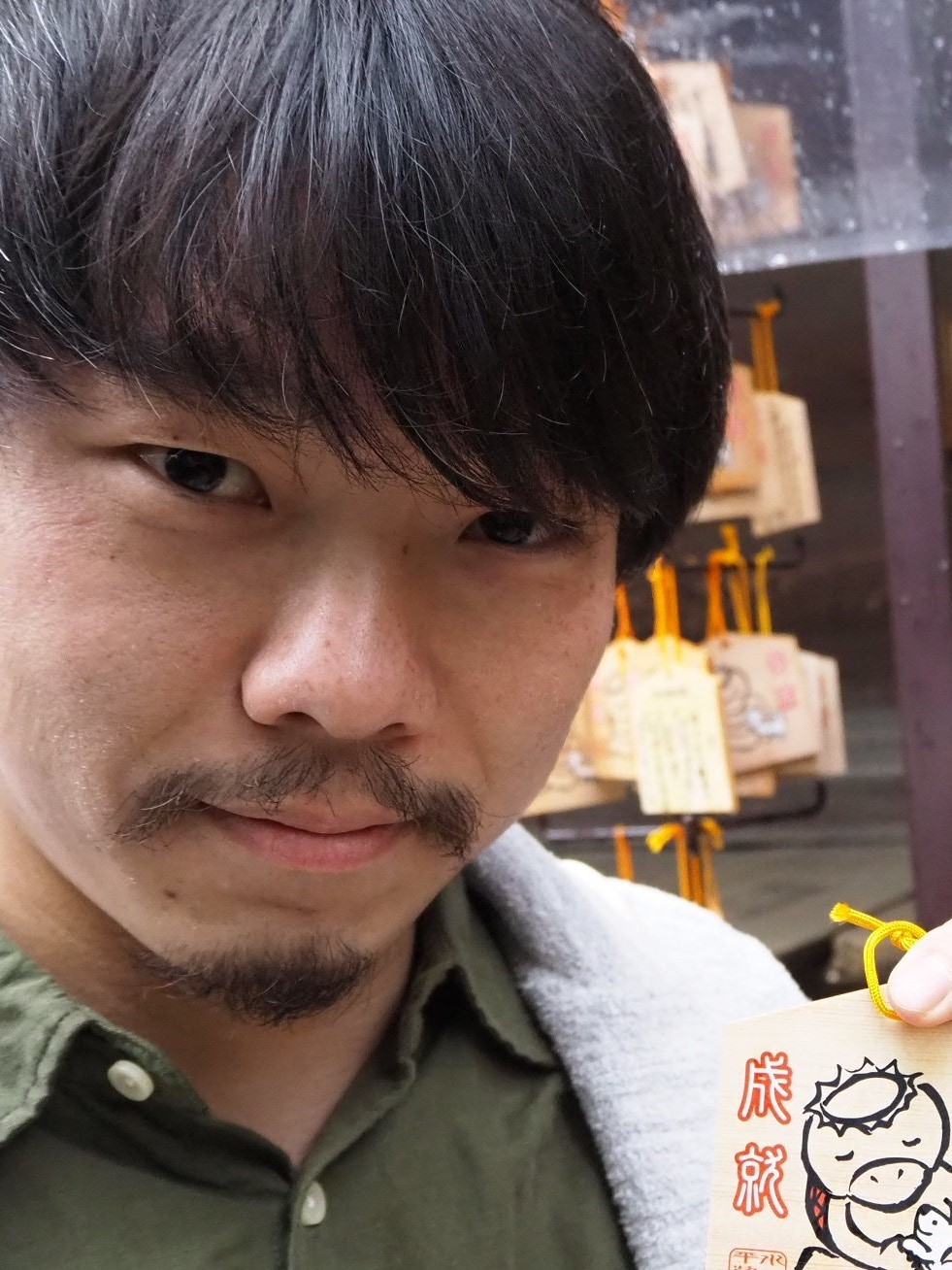}}]
  {Masaya Nakata} (Member, IEEE) received the Ph.D. degree in informatics from the University of Electro-Communications, Tokyo, Japan, in 2016.

  He is an Associate Professor with the Faculty of Engineering, Yokohama National University, Yokohama, Japan.
  Since 2019, he has focused his research on surrogate-assisted
evolutionary algorithms. His research has also included
evolutionary machine learning, data mining, and the theoretical
analysis of learning classifier systems.

Dr. Nakata chaired the International Workshop on Learning Classifier Systems in 2015, 2016, and 2018–2020 in GECCO.
\end{IEEEbiography}

\clearpage
\onecolumn

\setcounter{section}{0}
\setcounter{subsection}{0}
\setcounter{subsubsection}{0}
\setcounter{equation}{0}
\setcounter{figure}{0}
\setcounter{table}{0}
\renewcommand{\thesection}{S\arabic{section}}
\renewcommand{\thetable}{S\arabic{table}}
\renewcommand{\thefigure}{S\arabic{figure}}
\renewcommand{\theHsection}{supplement.\arabic{section}}
\renewcommand{\theHequation}{supplement.\arabic{equation}}
\renewcommand{\theHfigure}{supplement.\arabic{figure}}
\renewcommand{\theHtable}{supplement.\arabic{table}}

\makeatletter
\let\maketitle\arxiv@maketitlecmd
\let\@maketitle\arxiv@maketitleinternal
\let\thanks\arxiv@thankscmd
\let\arxiv@mainaddcontentsline\addcontentsline
\renewcommand{\cite}[1]{%
  \def\arxiv@supkeys{}%
  \@for\arxiv@key:=#1\do{%
    \ifx\arxiv@supkeys\@empty
      \edef\arxiv@supkeys{sup@\arxiv@key}%
    \else
      \edef\arxiv@supkeys{\arxiv@supkeys,sup@\arxiv@key}%
    \fi
  }%
  \expandafter\arxiv@maincite\expandafter{\arxiv@supkeys}%
}
\renewcommand{\addcontentsline}[3]{%
  \def\arxiv@target{#1}%
  \def\arxiv@toc{toc}%
  \ifx\arxiv@target\arxiv@toc
    \arxiv@mainaddcontentsline{stoc}{#2}{#3}%
  \else
    \arxiv@mainaddcontentsline{#1}{#2}{#3}%
  \fi
}
\makeatother

\title{Supplementary Material of ``Kolmogorov-Arnold Classifier Systems as Universal Approximators''}
\author{Hiroki Shiraishi$^{\orcidlink{0000-0001-8730-1276}}$,
  Hisao Ishibuchi$^{\orcidlink{0000-0001-9186-6472}}$,~\IEEEmembership{Fellow,~IEEE},
  and Masaya Nakata$^{\orcidlink{0000-0003-3428-7890}}$,~\IEEEmembership{Member,~IEEE}
  \thanks{
    Manuscript received 31 March 2026; revised 6 August 2026; accepted 11 September 2026.
  This work was supported by JSPS KAKENHI (Grant No. JP23KJ0993, JP25K03195), National Natural Science Foundation of China (Grant No. 62376115), and Guangdong Provincial Key Laboratory (Grant No. 2020B121201001). (\textit{Corresponding authors: Hisao Ishibuchi; Masaya Nakata.})}
  \thanks{Hiroki Shiraishi was with the Graduate School of Engineering Science, Yokohama National University, Yokohama 240-8501, Japan, at the time of manuscript submission on 31 March 2026. He is now with the Digital Healthcare Research Department, Hitachi, Ltd., Tokyo 185-8601, Japan (e-mail: hiroki.shiraishi.we@hitachi.com).}
  \thanks{Hisao Ishibuchi is with the Department of Computer Science and Engineering, Southern University of Science and Technology, Shenzhen 518055, China (e-mail: hisao@sustech.edu.cn).}
  \thanks{Masaya Nakata is with the Faculty of Engineering, Yokohama National University, Yokohama 240-8501, Japan (e-mail: nakata-masaya-tb@ynu.ac.jp).}
}

\markboth{IEEE TRANSACTIONS ON EVOLUTIONARY COMPUTATION, DOI: 10.1109/TEVC.2026.3736664, SEPTEMBER 2026}
{SHIRAISHI \MakeLowercase{{\em et al.}}:~Supplementary Material of ``Kolmogorov-Arnold Classifier Systems as Universal Approximators''}
\maketitle
\supplementtableofcontents
\clearpage

\section{{Inference Complexity Analysis}}
\label{sec:supp_inference_complexity}
\label{r2-20-2}

{This section compares the worst-case inference complexity of XCSF and KACS. The analysis concerns prediction only. It excludes all training operations, including covering, parameter updates, genetic search, subsumption, and deletion.

  Let $R_{\mathrm{X}}$ denote the number of macro-rules in the final XCSF population, i.e.,
  \begin{equation}
    R_{\mathrm{X}}
    =
    \left|\mathcal{P}\right|.
    \label{eq:rx}
  \end{equation}
  Let $R_{\mathrm{K}}$ denote the total number of macro-rules in all KACS inner and outer submodels:
  \begin{equation}
    R_{\mathrm{K}}
    =\left|\mathcal{P}\right|=
    \sum_{q=1}^{2n+1}\sum_{p=1}^{n}
    \left|\mathcal{P}_{q,p}^{\psi}\right|
    +
    \sum_{q=1}^{2n+1}
    \left|\mathcal{P}_{q}^{\Phi}\right|.
    \label{eq:rk}
  \end{equation}
  Here, a macro-rule is counted once regardless of its numerosity.

  \subsection{XCSF}

  For an input $\mathbf{x}\in\mathbb{R}^{n}$, XCSF forms a match set by testing whether $\mathbf{x}$ satisfies the interval condition of each rule. Since an XCSF rule has an $n$-dimensional condition, matching one rule requires $\mathcal{O}(n)$ comparisons in the worst case. Consequently, constructing the match set requires
  \begin{equation}
    \mathcal{O}\!\left(nR_{\mathrm{X}}\right)
    \label{eq:xcsf_matching}
  \end{equation}
  operations.

  Each matched rule has an $(n+1)$-parameter linear consequent. Its prediction therefore requires $\mathcal{O}(n)$ arithmetic operations. Because the number of matched rules is at most $R_{\mathrm{X}}$, computing and aggregating the rule predictions also requires
  \begin{equation}
    \mathcal{O}\!\left(nR_{\mathrm{X}}\right)
    \label{eq:xcsf_prediction}
  \end{equation}
  operations in the worst case. Therefore, the total worst-case inference complexity of XCSF is
  \begin{equation}
    T_{\mathrm{XCSF}}
    =
    \mathcal{O}\!\left(nR_{\mathrm{X}}\right).
    \label{eq:xcsf_complexity}
  \end{equation}

  \subsection{KACS}

  KACS evaluates $n(2n+1)$ inner submodels and $2n+1$ outer submodels. Thus, it contains
  \begin{equation}
    n(2n+1)+(2n+1)
    =
    (n+1)(2n+1)
    =
    2n^2+3n+1
    \label{eq:number_submodels}
  \end{equation}
  submodels.

  Each KACS rule has a one-dimensional interval condition and a two-parameter linear consequent. Matching and evaluating an individual rule therefore require $\mathcal{O}(1)$ operations. Scanning and evaluating all inner and outer rule populations require
  \begin{equation}
    \mathcal{O}\!\left(R_{\mathrm{K}}\right).
    \label{eq:kacs_rules}
  \end{equation}

  KACS additionally computes the intermediate sums
  \begin{equation}
    \hat{z}_q
    =
    \sum_{p=1}^{n}
    \hat{\psi}_{q,p}(x_p),
    \qquad q=1,\ldots,2n+1,
  \end{equation}
  which requires $n(2n+1)$ additions. It subsequently aggregates the $2n+1$ outer outputs to obtain $\hat{y}$. The total channel-wise aggregation cost is therefore
  \begin{equation}
    \mathcal{O}\!\left(n(2n+1)\right)
    =
    \mathcal{O}\!\left(n^2\right).
    \label{eq:kacs_aggregation}
  \end{equation}

  Combining rule evaluation and channel-wise aggregation yields
  \begin{equation}
    T_{\mathrm{KACS}}
    =
    \mathcal{O}\!\left(
      R_{\mathrm{K}}+n(2n+1)
    \right)
    =
    \mathcal{O}\!\left(R_{\mathrm{K}}+n^2\right).
    \label{eq:kacs_complexity}
  \end{equation}

  \subsection{Comparison of Inference Complexity}

  For the final learned populations, the worst-case inference complexities are
  \begin{equation}
    T_{\mathrm{XCSF}}
    =
    \mathcal{O}\!\left(nR_{\mathrm{X}}\right),
    \qquad
    T_{\mathrm{KACS}}
    =
    \mathcal{O}\!\left(R_{\mathrm{K}}+n^2\right).
  \end{equation}

  Thus, KACS removes the multiplicative factor $n$ from the cost of matching and evaluating an individual rule. However, it evaluates inner and outer submodels across multiple channels and consequently incurs a channel-wise aggregation cost of $\mathcal{O}(n^2)$. Hence, this analysis does not imply that KACS is always faster in wall-clock inference time. The practical runtime also depends on $R_{\mathrm{X}}$, $R_{\mathrm{K}}$, the allocation of KACS rules among submodels, and implementation-specific factors.

  \subsection{Asymptotic Rule-Count Interpretation}

  The preceding analysis is expressed in terms of the final numbers of macro-rules, $R_{\mathrm{X}}$ and $R_{\mathrm{K}}$. To relate these quantities to dimensionality, consider an idealized setting in which $m$ effective intervals are required per input dimension for XCSF and, on average, $m$ rules are required per one-dimensional KACS submodel.

  Under this assumption, XCSF requires
  \begin{equation}
    R_{\mathrm{X}}=\mathcal{O}(m^n)
  \end{equation}
  rules to cover an $n$-dimensional input space. KACS has $n(2n+1)$ inner submodels and $2n+1$ outer submodels. Hence, its total number of rules is
  \begin{equation}
    R_{\mathrm{K}}
    =
    \mathcal{O}\!\left(
      m[n(2n+1)+(2n+1)]
    \right)
    =
    \mathcal{O}\!\left(
      m(n+1)(2n+1)
    \right)
    =
    \mathcal{O}(mn^2).
  \end{equation}

  Substituting these estimates into Eqs.~\eqref{eq:xcsf_complexity} and~\eqref{eq:kacs_complexity} gives
  \begin{equation}
    T_{\mathrm{XCSF}}=\mathcal{O}(nm^n),
    \qquad
    T_{\mathrm{KACS}}=\mathcal{O}(mn^2).
  \end{equation}
  These expressions describe asymptotic scaling under the stated idealized rule-allocation assumption. They do not represent the numbers of rules learned in a finite experiment, which depend on the population budget $N$, the data distribution, and the learning dynamics.
}

\section{Generalization of the KACS Universal Approximation Theorem to Other Consequent Models}
\label{sec:uat_discussion_sup}

Although the proof in Section \ref{sec:uat} of the main article is presented with a linear consequent $P_k(z) = w_{k,0} + w_{k,1}z$ as in \eqref{eq:kacs_rule}, the same argument extends to a broader class of consequent models. In particular, KACS retains the universal approximation property even when the consequent is chosen as (i) polynomial or B-spline models that can represent linear functions exactly~\cite{de1978practical}, or (ii) universal approximators such as MLP, RBFN~\cite{park1991universal}, and KAN.

This follows because the construction in Lemma~\ref{lem:cpl} (see also \eqref{eq:ref_err}) of the main article only requires the consequent of each reference rule $cl^*_j$, denoted by $P^*_j(z)$, to approximate the target linear function $h_j(z)$ on the corresponding region $R_j$ with error at most $\varepsilon/4$. A linear consequent can represent $h_j$ exactly, and polynomial or B-spline consequents can do so as well because they include linear functions as a special case. If a universal approximator (MLP, RBFN, KAN) is used as the consequent, it can approximate $h_j$ within $\varepsilon/4$, which is sufficient for the same error bound to carry through.

Therefore, the universal approximation capability of KACS is not limited to linear consequents; it generalizes to more flexible consequent models that can represent linear functions either exactly or to arbitrary accuracy.
\clearpage

\section{Comparison with SupRB: A State-of-the-Art Pittsburgh-Style LCS}
\label{sec:suprb_comparison_sup}
\label{r2-8-2}
\label{r2-10-2}
\label{r2-15-2}
\label{r2-19-2}

\subsection{Experimental Protocol and Comparability}

{We conducted an additional comparison with SupRB, a Pittsburgh-style LCS that separately optimizes rule discovery and final rule-set composition~\cite{heider2023suprb,heider2025disentangling}. The comparison includes the four real-world benchmarks used in the main article and two additional benchmarks considered by Heider~\cite{heider2025disentangling}: the Physicochemical Properties of Protein Tertiary Structure (PPPTS) dataset with $n=9$ and the Parkinsons Telemonitoring (PT) dataset with $n=18$. The latter two extend the dimensionality of the real-world evaluation beyond the original maximum of $n=8$. They should be interpreted as higher-dimensional relative to the original benchmark set, rather than as extremely high-dimensional data.

  For our XCSF and KACS experiments, we followed Heider's evaluation protocol~\cite{heider2025disentangling}: eight Monte Carlo data splits were combined with eight random seeds per split, yielding 64 runs for each method and dataset; 25\% of the samples were held out for testing; and the target was standardized. XCSF and KACS used the same splits and seeds. The input features were scaled to $[0,1]$ instead of the $[-1,1]$ range used by Heider because the present KACS formulation assumes inputs in $[0,1]$. Our XCSF and KACS populations were evaluated without post-training compaction. For PT ($n=18$) only, we set the population budget to $N=15984$ for both XCSF and KACS, following the practical guideline in Section~\ref{sss:scalability_discussion} of the main article and retaining an equal budget for the controlled comparison. For all other datasets, we used $N=6400$, as in Section~\ref{sss: Hyperparameter Settings and Protocol} of the main article.

  The literature values require several qualifications. Heider's reported XCSF results for ASN, CCPP, CS, and EEC use dataset-specific hyperparameter tuning, recursive-least-squares consequent updates, and post-training population compaction; XCSF results were not reported for PPPTS or PT. For SupRB, we use the evolution-strategy rule-discovery configuration, denoted SupRB-ES, because this is the reported configuration for which results are available on all six datasets. Its rule count is the number of rules in the final selected solution. The reported XCSF and SupRB-ES values were obtained in separate experimental campaigns and are therefore included as literature reference points, not as paired observations from our runs. Consequently, the tables are used for descriptive comparison only; we do not report cross-study statistical tests or average ranks.}

\subsection{Prediction Accuracy}

{Table~\ref{tb:suprb accuracy sup} compares testing MSE in the standardized target space. Values for all methods are reported as the mean and standard deviation over 64 runs.}

{On the four benchmarks used in the main article, KACS obtains lower MSE than our uncompacted XCSF in every case: 0.14 vs. 0.32 on ASN, 0.06 vs. 0.07 on CCPP, 0.14 vs. 0.21 on CS, and 0.06 vs. 0.08 on EEC. Its MSE is also close to the literature reference values on ASN, CCPP, and CS, whereas the reported XCSF and SupRB-ES results are lower on EEC. These observations indicate that the favorable KACS--XCSF comparison is retained under the 64-run, standardized-target protocol, but they do not establish superiority over independently reported methods.

  The two additional datasets give a more mixed result. On PPPTS, XCSF and KACS both obtain an MSE of 0.74, compared with 0.63 for SupRB-ES. On PT, our XCSF obtains the lowest MSE (0.09), followed by KACS (0.21) and SupRB-ES (0.31). Thus, the higher-dimensional real-world experiments do not support a claim that KACS uniformly provides the best predictive accuracy. Instead, they show that its relative performance depends on the dataset: KACS remains competitive on PT but does not improve upon XCSF on PPPTS.

  Descriptively, KACS records lower mean testing MSE than the reported SupRB-ES values on ASN, CCPP, CS, and PT, whereas SupRB-ES records lower values on EEC and PPPTS. These results suggest that KACS can achieve predictive performance competitive with SupRB-ES. However, because the SupRB-ES values were obtained in independent experiments, this comparison should not be interpreted as evidence of statistical superiority.

  \begin{table*}[h]
    \centering
    \caption{{Testing MSE on the Six Real-World Regression
        Benchmarks. Values Are Mean $\pm$ Standard Deviation over 64 Runs.
        Our XCSF and KACS Results Were Obtained without Compaction. The
        Literature Values Are Reported by Heider; SupRB-ES Uses
        Evolution-Strategy Rule Discovery. N/A Indicates That No Corresponding
        XCSF Result Was Reported. The Cross-Study Results Are Presented
    Descriptively, without Statistical Tests or Average Ranks}}
    \label{tb:suprb accuracy sup}
    \renewcommand{\arraystretch}{1.2}
    \resizebox{\width}{!}{
      \small
      \begin{tabular}{ccc|cc|cc}
  \bhline{1pt}
  &&&\multicolumn{2}{c|}{Ours (no compaction)}&\multicolumn{2}{c}{Reported by \cite{heider2025disentangling}}\\
  Benchmark& \#Samples & $n$ & XCSF & KACS & XCSF (after compaction) & SupRB-ES (final solution)\\
  \bhline{1pt}
  ASN & 1503 & 5 & 0.32 $\pm$ 0.08 & 0.14 $\pm$ 0.02 & 0.12 $\pm$ 0.16 & 0.15 $\pm$ 0.02 \\
  CCPP & 9568 & 4 & 0.07 $\pm$ 0.00 & 0.06 $\pm$ 0.01 & 0.06 $\pm$ 0.00 & 0.07 $\pm$ 0.00 \\
  CS & 1030 & 8 & 0.21 $\pm$ 0.05 & 0.14 $\pm$ 0.03 & 0.17 $\pm$ 0.13 & 0.15 $\pm$ 0.04 \\
  EEC & 768 & 8 & 0.08 $\pm$ 0.02 & 0.06 $\pm$ 0.02 & 0.02 $\pm$ 0.02 & 0.04 $\pm$ 0.03  \\
  PPPTS & 45739 & 9 & 0.74 $\pm$ 0.01 & 0.74 $\pm$ 0.04 & N/A & 0.63 $\pm$ 0.02 \\
  PT & 5875 & 18 & 0.09 $\pm$ 0.02 & 0.21 $\pm$ 0.07 & N/A & 0.31 $\pm$ 0.03 \\
  \bhline{1pt}
\end{tabular}

    }
  \end{table*}

  \subsection{Rule Complexity}

  Table~\ref{tb:suprb rules sup} reports the corresponding numbers of rules.}

{KACS uses fewer rules than our uncompacted XCSF on five of the six datasets; CCPP is the sole exception, where the counts are similar (1927 vs. 1898). The reduction is particularly pronounced on CS, EEC, and PPPTS. Nevertheless, the results also qualify the compactness claim made from the controlled KACS--XCSF comparison. The compacted XCSF models reported by Heider use fewer rules than KACS on CS and EEC, and SupRB-ES uses substantially fewer rules on every dataset, with only 5.5--45.7 rules on average. KACS should therefore be described as compact relative to the uncompacted XCSF configuration evaluated in our controlled experiments, not as the most compact LCS among the methods considered here.

  Raw rule counts also represent different model structures. SupRB-ES selects a small final set of multidimensional rules, whereas KACS distributes one-dimensional rules among $(n+1)(2n+1)=2n^2+3n+1$ inner and outer submodels. This fixed decomposition imposes a structural cost because every submodel must maintain input coverage. Moreover, Heider did not report a directly comparable total parameter count for SupRB-ES. We therefore cannot infer a precise parameter-count ordering from Table~\ref{tb:suprb rules sup}, although its much smaller final rule sets clearly demonstrate greater rule-set compactness.}

\subsection{Interpretability of the Learned Models}

{A KACS population containing approximately 1000 or more macro-rules should not be regarded as a directly human-readable rule list. Its structure is nevertheless different from that of a flat XCSF population. Across the six datasets, KACS distributes its rules among 45 submodels for CCPP, 66 for ASN, 153 for CS and EEC, 190 for PPPTS, and 703 for PT. The observed totals therefore correspond to approximately 5.1--42.8 one-dimensional rules per submodel. Each submodel can be plotted and inspected as a one-dimensional local function, which provides structured inspectability that is not apparent from the total rule count alone.

  This decomposition does not, however, guarantee global interpretability. A complete KACS prediction depends on the composition of all inner and outer submodels across $2n+1$ channels. Understanding an individual one-dimensional submodel is considerably easier than reading thousands of multidimensional XCSF rules, but understanding the complete mapping still requires tracing many interacting local functions. In this sense, KACS is better characterized as a structured and visualizable model than as a globally transparent model.

  SupRB-ES addresses interpretability more directly by selecting a small final set of multidimensional rules. Its 5.5--45.7-rule solutions are far closer to a model that can be inspected rule by rule than either the XCSF or KACS populations reported here. Accordingly, we do not claim that KACS is more globally interpretable than SupRB. The present results support parameter efficiency and structured inspectability relative to uncompacted XCSF, while SupRB-ES provides substantially greater final rule-set compactness. A formal comparison of interpretability---including visualization of learned KACS submodels, fidelity of local explanations, human-subject evaluation, and the effects of pruning or compaction---is left for future work.}

\begin{table*}[t]
  \centering
  \caption{{Numbers of Rules on the Six Real-World Regression
  Benchmarks. Notation Follows Table \ref{tb:suprb accuracy sup}}}
  \label{tb:suprb rules sup}
  \renewcommand{\arraystretch}{1.2}
  \resizebox{\width}{!}{
    \small
    \begin{tabular}{ccc|cc|cc}
  \bhline{1pt}
  &&&\multicolumn{2}{c|}{Ours (no compaction)}&\multicolumn{2}{c}{Reported by \cite{heider2025disentangling}}\\
  Benchmark& \#Samples & $n$ & XCSF & KACS & XCSF (after compaction) & SupRB-ES (final solution)\\
  \bhline{1pt}
  ASN & 1503 & 5 & 2000 $\pm$ 410.8 & 1434 $\pm$ 40.38 & 1617.58 $\pm$ 413.12 & 43.3 $\pm$ 2.8 \\
  CCPP & 9568 & 4 & 1898 $\pm$ 297.1 & 1927 $\pm$ 45.26 & 1922.71 $\pm$ 390.7 & 5.5 $\pm$ 0.8 \\
  CS & 1030 & 8 & 2837 $\pm$ 434.5 & 1416 $\pm$ 42.30 & 481.68 $\pm$ 336.33 & 42.1 $\pm$ 2.9 \\
  EEC & 768 & 8 & 2985 $\pm$ 391.4 & 1014 $\pm$ 35.34 & 707.94 $\pm$ 282.16 & 20.1 $\pm$ 2.2 \\
  PPPTS & 45739 & 9 & 3001 $\pm$ 578.4 & 1405 $\pm$ 35.96 & N/A & 45.7 $\pm$ 3.6 \\
  PT & 5875 & 18 & 4422 $\pm$ 999.9 & 3564 $\pm$ 89.64 & N/A & 38.8 $\pm$ 3.3 \\
  \bhline{1pt}
\end{tabular}

  }
\end{table*}
\clearpage

\section{Formulations of Synthetic Functions}
\label{sec:Formulations of Synthetic Functions sup}
This section summarizes the synthetic benchmark functions used in the main article. Fig.~\ref{fig:all_functions sup} illustrates each function for $n=2$. The four synthetic benchmark functions are:
\begin{itemize}
  \item \textbf{Rastrigin function} ($f_1$, cf. Fig. \ref{fig:rastrigin sup}): a highly multimodal landscape with many local minima, defined as:
    \begin{equation}
      f_1(\mathbf{x}) = 10n + \sum_{i=1}^{n} \left({x}_i^2 - 10\cos\left(2\pi{x}_i\right)\right).
    \end{equation}
  \item \textbf{Rosenbrock function} ($f_2$, cf. Fig. \ref{fig:rosenbrock sup}): a narrow curved valley with strong interactions between adjacent variables, defined as:
    \begin{equation}
      f_2(\mathbf{x}) = \sum_{i=1}^{n-1} \left[100\left({x}_{i+1}-{x}_i^2\right)^2 + \left({x}_i-1\right)^2\right].
    \end{equation}
  \item \textbf{Cross function} ($f_3$, cf. Fig. \ref{fig:cross_function sup}): axis-aligned ridge-like components along each coordinate axis and a central Gaussian peak, with interactions across all dimensions, defined as:
    \begin{align}
      a &= \frac{1}{\lfloor n/2 \rfloor}\sum_{i=1}^{\lfloor n/2 \rfloor}{x}_i,\,\,\,
      b = \frac{1}{\lceil n/2 \rceil}\sum_{i=\lfloor n/2 \rfloor+1}^{n}{x}_i, \\
      f_3(\mathbf{x}) &= \max\!\left(
        \exp(-10a^2),\,
        \exp(-50b^2),\,
        1.25\,\exp\!\left(-5(a^2+b^2)\right)
      \right).
    \end{align}
  \item \textbf{Styblinski-Tang function} ($f_4$, cf. Fig. \ref{fig:styblinski_tang sup}): a separable nonconvex polynomial surface with multiple local minima per dimension, defined as:
    \begin{equation}
      f_4(\mathbf{x})=
      \frac{1}{2}\sum_{i=1}^{n}
      \left(
        {x}_i^4 - 16{x}_i^2 + 5{x}_i
      \right).
    \end{equation}
\end{itemize}
These functions collectively cover multimodality ($f_1$), inter-variable curvature ($f_2$), anisotropic structure ($f_3$), and separable nonconvexity ($f_4$).

\preparesyntheticsurfaceplots
\begin{figure*}[h]
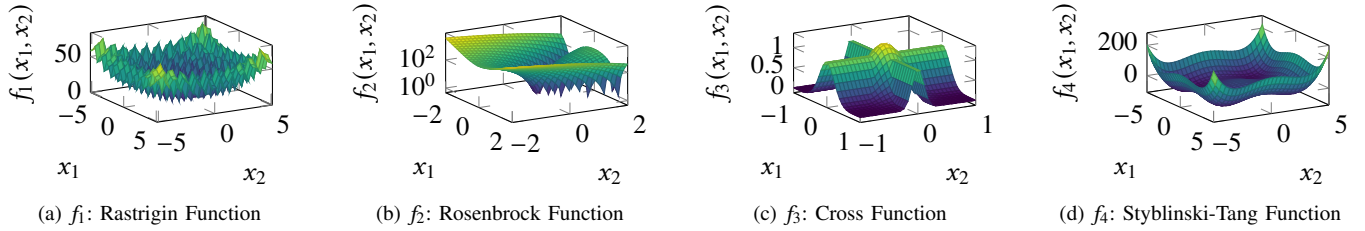

  \centering
  \subfloat[\shortstack{$f_1$: Rastrigin Function}]{
    \makebox[0.22\textwidth][c]{\syntheticrastriginplot}
    \label{fig:rastrigin sup}
  }\hfill
  \subfloat[\shortstack{$f_2$: Rosenbrock Function}]{
    \makebox[0.22\textwidth][c]{\syntheticrosenbrockplot}
    \label{fig:rosenbrock sup}
  }\hfill
  \subfloat[\shortstack{$f_3$: Cross Function}]{
    \makebox[0.22\textwidth][c]{\syntheticcrossplot}
    \label{fig:cross_function sup}
  }\hfill
  \subfloat[\shortstack{$f_4$: Styblinski-Tang Function}]{
    \makebox[0.22\textwidth][c]{\syntheticstyblinskitangplot}
    \label{fig:styblinski_tang sup}
  }
  \caption{Surface plots of the four synthetic test functions ($n=2$) shown for visualization (same as Fig. \ref{fig:all_functions} of the main article).}
  \label{fig:all_functions sup}
\end{figure*}

\clearpage

\section{A Description of Hyperparameters}
\label{sec:Descriptions of Hyperparameters sup}
The following hyperparameters are commonly used in both XCSF and KACS.

\begin{itemize}
  \item $N$: The maximum population size $\mathcal{P}$ (in micro-rules).
  \item $\epsilon_0$: The error threshold; rules with prediction error below $\epsilon_0$ are considered accurate. {A suitable tolerance can improve robustness to noise-induced error fluctuations and avoid unnecessary specialization, whereas an excessively large value can reduce approximation accuracy \cite{urbanowicz2017introduction}.}\label{r3-9}
  \item $\beta$: The learning rate for updating prediction error, fitness, and match set size estimates.
  \item $\alpha$: The fall-off coefficient in the accuracy calculation; controls the accuracy assigned when the prediction error is at or above $\epsilon_0$.
  \item $\nu$: The power parameter for the accuracy calculation; controls the steepness of the accuracy fall-off.
  \item $\delta$: The fraction of the mean fitness used as the threshold for rule deletion.
  \item $m_0$: The maximum mutation range for interval parameters.
  \item $r_0$: The maximum spread for the covering operator.
  \item $\chi$: The crossover probability in the GA.
  \item $\mu$: The mutation probability per allele in the GA.
  \item $\tau$: The tournament size ratio for parent selection.
  \item $\theta_{\rm GA}$: The GA application threshold.
  \item $\theta_{\rm del}$: The experience threshold for rule deletion.
  \item $\theta_{\rm sub}$: The experience threshold for subsumption.
  \item $P_\#$: The probability of generating a $[0,1]$ interval during covering.
  \item $\mathit{doSubsumption}$: A Boolean flag to enable GA subsumption.
\end{itemize}

\section{Detailed Statistical Results}
\label{sec:supplementary-statistics}
\label{r2-11-3}
\label{r1-3-2}
{This section presents run-level statistics that complement Table~\ref{tb:result} in the main article. For each benchmark and metric, Table~\ref{tb:supplementary-statistics sup} reports the mean and standard deviation over the same 30 train--test splits, together with the two-sided Wilcoxon signed-rank $p$-value and matched-pairs rank-biserial correlation. The effect size is computed from the paired differences XCSF$-$KACS; positive values favor KACS and negative values favor XCSF because lower values are preferable for every metric.

  KACS shows its most consistent advantage in model complexity. Its parameter count and AIC are lower on all eight benchmarks, with an effect size of $1$ in every comparison. Its rule count is also lower with an effect size of $1$ on seven benchmarks; on CCPP, the rule counts are similar and the effect is small ($0.0796$). This pattern reflects the KA decomposition, which replaces full $n$-dimensional local models with one-dimensional submodels.

  Accuracy depends more strongly on problem characteristics. For testing MAE, KACS shows positive effects on the four nonlinear synthetic functions ($0.7462$--$1$) and ASN ($0.9484$), supporting its advantage on problems with pronounced nonlinearity or cross-dimensional structure. On $f_1$, the effect favors XCSF for training MAE ($-0.5226$) but KACS for testing MAE ($0.7677$), which, together with the large parameter-count difference, is consistent with greater overfitting by XCSF. The testing-MAE effects are smaller on CCPP ($0.2430$) and CS ($0.2172$), while EEC favors XCSF ($-0.3462$). For problems that are adequately represented by simpler linear models, the coordination of KACS's inner and outer submodels can introduce additional approximation overhead, although KACS remains substantially more compact.

  The standard deviations show that stability is metric dependent. KACS has lower standard deviations in rule count, parameter count, and AIC on all benchmarks. For testing MAE, however, KACS is less variable on ASN ($0.0149$ vs. $0.0222$) and CS ($0.0169$ vs. $0.0231$), but more variable on the synthetic functions, CCPP, and EEC. In the small-MAE cases, the standard deviations are $0.0029$ for XCSF and $0.0073$ for KACS on CCPP, and $0.0223$ and $0.0314$, respectively, on EEC. Thus, KACS provides consistently stable model complexity but does not uniformly reduce variability in prediction accuracy. All MAE statistics use the normalized target scale of the main experiments.}

\begin{table*}[h]
  \centering
  \caption{{Detailed Paired Comparison of XCSF and KACS over 30 Matched
      Runs for Each Benchmark and Metric. Values Are Reported as the Mean
      $\pm$ Standard Deviation. Green Cells Indicate the Lower Mean for
      Each Benchmark and Metric. The Wilcoxon Column Reports the Two-Sided
      Wilcoxon Signed-Rank $p$-Value Computed from the 30 Paired Runs. The
      Effect Size Is the Matched-Pairs Rank-Biserial Correlation Computed
      from the Paired Differences (XCSF Minus KACS); Positive Values
      Indicate That KACS Tended to Produce Lower Values, Whereas Negative
  Values Indicate That XCSF Tended to Produce Lower Values}}
  \label{tb:supplementary-statistics sup}
  \renewcommand{\arraystretch}{1.2}
  \resizebox{\width}{!}{
    \small
    \begin{tabular}{cccccc}
  \bhline{1pt}
  Benchmark & Metric & XCSF mean $\pm$ SD & KACS mean $\pm$ SD & Wilcoxon $p$ & Effect size \\
  \bhline{1pt}
  $f_1$ & Training MAE & \cellcolor{g}0.2033 $\pm$ 0.0122 & 0.2194 $\pm$ 0.0291 & 0.0113 & -0.5226 \\
  $f_1$ & Testing MAE & 0.2642 $\pm$ 0.0220 & \cellcolor{g}0.2351 $\pm$ 0.0250 & 8.86E-05 & 0.7677 \\
  $f_1$ & \#Rules & 4162 $\pm$ 57.59 & \cellcolor{g}1341 $\pm$ 35.07 & 1.82E-06 & 1 \\
  $f_1$ & \#Parameters & 45781 $\pm$ 633.5 & \cellcolor{g}2681 $\pm$ 70.13 & 1.86E-09 & 1 \\
  $f_1$ & AIC & 89207 $\pm$ 1266 & \cellcolor{g}3024 $\pm$ 262.6 & 1.86E-09 & 1 \\
  \bhline{1pt}
  $f_2$ & Training MAE & \cellcolor{g}0.1124 $\pm$ 0.0119 & 0.1157 $\pm$ 0.0280 & 0.5838 & 0.1183 \\
  $f_2$ & Testing MAE & 0.1564 $\pm$ 0.0195 & \cellcolor{g}0.1253 $\pm$ 0.0274 & 0.0001529 & 0.7462 \\
  $f_2$ & \#Rules & 4087 $\pm$ 56.88 & \cellcolor{g}1462 $\pm$ 37.63 & 1.82E-06 & 1 \\
  $f_2$ & \#Parameters & 44960 $\pm$ 625.6 & \cellcolor{g}2924 $\pm$ 75.27 & 1.82E-06 & 1 \\
  $f_2$ & AIC & 86579 $\pm$ 1193 & \cellcolor{g}2342 $\pm$ 314.1 & 1.86E-09 & 1 \\
  \bhline{1pt}
  $f_3$ & Training MAE & 0.2510 $\pm$ 0.0167 & \cellcolor{g}0.2097 $\pm$ 0.0395 & 4.42E-06 & 0.8667 \\
  $f_3$ & Testing MAE & 0.3246 $\pm$ 0.0250 & \cellcolor{g}0.2329 $\pm$ 0.0508 & 3.73E-09 & 0.9957 \\
  $f_3$ & \#Rules & 4130 $\pm$ 74.08 & \cellcolor{g}1429 $\pm$ 36.94 & 1.82E-06 & 1 \\
  $f_3$ & \#Parameters & 45426 $\pm$ 814.9 & \cellcolor{g}2857 $\pm$ 73.87 & 1.86E-09 & 1 \\
  $f_3$ & AIC & 88892 $\pm$ 1580 & \cellcolor{g}3324 $\pm$ 349.7 & 1.86E-09 & 1 \\
  \bhline{1pt}
  $f_4$ & Training MAE & 0.1680 $\pm$ 0.0150 & \cellcolor{g}0.1334 $\pm$ 0.0301 & 2.08E-05 & 0.8194 \\
  $f_4$ & Testing MAE & 0.2303 $\pm$ 0.0183 & \cellcolor{g}0.1430 $\pm$ 0.0299 & 1.86E-09 & 1 \\
  $f_4$ & \#Rules & 4170 $\pm$ 61.11 & \cellcolor{g}1414 $\pm$ 35.27 & 1.82E-06 & 1 \\
  $f_4$ & \#Parameters & 45873 $\pm$ 672.2 & \cellcolor{g}2829 $\pm$ 70.54 & 1.86E-09 & 1 \\
  $f_4$ & AIC & 89078 $\pm$ 1300 & \cellcolor{g}2412 $\pm$ 425.1 & 1.86E-09 & 1 \\
  \bhline{1pt}
  \text{ASN} & Training MAE & 0.1434 $\pm$ 0.0203 & \cellcolor{g}0.1094 $\pm$ 0.0139 & 2.05E-07 & 0.9398 \\
  \text{ASN} & Testing MAE & 0.1478 $\pm$ 0.0222 & \cellcolor{g}0.1144 $\pm$ 0.0149 & 1.30E-07 & 0.9484 \\
  \text{ASN} & \#Rules & 2224 $\pm$ 356.7 & \cellcolor{g}1417 $\pm$ 46.26 & 1.86E-09 & 1 \\
  \text{ASN} & \#Parameters & 13345 $\pm$ 2140 & \cellcolor{g}2834 $\pm$ 92.51 & 1.86E-09 & 1 \\
  \text{ASN} & AIC & 22275 $\pm$ 3953 & \cellcolor{g}371.1 $\pm$ 344.2 & 1.86E-09 & 1 \\
  \bhline{1pt}
  \text{CCPP} & Training MAE & 0.0930 $\pm$ 0.0016 & \cellcolor{g}0.0918 $\pm$ 0.0062 & 0.1294 & 0.3204 \\
  \text{CCPP} & Testing MAE & 0.0929 $\pm$ 0.0029 & \cellcolor{g}0.0926 $\pm$ 0.0073 & 0.2534 & 0.243 \\
  \text{CCPP} & \#Rules & 1953 $\pm$ 292.8 & \cellcolor{g}1924 $\pm$ 48.14 & 0.7112 & 0.0796 \\
  \text{CCPP} & \#Parameters & 9764 $\pm$ 1464 & \cellcolor{g}3848 $\pm$ 96.29 & 1.86E-09 & 1 \\
  \text{CCPP} & AIC & -17358 $\pm$ 2823 & \cellcolor{g}-29065 $\pm$ 1194 & 1.86E-09 & 1 \\
  \bhline{1pt}
  \text{CS} & Training MAE & 0.1299 $\pm$ 0.0198 & \cellcolor{g}0.1213 $\pm$ 0.0137 & 0.1191 & 0.329 \\
  \text{CS} & Testing MAE & 0.1350 $\pm$ 0.0231 & \cellcolor{g}0.1288 $\pm$ 0.0169 & 0.3085 & 0.2172 \\
  \text{CS} & \#Rules & 2985 $\pm$ 393.1 & \cellcolor{g}1425 $\pm$ 34.58 & 1.86E-09 & 1 \\
  \text{CS} & \#Parameters & 26868 $\pm$ 3538 & \cellcolor{g}2851 $\pm$ 69.16 & 1.86E-09 & 1 \\
  \text{CS} & AIC & 50453 $\pm$ 6860 & \cellcolor{g}2237 $\pm$ 229.1 & 1.86E-09 & 1 \\
  \bhline{1pt}
  \text{EEC} & Training MAE & \cellcolor{g}0.0820 $\pm$ 0.0202 & 0.0967 $\pm$ 0.0320 & 0.1048 & -0.3419 \\
  \text{EEC} & Testing MAE & \cellcolor{g}0.0847 $\pm$ 0.0223 & 0.1020 $\pm$ 0.0314 & 0.1004 & -0.3462 \\
  \text{EEC} & \#Rules & 2925 $\pm$ 373.6 & \cellcolor{g}1012 $\pm$ 31.51 & 1.82E-06 & 1 \\
  \text{EEC} & Parameters & 26326 $\pm$ 3363 & \cellcolor{g}2023 $\pm$ 63.02 & 1.86E-09 & 1 \\
  \text{EEC} & AIC & 49849 $\pm$ 6567 & \cellcolor{g}1125 $\pm$ 302.7 & 1.86E-09 & 1 \\
  \bhline{1pt}
\end{tabular}

  }
\end{table*}
\clearpage

\section{Robustness under an Alternative Evaluation Protocol}
\label{sec:alternative_protocol_sup}
\label{r2-9-3}

\subsection{Experimental Protocol}

{To examine whether the controlled XCSF--KACS comparison depends on the evaluation protocol used in the main article, we repeated the experiments on all eight benchmark problems using a protocol based on Heider's evaluation design~\cite{heider2023suprb,heider2025disentangling}. The main experiments use 30 Monte Carlo splits with 90\% of the samples for training, min--max-normalized targets, and MAE. In this alternative protocol, we instead used eight Monte Carlo splits and eight random seeds per split, yielding 64 matched runs for each method and problem. For each split, 75\% of the samples were used for training and 25\% for testing, the target was standardized, and prediction accuracy was evaluated using MSE.

  The target standardization parameters were estimated from the training split and then applied to the corresponding test split. The synthetic inputs retained their prescribed $[0,1]^n$ domains. For the real-world datasets, input-scaling parameters were estimated from the training split only and used to map the inputs to $[0,1]^n$; the same transformations were then applied to the test split. This differs from Heider's input range of $[-1,1]^n$ because the present KACS formulation assumes inputs in $[0,1]^n$. Thus, this experiment follows the principal elements of Heider's evaluation design but is not an exact reproduction of all preprocessing choices.

  XCSF and KACS used identical data splits and random seeds. Their hyperparameters, population budget ($N=6400$), learning duration, and absence of post-training compaction were otherwise unchanged from the main experiments. Learning was disabled when evaluating the held-out test samples.}

\subsection{Main Results}

{Table~\ref{tb:heider_protocol_result_sup} summarizes prediction accuracy, model complexity, and AIC under the alternative protocol. KACS obtains lower mean training and testing MSE than XCSF on all eight benchmark problems. The advantage is especially pronounced on the four synthetic functions, ASN, and CS, while the difference is smaller but remains statistically significant on CCPP and EEC. The across-problem paired Wilcoxon signed-rank test applied to the 64-run mean for each benchmark also favors KACS for both training and testing MSE ($p=0.007813$).

  KACS also uses fewer macro-rules on seven of the eight problems. CCPP is the only exception: XCSF uses 1898 rules on average and KACS uses 1927, with no significant difference between them. Despite these similar rule counts, KACS uses substantially fewer parameters on CCPP because every KACS rule has only two consequent parameters. Across all eight problems, KACS uses only approximately 5.9\%--40.6\% as many parameters as XCSF and obtains a lower AIC. The across-problem tests favor KACS for rule count ($p=0.01563$), parameter count ($p=0.007813$), and AIC ($p=0.007813$).}

\FloatBarrier
\begin{table*}[h]
  \centering
  \caption{{Summary of the XCSF--KACS Comparison under the Alternative
      Evaluation Protocol. Values Are Means over 64 Matched Runs Consisting
      of Eight Monte Carlo Splits and Eight Random Seeds per Split. Training
      and Testing Errors Are MSE in the Standardized Target Space. Green
      Cells Indicate the Lower Mean. For Each Problem, Symbols
      $+/-/\sim$ Denote That XCSF Is Significantly Better/Worse/Similar to
      KACS Based on a Two-Sided Paired Wilcoxon Signed-Rank Test. Rank Is
      the Average Rank across the Eight Problems. The $p$-Values in the
      Final Row Are from Across-Problem Paired Wilcoxon Signed-Rank Tests
      Applied to the 64-Run Mean Values. Statistical Significance Is at
  $\alpha=0.05$ ($\dag$)}}
  \label{tb:heider_protocol_result_sup}
  \renewcommand{\arraystretch}{1.2}
  \resizebox{\width}{!}{
    \small
     \begin{tabular}{c|cc|cc|cc|cc|cc}
      \bhline{1pt}
&\multicolumn{4}{c|}{\textsc{Model Accuracy}}
      &\multicolumn{4}{c|}{\textsc{Model Complexity}}
      &\multicolumn{2}{c}{\textsc{Trade-Off}}\\[-0.2ex]\rule{0pt}{2.6ex}
      &\multicolumn{2}{c|}{\text{Training MSE}}
      &\multicolumn{2}{c|}{\text{Testing MSE}}
      &\multicolumn{2}{c|}{\text{\#Rules} $(|\mathcal{P}|)$}
      &\multicolumn{2}{c|}{\text{\#Parameters} $(k)$}
      &\multicolumn{2}{c}{\text{AIC}}\\
 & XCSF & KACS & XCSF & KACS & XCSF & KACS & XCSF & KACS & XCSF & KACS \\
 \bhline{1pt}
$f_1$ & 0.6028 $-$ & \cellcolor{g}0.5683 & 0.8825 $-$ & \cellcolor{g}0.6655 & 4198 $-$ & \cellcolor{g}1356 & 46175 $-$ & \cellcolor{g}2712 & 91970 $-$ & \cellcolor{g}4999 \\
$f_2$ & 0.4405 $-$ & \cellcolor{g}0.2622 & 0.7426 $-$ & \cellcolor{g}0.3226 & 4178 $-$ & \cellcolor{g}1485 & 45954 $-$ & \cellcolor{g}2971 & 91292 $-$ & \cellcolor{g}4919 \\
$f_3$ & 0.5968 $-$ & \cellcolor{g}0.2870 & 0.8707 $-$ & \cellcolor{g}0.3654 & 4174 $-$ & \cellcolor{g}1444 & 45919 $-$ & \cellcolor{g}2888 & 91450 $-$ & \cellcolor{g}4797 \\
$f_4$ & 0.4797 $-$ & \cellcolor{g}0.2223 & 0.7784 $-$ & \cellcolor{g}0.2795 & 4188 $-$ & \cellcolor{g}1420 & 46069 $-$ & \cellcolor{g}2839 & 91583 $-$ & \cellcolor{g}4529 \\
\text{ASN} & 0.3203 $-$ & \cellcolor{g}0.1285 & 0.3232 $-$ & \cellcolor{g}0.1456 & 2000 $-$ & \cellcolor{g}1434 & 12000 $-$ & \cellcolor{g}2868 & 22688 $-$ & \cellcolor{g}3416 \\
\text{CCPP} & 0.0680 $-$ & \cellcolor{g}0.0631 & 0.0699 $-$ & \cellcolor{g}0.0653 & \cellcolor{g}1898 $\sim$ & 1927 & 9490 $-$ & \cellcolor{g}3854 & -310.4 $-$ & \cellcolor{g}-12171 \\
\text{CS} & 0.1955 $-$ & \cellcolor{g}0.1223 & 0.2098 $-$ & \cellcolor{g}0.1491 & 2837 $-$ & \cellcolor{g}1416 & 25535 $-$ & \cellcolor{g}2831 & 49788 $-$ & \cellcolor{g}4034 \\
\text{EEC} & 0.0770 $-$ & \cellcolor{g}0.0510 & 0.0842 $-$ & \cellcolor{g}0.0593 & 2985 $-$ & \cellcolor{g}1014 & 26867 $-$ & \cellcolor{g}2028 & 52228 $-$ & \cellcolor{g}2319 \\
\bhline{1pt}
Rank & \textit{2.00}$\downarrow$$^\dag$ & \cellcolor{g}\textit{1.00} & \textit{2.00}$\downarrow$$^\dag$ & \cellcolor{g}\textit{1.00} & \textit{1.88}$\downarrow$$^\dag$ & \cellcolor{g}\textit{1.12} & \textit{2.00}$\downarrow$$^\dag$ & \cellcolor{g}\textit{1.00} & \textit{2.00}$\downarrow$$^\dag$ & \cellcolor{g}\textit{1.00} \\
$+/-/\sim$ & 0/8/0 & - & 0/8/0 & - & 0/7/1 & - & 0/8/0 & - & 0/8/0 & - \\
$p$-value & 0.007813 & - & 0.007813 & - & 0.01563 & - & 0.007813 & - & 0.007813 & - \\
\bhline{1pt}
\end{tabular}

  }
\end{table*}
\clearpage

\subsection{Statistical Results and Interpretation}

{Table~\ref{tb:heider_protocol_statistics_sup} provides the means, standard deviations, paired Wilcoxon $p$-values, and matched-pairs rank-biserial correlations for all benchmark--metric combinations. All 16 training- and testing-MSE comparisons favor KACS significantly ($p<0.001$). The corresponding effect sizes are positive and range from 0.5452 to 1 for training MSE and from 0.7635 to 1 for testing MSE. The parameter-count and AIC comparisons have an effect size of 1 on every benchmark. The rule-count comparisons also have an effect size of 1 except on CCPP, where the difference is nonsignificant ($p=0.4458$; effect size $=-0.1101$).

  Compared with the MAE-based results in the main article, the alternative protocol produces a more uniform accuracy advantage for KACS. Because KACS also obtains lower training MSE on every benchmark, this pattern cannot be explained solely by improved generalization from a comparable training fit; under this protocol, KACS fits both the training and test data more accurately according to squared error. These results strengthen the evidence that the favorable KACS--XCSF comparison is robust to a different target transformation, error metric, train--test ratio, and repetition design.

  Nevertheless, the larger performance difference cannot be attributed to any single protocol component because target standardization, MSE, and the train--test ratio were changed simultaneously. In particular, MSE places greater weight on large residuals than MAE, and the smaller training fraction changes the amount of data available for fitting local models. Moreover, absolute MSE and AIC values should not be compared directly with those in the main experiments because the target scale and training-set size differ. We therefore treat this experiment as a robustness analysis rather than a replacement for the controlled main evaluation.}

\begin{table*}[t]
  \centering
  \caption{{Detailed Paired Comparison of XCSF and KACS under the
      Alternative Evaluation Protocol. Values Are Reported as the Mean
  $\pm$ Standard Deviation over 64 Matched Runs. Other Notation Follows Table \ref{tb:supplementary-statistics sup}}}
  \label{tb:heider_protocol_statistics_sup}
  \renewcommand{\arraystretch}{1.2}
  \resizebox{\width}{!}{
    \small
    \begin{tabular}{cccccc}
  \bhline{1pt}
Benchmark & Metric & XCSF mean $\pm$ SD & KACS mean $\pm$ SD & Wilcoxon $p$ & Effect size \\
 \bhline{1pt}
$f_1$ & Training MSE & 0.6028 $\pm$ 0.0533 & \cellcolor{g}0.5683 $\pm$ 0.0640 & 0.0001516 & 0.5452 \\
$f_1$ & Testing MSE & 0.8825 $\pm$ 0.0866 & \cellcolor{g}0.6655 $\pm$ 0.0866 & 5.52E-12 & 0.9913 \\
$f_1$ & \#Rules & 4198 $\pm$ 62.85 & \cellcolor{g}1356 $\pm$ 41.14 & 3.60E-12 & 1 \\
$f_1$ & \#Parameters & 46175 $\pm$ 691.4 & \cellcolor{g}2712 $\pm$ 82.28 & 3.61E-12 & 1 \\
$f_1$ & AIC & 91970 $\pm$ 1370 & \cellcolor{g}4999 $\pm$ 179.8 & 3.61E-12 & 1 \\
\bhline{1pt}
$f_2$ & Training MSE & 0.4405 $\pm$ 0.0440 & \cellcolor{g}0.2622 $\pm$ 0.0714 & 7.17E-11 & 0.9375 \\
$f_2$ & Testing MSE & 0.7426 $\pm$ 0.1222 & \cellcolor{g}0.3226 $\pm$ 0.0864 & 3.61E-12 & 1 \\
$f_2$ & \#Rules & 4178 $\pm$ 67.62 & \cellcolor{g}1485 $\pm$ 31.24 & 3.61E-12 & 1 \\
$f_2$ & \#Parameters & 45954 $\pm$ 743.9 & \cellcolor{g}2971 $\pm$ 62.49 & 3.61E-12 & 1 \\
$f_2$ & AIC & 91292 $\pm$ 1463 & \cellcolor{g}4919 $\pm$ 207.2 & 3.61E-12 & 1 \\
\bhline{1pt}
$f_3$ & Training MSE & 0.5968 $\pm$ 0.0550 & \cellcolor{g}0.2870 $\pm$ 0.0984 & 3.79E-12 & 0.999 \\
$f_3$ & Testing MSE & 0.8707 $\pm$ 0.0765 & \cellcolor{g}0.3654 $\pm$ 0.1201 & 3.61E-12 & 1 \\
$f_3$ & \#Rules & 4174 $\pm$ 51.89 & \cellcolor{g}1444 $\pm$ 35.87 & 3.61E-12 & 1 \\
$f_3$ & \#Parameters & 45919 $\pm$ 570.8 & \cellcolor{g}2888 $\pm$ 71.74 & 3.61E-12 & 1 \\
$f_3$ & AIC & 91450 $\pm$ 1121 & \cellcolor{g}4797 $\pm$ 307.1 & 3.61E-12 & 1 \\
\bhline{1pt}
$f_4$ & Training MSE & 0.4797 $\pm$ 0.0611 & \cellcolor{g}0.2223 $\pm$ 0.0649 & 3.61E-12 & 1 \\
$f_4$ & Testing MSE & 0.7784 $\pm$ 0.0702 & \cellcolor{g}0.2795 $\pm$ 0.0783 & 3.61E-12 & 1 \\
$f_4$ & \#Rules & 4188 $\pm$ 61.59 & \cellcolor{g}1420 $\pm$ 32.77 & 3.60E-12 & 1 \\
$f_4$ & \#Parameters & 46069 $\pm$ 677.5 & \cellcolor{g}2839 $\pm$ 65.54 & 3.61E-12 & 1 \\
$f_4$ & AIC & 91583 $\pm$ 1336 & \cellcolor{g}4529 $\pm$ 217.0 & 3.61E-12 & 1 \\
\bhline{1pt}
\text{ASN} & Training MSE & 0.3203 $\pm$ 0.0720 & \cellcolor{g}0.1285 $\pm$ 0.0182 & 3.61E-12 & 1 \\
\text{ASN} & Testing MSE & 0.3232 $\pm$ 0.0792 & \cellcolor{g}0.1456 $\pm$ 0.0243 & 3.97E-12 & 0.9981 \\
\text{ASN} & \#Rules & 2000 $\pm$ 410.8 & \cellcolor{g}1434 $\pm$ 40.38 & 3.61E-12 & 1 \\
\text{ASN} & \#Parameters & 12000 $\pm$ 2465 & \cellcolor{g}2868 $\pm$ 80.77 & 3.61E-12 & 1 \\
\text{ASN} & AIC & 22688 $\pm$ 4671 & \cellcolor{g}3416 $\pm$ 191.6 & 3.61E-12 & 1 \\
\bhline{1pt}
\text{CCPP} & Training MSE & 0.0680 $\pm$ 0.0016 & \cellcolor{g}0.0631 $\pm$ 0.0090 & 1.24E-08 & 0.8192 \\
\text{CCPP} & Testing MSE & 0.0699 $\pm$ 0.0042 & \cellcolor{g}0.0653 $\pm$ 0.0086 & 3.13E-08 & 0.7962 \\
\text{CCPP} & \#Rules & \cellcolor{g}1898 $\pm$ 297.1 & 1927 $\pm$ 45.26 & 0.4458 & -0.1101 \\
\text{CCPP} & \#Parameters & 9490 $\pm$ 1485 & \cellcolor{g}3854 $\pm$ 90.51 & 3.61E-12 & 1 \\
\text{CCPP} & AIC & -310.4 $\pm$ 2911 & \cellcolor{g}-12171 $\pm$ 816.7 & 3.61E-12 & 1 \\
\bhline{1pt}
\text{CS} & Training MSE & 0.1955 $\pm$ 0.0471 & \cellcolor{g}0.1223 $\pm$ 0.0186 & 7.50E-11 & 0.9365 \\
\text{CS} & Testing MSE & 0.2098 $\pm$ 0.0461 & \cellcolor{g}0.1491 $\pm$ 0.0271 & 5.10E-10 & 0.8942 \\
\text{CS} & \#Rules & 2837 $\pm$ 434.5 & \cellcolor{g}1416 $\pm$ 42.30 & 3.61E-12 & 1 \\
\text{CS} & \#Parameters & 25535 $\pm$ 3910 & \cellcolor{g}2831 $\pm$ 84.60 & 3.61E-12 & 1 \\
\text{CS} & AIC & 49788 $\pm$ 7663 & \cellcolor{g}4034 $\pm$ 197.8 & 3.61E-12 & 1 \\
\bhline{1pt}
\text{EEC} & Training MSE & 0.0770 $\pm$ 0.0214 & \cellcolor{g}0.0510 $\pm$ 0.0155 & 2.59E-08 & 0.801 \\
\text{EEC} & Testing MSE & 0.0842 $\pm$ 0.0236 & \cellcolor{g}0.0593 $\pm$ 0.0196 & 1.12E-07 & 0.7635 \\
\text{EEC} & \#Rules & 2985 $\pm$ 391.4 & \cellcolor{g}1014 $\pm$ 35.34 & 3.61E-12 & 1 \\
\text{EEC} & \#Parameters & 26867 $\pm$ 3523 & \cellcolor{g}2028 $\pm$ 70.68 & 3.61E-12 & 1 \\
\text{EEC} & AIC & 52228 $\pm$ 6897 & \cellcolor{g}2319 $\pm$ 231.5 & 3.61E-12 & 1 \\
 \bhline{1pt}
 \end{tabular}

  }
\end{table*}
\clearpage

\section{Additional Learning Curves}
\label{sec:Additional Learning Curves sup}

\subsection{Training MAE Learning Curves}
\label{ss:Training MAE Learning Curves sup}
Fig.~\ref{fig:all_mae sup} presents the training MAE curves for all eight benchmark problems.

\begin{figure*}[h]
  \centering
  \subfloat[$f_1$: Rastrigin Function]{
    \includegraphics[width=0.24\textwidth]{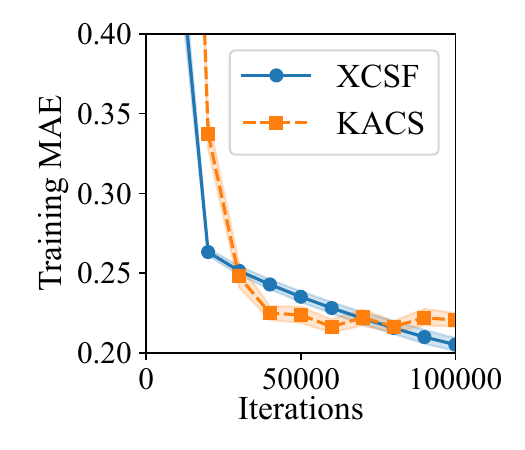}
    \label{fig:f1_mae_train sup}
  }
  \subfloat[$f_2$: Rosenbrock Function]{\includegraphics[width=0.24\textwidth]{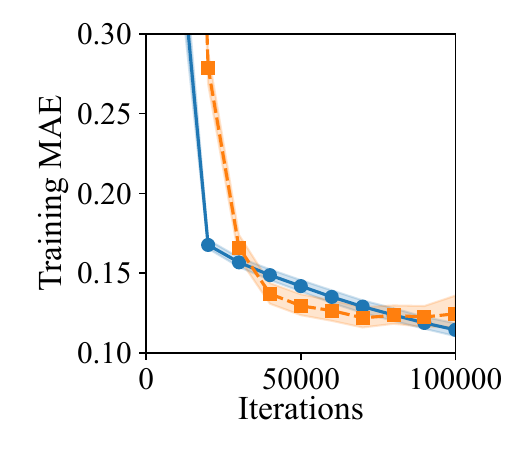} \label{fig:f2_mae_train sup}}
  \subfloat[$f_3$: Cross Function]{\includegraphics[width=0.24\textwidth]{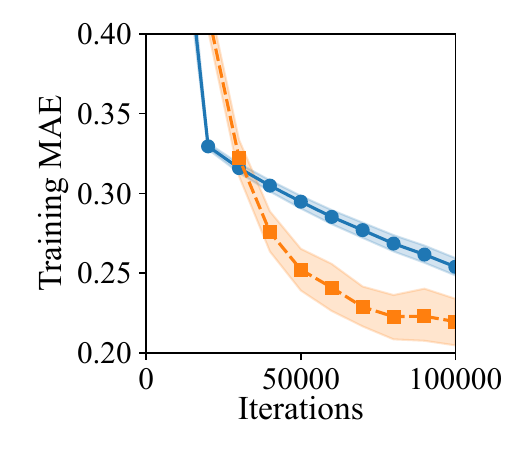} \label{fig:f3_mae_train sup}}
  \subfloat[$f_4$: Styblinski-Tang Function]{\includegraphics[width=0.24\textwidth]{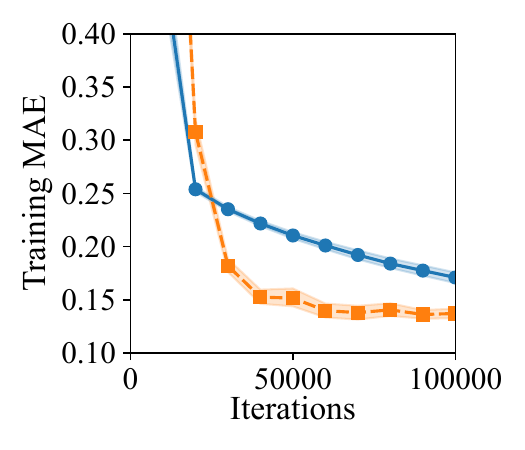} \label{fig:f4_mae_train sup}}\\
  \subfloat[ASN]{
    \includegraphics[width=0.24\textwidth]{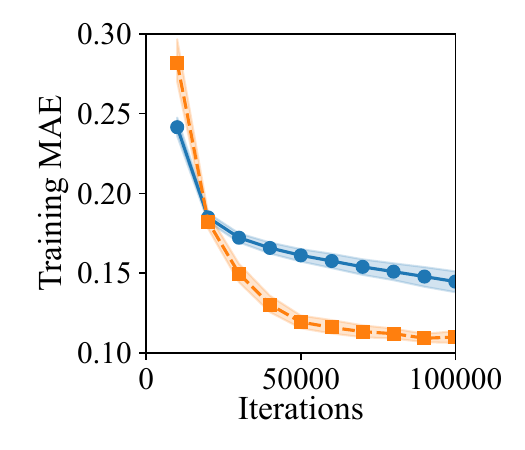}
    \label{fig:asn_mae_train sup}
  }
  \subfloat[CCPP]{\includegraphics[width=0.24\textwidth]{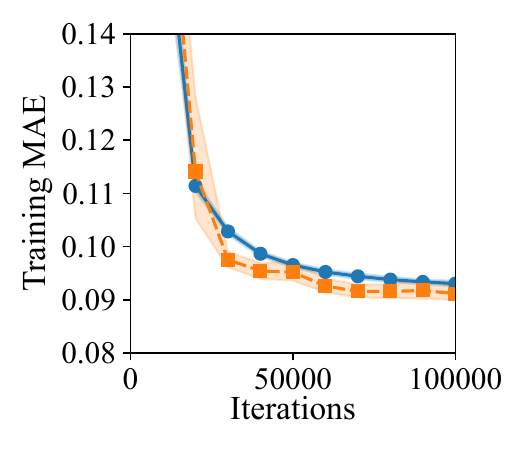} \label{fig:ccpp_mae_train sup}}
  \subfloat[CS]{\includegraphics[width=0.24\textwidth]{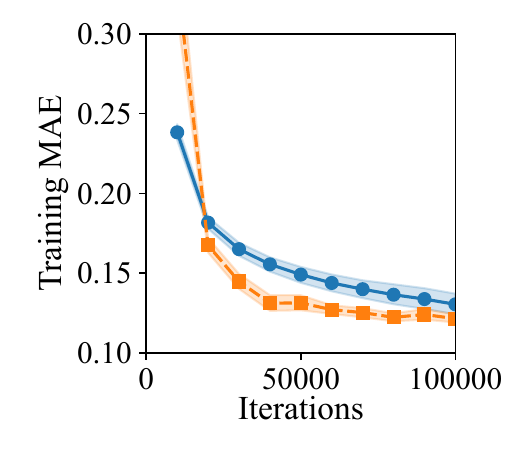} \label{fig:cs_mae_train sup}}
  \subfloat[EEC]{\includegraphics[width=0.24\textwidth]{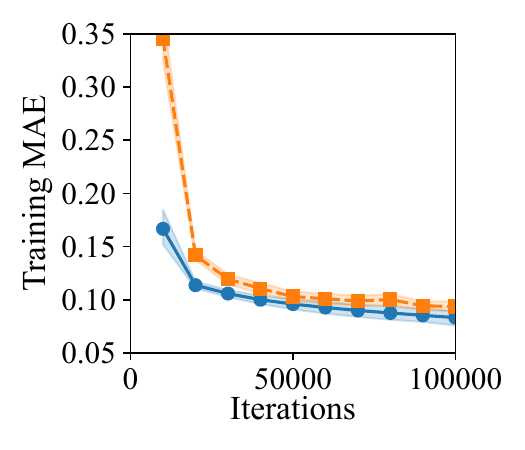} \label{fig:eec_mae_train sup}}
  \caption{Training MAE learning curves for all eight benchmark problems. Curves show the mean over 30 runs, and shaded regions denote 95\% confidence intervals.}
  \label{fig:all_mae sup}
\end{figure*}

\subsection{Testing MAE Learning Curves}
\label{ss:Testing MAE Learning Curves sup}
\label{others-7}
{Fig.~\ref{fig:all_mae test sup} presents the testing MAE curves for all eight benchmark problems. Note that Figs.~\ref{fig:f1_mae sup}--\ref{fig:f4_mae sup} reproduce Figs.~\ref{fig:f1_mae}--\ref{fig:f4_mae} from the main article.}

\begin{figure*}[h]
  \centering
  \subfloat[$f_1$: Rastrigin Function]{
    \includegraphics[width=0.24\textwidth]{fig/mae/rastrigin_10dim_1000_Phash0.pdf}
    \label{fig:f1_mae sup}
  }
  \subfloat[$f_2$: Rosenbrock Function]{\includegraphics[width=0.24\textwidth]{fig/mae/rosenbrock_10dim_1000_Phash0.pdf} \label{fig:f2_mae sup}}
  \subfloat[$f_3$: Cross Function]{\includegraphics[width=0.24\textwidth]{fig/mae/f3_10dim_1000_Phash0.pdf} \label{fig:f3_mae sup}}
  \subfloat[$f_4$: Styblinski-Tang Function]{\includegraphics[width=0.24\textwidth]{fig/mae/f4_10dim_1000_Phash0.pdf} \label{fig:f4_mae sup}}
  \\
  \subfloat[ASN]{
    \includegraphics[width=0.24\textwidth]{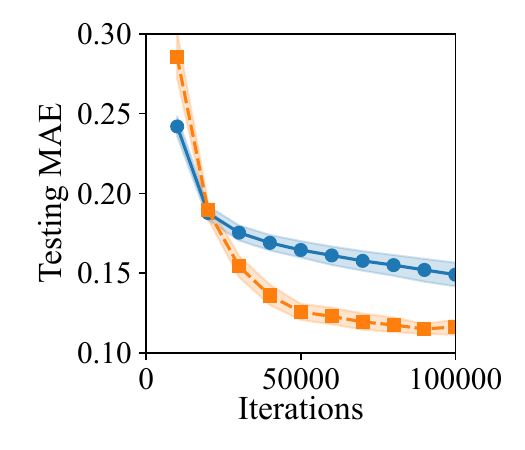}
    \label{fig:asn_mae sup}
  }
  \subfloat[CCPP]{\includegraphics[width=0.24\textwidth]{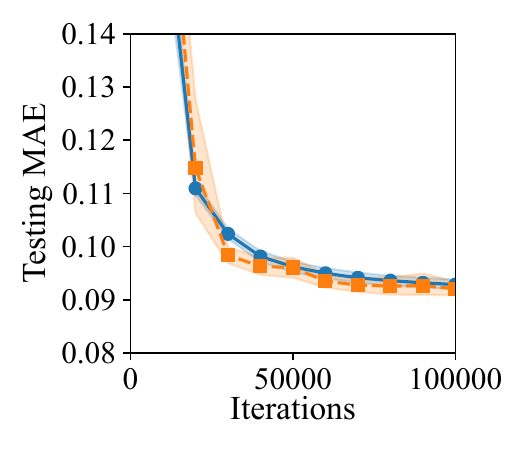} \label{fig:ccpp_mae sup}}
  \subfloat[CS]{\includegraphics[width=0.24\textwidth]{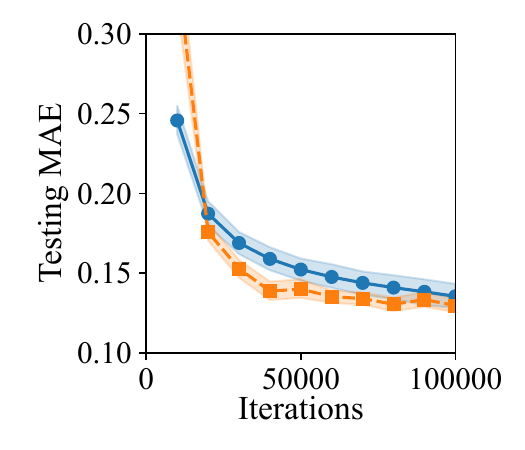} \label{fig:cs_mae sup}}
  \subfloat[EEC]{\includegraphics[width=0.24\textwidth]{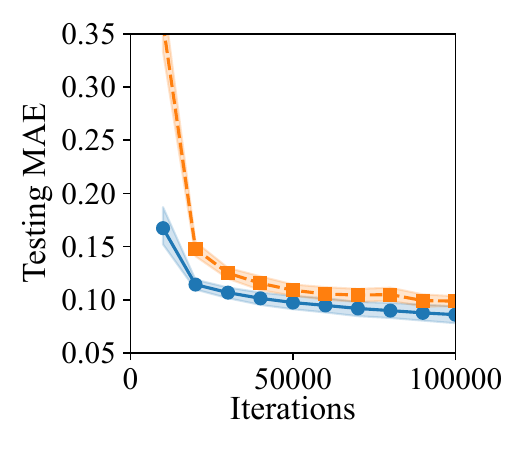} \label{fig:eec_mae sup}}
  \caption{Testing MAE learning curves for all eight benchmark problems. Curves show the mean over 30 runs, and shaded regions denote 95\% confidence intervals.}
  \label{fig:all_mae test sup}
\end{figure*}

\clearpage
\subsection{Parameter-Count Learning Curves}
\label{ss:Parameter-Count Learning Curves sup}
Fig.~\ref{fig:all_nop sup} summarizes the parameter-count trajectories during training. Compared with MAE reduction, parameter growth remains controlled, suggesting that KACS improves predictive accuracy without relying on uncontrolled model-size expansion. This compactness trend is observed consistently over all benchmark problems.

\begin{figure*}[h]
  \centering
  \subfloat[$f_1$: Rastrigin Function]{
    \includegraphics[width=0.24\textwidth]{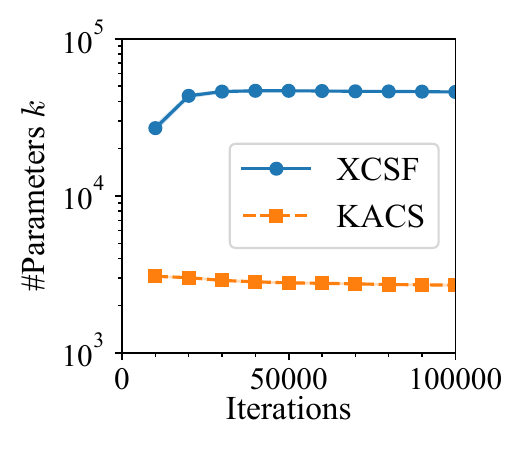}
    \label{fig:f1_nop sup}
  }
  \subfloat[$f_2$: Rosenbrock Function]{\includegraphics[width=0.24\textwidth]{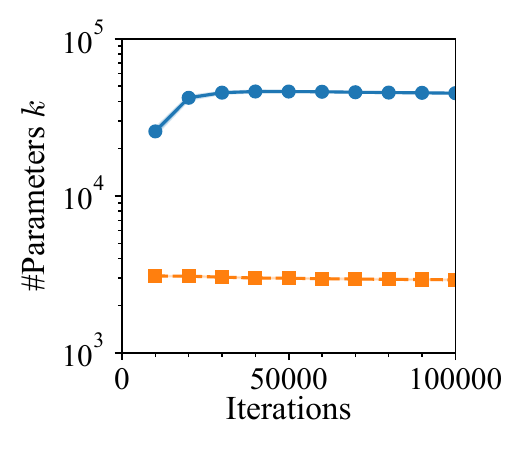} \label{fig:f2_nop sup}}
  \subfloat[$f_3$: Cross Function]{\includegraphics[width=0.24\textwidth]{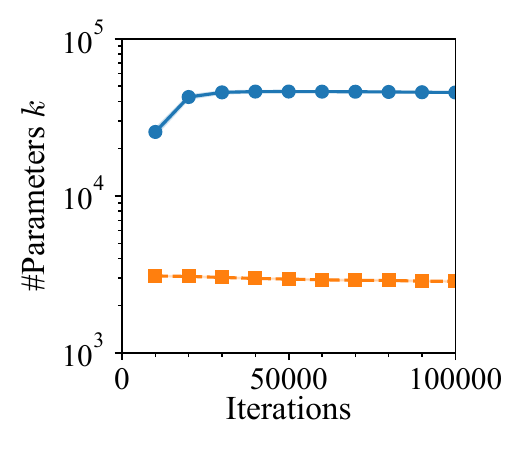} \label{fig:f3_nop sup}}
  \subfloat[$f_4$: Styblinski-Tang Function]{\includegraphics[width=0.24\textwidth]{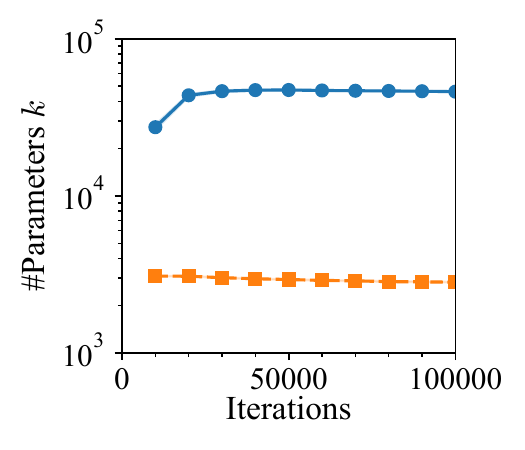} \label{fig:f4_nop sup}}\\
  \subfloat[ASN]{
    \includegraphics[width=0.24\textwidth]{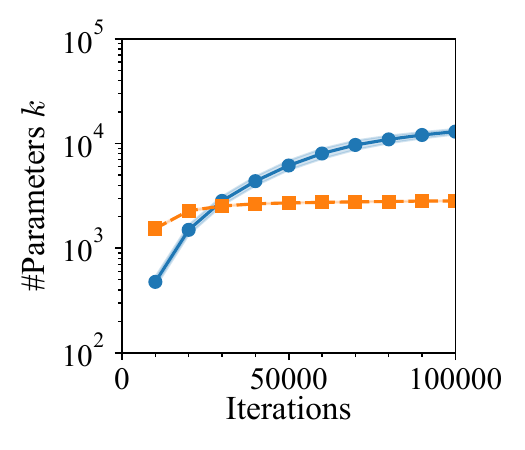}
    \label{fig:asn_nop sup}
  }
  \subfloat[CCPP]{\includegraphics[width=0.24\textwidth]{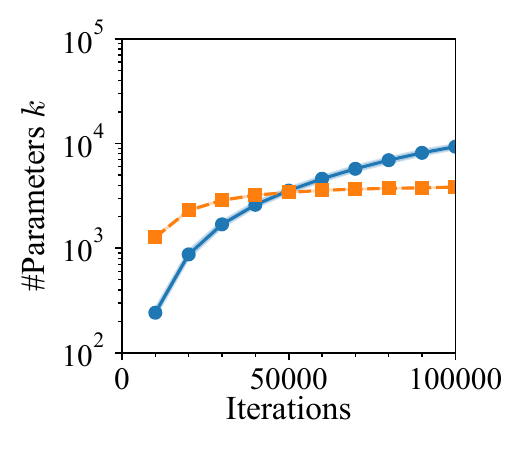} \label{fig:ccpp_nop sup}}
  \subfloat[CS]{\includegraphics[width=0.24\textwidth]{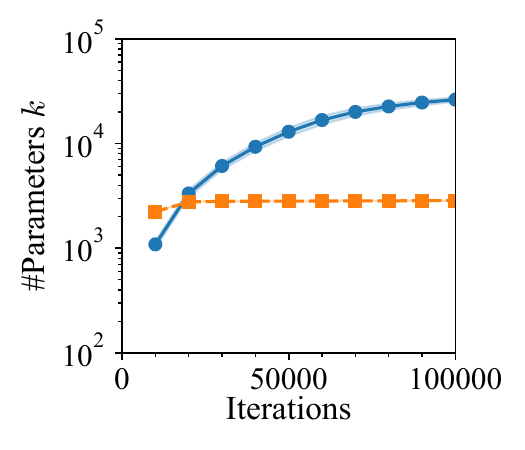} \label{fig:cs_nop sup}}
  \subfloat[EEC]{\includegraphics[width=0.24\textwidth]{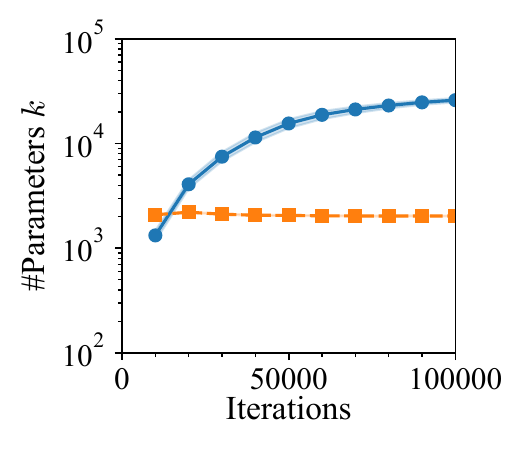} \label{fig:eec_nop sup}}
  \caption{Parameter-count learning curves for all eight benchmark problems. Curves show the mean over 30 runs, and shaded regions denote 95\% confidence intervals.}
  \label{fig:all_nop sup}
\end{figure*}

\clearpage

\clearpage

\section{Effects of Rotation, Translation, and Scaling on $f_3$}
\label{sec:rotation_sup}
\label{r2-13-2}

The original $f_3$ (cf. Fig. \ref{fig:f3 sup}) has ridges aligned with the coordinate axes, which would be a favorable structure for both KACS and XCSF. To test whether the advantage of KACS depends on this alignment, we evaluate both methods on a rotated variant of $f_3$ in which the ridges are tilted by $45^\circ$ relative to the coordinate axes. {To complement this primary rotation analysis, we also define translated and scaled variants of $f_3$ below.}

\subsection{Problem Definition}
Define two aggregate variables
\begin{equation}
  a = \frac{1}{\lfloor n/2 \rfloor} \sum_{i=1}^{\lfloor n/2 \rfloor} {x}_i, \qquad
  b = \frac{1}{\lceil n/2 \rceil} \sum_{i=\lfloor n/2 \rfloor+1}^{n} {x}_i.
  \label{eq:ab}
\end{equation}
The $45^\circ$-rotated coordinates are
\begin{equation}
  a_\mathrm{rot} = \frac{a + b}{\sqrt{2}}, \qquad
  b_\mathrm{rot} = \frac{-a + b}{\sqrt{2}}.
  \label{eq:rot}
\end{equation}
The rotated $f_3$ is then defined as
\begin{equation}
  f_3^\mathrm{rot}(\mathbf{x}) = \max\!\left(
    \exp(-10\, a_\mathrm{rot}^2),\;
    \exp(-50\, b_\mathrm{rot}^2),\;
    1.25\, \exp\left(-5(a^2 + b^2)\right)
  \right).
  \label{eq:f3rot}
\end{equation}
The third term $1.25\,\exp\left(-5(a^2+b^2)\right)$ does not change under rotation. Fig.~\ref{fig:f3_rotated sup} illustrates the rotated $f_3$ for $n=2$.

{
  The translated variant shifts every input coordinate by $0.20$, i.e.,
  \begin{equation}
    z_i^\mathrm{tr}=x_i-0.20,
    \qquad i=1,\ldots,n,
    \label{eq:f3_translation}
  \end{equation}
  and
  \begin{equation}
    a_\mathrm{tr}
    =\frac{1}{\lfloor n/2\rfloor}
    \sum_{i=1}^{\lfloor n/2\rfloor}z_i^\mathrm{tr},
    \qquad
    b_\mathrm{tr}
    =\frac{1}{\lceil n/2\rceil}
    \sum_{i=\lfloor n/2\rfloor+1}^{n}z_i^\mathrm{tr}.
    \label{eq:f3_translation_ab}
  \end{equation}
  The translated $f_3$ is
  \begin{equation}
    f_3^\mathrm{tr}(\mathbf{x})
    =\max\!\left(
      \exp(-10a_\mathrm{tr}^2),\;
      \exp(-50b_\mathrm{tr}^2),\;
      1.25\exp\!\left[-5\left(a_\mathrm{tr}^2+b_\mathrm{tr}^2\right)\right]
    \right).
    \label{eq:f3_translated}
  \end{equation}
  Fig.~\ref{fig:f3_translated sup} illustrates the translated $f_3$ for $n=2$.

  The scaled variant scales every input coordinate by a factor of $1.25$, i.e.,
  \begin{equation}
    z_i^\mathrm{sc}=1.25x_i,
    \qquad i=1,\ldots,n,
    \label{eq:f3_scaling}
  \end{equation}
  and
  \begin{equation}
    a_\mathrm{sc}
    =\frac{1}{\lfloor n/2\rfloor}
    \sum_{i=1}^{\lfloor n/2\rfloor}z_i^\mathrm{sc},
    \qquad
    b_\mathrm{sc}
    =\frac{1}{\lceil n/2\rceil}
    \sum_{i=\lfloor n/2\rfloor+1}^{n}z_i^\mathrm{sc}.
    \label{eq:f3_scaling_ab}
  \end{equation}
  The scaled $f_3$ is
  \begin{equation}
    f_3^\mathrm{sc}(\mathbf{x})
    =\max\!\left(
      \exp(-10a_\mathrm{sc}^2),\;
      \exp(-50b_\mathrm{sc}^2),\;
      1.25\exp\!\left[-5\left(a_\mathrm{sc}^2+b_\mathrm{sc}^2\right)\right]
    \right).
    \label{eq:f3_scaled}
  \end{equation}
Fig.~\ref{fig:f3_scaled sup} illustrates the scaled $f_3$ for $n=2$.}

All other experimental conditions follow the main article: $n=10$, $N_S=1000$, and 30 independent runs.

\subsection{Experimental Results}
Figs.~\ref{fig:f3 test sup} and \ref{fig:f3_rotated test sup} show the testing MAE learning curves for the original and rotated $f_3$, respectively. On both problems, KACS achieves significantly lower testing MAE than XCSF at the end of training. After rotation, KACS accuracy improves slightly ($0.2329 \to 0.2184$) while XCSF accuracy degrades ($0.3246 \to 0.3430$), widening the gap between the two methods from $0.0917$ to $0.1246$.

{
  Figs.~\ref{fig:f3_shifted test sup} and \ref{fig:f3_scaled test sup} show the corresponding results for the translated and scaled variants. On the translated $f_3$, the final testing MAE is $0.2901$ for XCSF and $0.2478$ for KACS, whereas on the scaled $f_3$, it is $0.3911$ and $0.2488$, respectively. KACS therefore retains significantly lower testing MAE under both transformations. Relative to the original $f_3$, the accuracy gap narrows from $0.0917$ to $0.0423$ after translation but widens to $0.1423$ after scaling.
}

\subsection{Discussion}
XCSF uses axis-aligned rectangles to cover the input space. These rectangles fit well when the ridges of $f_3$ run along the coordinate axes, but become less efficient when the ridges are tilted, requiring more rules to cover the diagonal structure. This is why XCSF accuracy degrades after rotation (testing MAE: $0.3246 \to 0.3430$).

KACS decomposes the target function into sums of one-dimensional components following the KA theorem (cf.~\eqref{eq:ka_theorem} in the main article). This additive structure is compatible with both the original and rotated ridges of $f_3$, so the rotation does not disadvantage KACS. As a result, the accuracy gap between the two methods widens after rotation ($0.0917 \to 0.1246$), showing that the advantage of KACS does not depend on axis-aligned function structure.

{
  Translation preserves the orientation and width of the ridges while moving their location by $0.20$ and therefore introduces no orientation mismatch. Although the accuracy gap becomes smaller, KACS remains significantly more accurate, indicating that its advantage is not restricted to a function centered at the origin.

  Scaling the inputs by $1.25$ contracts the ridge widths to $80\%$ of their original values, requiring finer spatial resolution. The larger degradation of XCSF is consistent with the need for finer multidimensional partitions, whereas KACS remains comparatively stable through its one-dimensional submodels.

  Together with the rotation results, these findings show that the advantage of KACS persists across changes in the orientation, location, and spatial scale of $f_3$.
}

\begin{figure}[h]
  \centering
  \subfloat[$f_3$]{
    \syntheticcrossplot
    \label{fig:f3 sup}
  }\hfill
  \subfloat[Rotated $f_3$]{
    \begin{tikzpicture}[
        declare function={
          crossfunc(\x,\y) = max(
            exp(-10*((\x+\y)/sqrt(2))*((\x+\y)/sqrt(2))),
            exp(-50*((-\x+\y)/sqrt(2))*((-\x+\y)/sqrt(2))),
            1.25*exp(-5*(\x*\x + \y*\y))
          );
        }
      ]
      \begin{axis}[
          width=0.22\textwidth,
          height=0.165\textwidth,
          xlabel={$x_1$}, ylabel={$x_2$}, zlabel={$f_3^\mathrm{rot}(x_1,x_2)$},
          view={60}{30},
          colormap/viridis,
          xmin=-1.414, xmax=1.414,
          ymin=-1.414, ymax=1.414,
          samples=30, samples y=30,
        ]
        \addplot3[
          surf,
          domain=-1.414:1.414,
          domain y=-1.414:1.414,
        ]
        {crossfunc(x,y)};
      \end{axis}
    \end{tikzpicture}
    \label{fig:f3_rotated sup}
  }\hfill
  \subfloat[Translated $f_3$]{
    \begin{tikzpicture}[
        declare function={
          crossfunc(\x,\y) = max(
            exp(-10*(\x-0.20)*(\x-0.20)),
            exp(-50*(\y-0.20)*(\y-0.20)),
            1.25*exp(-5*((\x-0.20)*(\x-0.20) + (\y-0.20)*(\y-0.20)))
          );
        }
      ]
      \begin{axis}[
          width=0.22\textwidth,
          height=0.165\textwidth,
          xlabel={$x_1$}, ylabel={$x_2$}, zlabel={$f_3^\mathrm{tr}(x_1,x_2)$},
          view={60}{30},
          colormap/viridis,
          xmin=-1, xmax=1,
          ymin=-1, ymax=1,
          samples=30, samples y=30,
        ]
        \addplot3[
          surf,
          domain=-1:1,
          domain y=-1:1,
        ]
        {crossfunc(x,y)};
      \end{axis}
    \end{tikzpicture}
    \label{fig:f3_translated sup}
  }\hfill
  \subfloat[Scaled $f_3$]{
    \begin{tikzpicture}[
        declare function={
          crossfunc(\x,\y) = max(
            exp(-10*(1.25*\x)*(1.25*\x)),
            exp(-50*(1.25*\y)*(1.25*\y)),
            1.25*exp(-5*((1.25*\x)*(1.25*\x) + (1.25*\y)*(1.25*\y)))
          );
        }
      ]
      \begin{axis}[
          width=0.22\textwidth,
          height=0.165\textwidth,
          xlabel={$x_1$}, ylabel={$x_2$}, zlabel={$f_3^\mathrm{sc}(x_1,x_2)$},
          view={60}{30},
          colormap/viridis,
          xmin=-1, xmax=1,
          ymin=-1, ymax=1,
          samples=30, samples y=30,
        ]
        \addplot3[
          surf,
          domain=-1:1,
          domain y=-1:1,
        ]
        {crossfunc(x,y)};
      \end{axis}
    \end{tikzpicture}
    \label{fig:f3_scaled sup}
  }\\
  \subfloat[Testing MAE on $f_3$]{\includegraphics[width=0.23\textwidth]{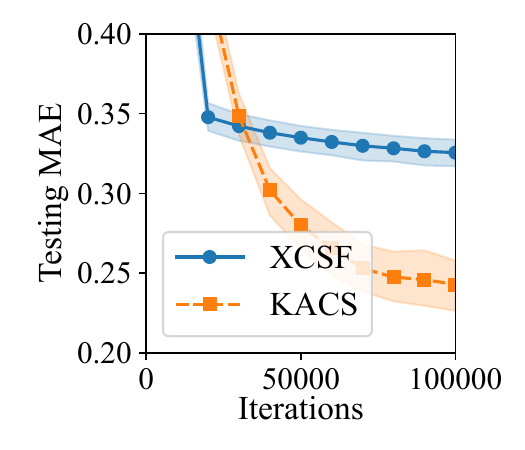} \label{fig:f3 test sup}}\hfill
  \subfloat[Testing MAE on Rotated $f_3$]{\includegraphics[width=0.23\textwidth]{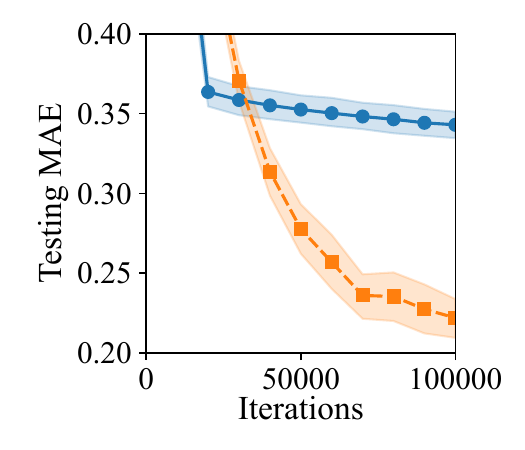} \label{fig:f3_rotated test sup}}\hfill
  \subfloat[Testing MAE on Translated $f_3$]{\includegraphics[width=0.23\textwidth]{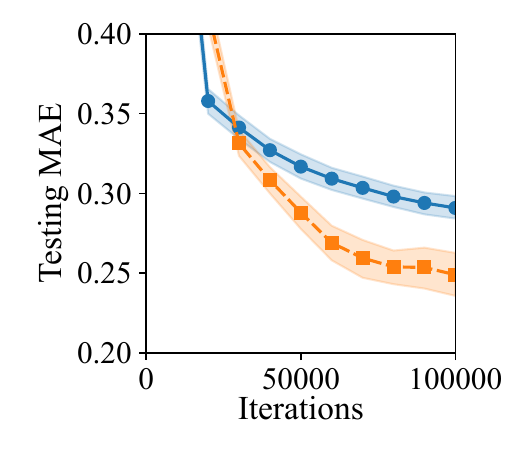} \label{fig:f3_shifted test sup}}\hfill
  \subfloat[Testing MAE on Scaled $f_3$]{\includegraphics[width=0.23\textwidth]{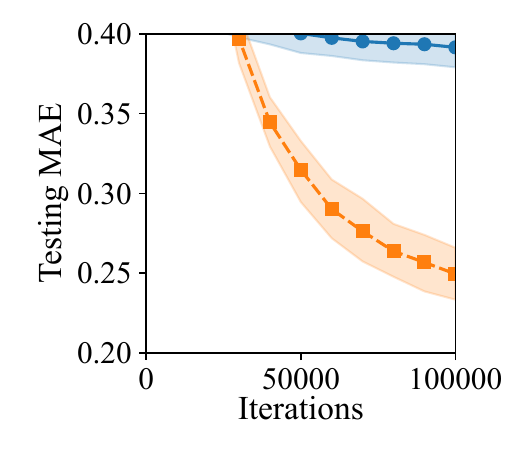} \label{fig:f3_scaled test sup}}\\
  \caption{Original, rotated, translated, and scaled $f_3$ settings. (a)--(d): 3D landscapes of the original function and its rotated, translated by $0.20$, and scaled by $1.25$ variants. (e)--(h): corresponding testing MAE curves.}
  \label{fig:rotated sup}
\end{figure}

\clearpage
\section{Scalability Analysis --- Detailed Results}
\label{sec: Scalability Analysis --- Detailed Results sup}

\subsection{Sample-Size Scaling}
\label{ss:Sample-Size Scaling sup}

Table~\ref{tb:result samscale sup} reports detailed results under varying sample sizes $N_S \in \{125, 250, 500, 1000, 2000, 4000, 8000\}$ with $n=10$ fixed. {Table~\ref{tb:statistics samscale sup} reports the corresponding run-level variability, paired Wilcoxon $p$-values, and effect sizes for each sample size and metric.}\label{r2-11-4}

Several patterns are evident. In training MAE, XCSF achieves significantly lower values than KACS at small sample sizes ($N_S \leq 250$), indicating that $n$-dimensional rules overfit when data are scarce. This gap narrows and reverses as $N_S$ increases: from $N_S=1000$ onward, KACS achieves significantly lower training MAE, suggesting that one-dimensional submodels become more effective as more data are available per subproblem. In testing MAE, KACS consistently achieves significantly lower values across all sample sizes ($p=0.0156$), confirming that the generalization advantage of KA-based decomposition holds regardless of data volume.

{Fig.~\ref{fig:sample scaling mae sup} provides an anytime view of testing MAE for all sample sizes. XCSF improves rapidly during the early iterations but subsequently changes only gradually, reaching an apparent plateau for most settings and showing increasing testing error at $N_S=250$. In contrast, KACS requires a longer warm-up but continues to improve and approaches a lower testing-MAE level. For $N_S\geq500$, the XCSF curves vary little during the latter half of training while a clear accuracy gap remains. At $N_S=125$, both curves remain relatively variable and KACS is still improving at the final iteration; nevertheless, XCSF has already stabilized at a higher error level.}\label{r2-17}

Model complexity shows a clear structural difference. XCSF rules and parameters grow monotonically with $N_S$ (3164 to 4724 rules; 34800 to 51960 parameters), whereas KACS remains compact throughout (1242--1430 rules; 2484--2860 parameters). The AIC advantage of KACS widens substantially as $N_S$ increases: at $N_S = 8000$, KACS achieves an AIC of $-13090$ compared with $90170$ for XCSF, reflecting both higher accuracy and dramatically fewer parameters at large data scales.

\begin{table*}[h]
  \centering
  \caption{Summary of Sample-Size Scalability Analysis. Green Indicates the Best Value. Rank Is the Average Rank Across Problems. Symbols $+/-/\sim$ Denote Significantly Better/Worse/Similar Performance Compared With KACS by the Wilcoxon Signed-Rank Test. Arrows $\uparrow/\downarrow$ Indicate Rank Improvement/Decline Relative to KACS. Statistical Significance Is at $\alpha=0.05$ ($\dag$)}
  \label{tb:result samscale sup}
  \renewcommand{\arraystretch}{1.2}
  \resizebox{\width}{!}{
    \small
    \begin{tabular}{c|cc|cc|cc|cc|cc}
  \bhline{1pt}
  &\multicolumn{4}{c|}{\textsc{Model Accuracy}}
  &\multicolumn{4}{c|}{\textsc{Model Complexity}}
  &\multicolumn{2}{c}{\textsc{Trade-Off}}\\[-0.2ex]\rule{0pt}{2.6ex}
  &\multicolumn{2}{c|}{\text{Training MAE}}
  &\multicolumn{2}{c|}{\text{Testing MAE}}
  &\multicolumn{2}{c|}{\text{\#Rules} $(|\mathcal{P}|)$}
  &\multicolumn{2}{c|}{\text{\#Parameters} $(k)$}
  &\multicolumn{2}{c}{\text{AIC}}\\
  $N_S$& XCSF & KACS & XCSF & KACS & XCSF & KACS & XCSF & KACS & XCSF & KACS \\
  \bhline{1pt}
  125 & \cellcolor{g}0.0124 $+$ & 0.2077 & 0.4156 $\sim$ & \cellcolor{g}0.3798 & 3164 $-$ & \cellcolor{g}1242 & 34802 $-$ & \cellcolor{g}2484 & 68860 $-$ & \cellcolor{g}4652 \\
  250 & \cellcolor{g}0.0893 $+$ & 0.207 & 0.3964 $-$ & \cellcolor{g}0.3308 & 3585 $-$ & \cellcolor{g}1333 & 39436 $-$ & \cellcolor{g}2666 & 78067 $-$ & \cellcolor{g}4722 \\
  500 & \cellcolor{g}0.1805 $\sim$ & 0.1903 & 0.3484 $-$ & \cellcolor{g}0.2246 & 3924 $-$ & \cellcolor{g}1353 & 43166 $-$ & \cellcolor{g}2707 & 85158 $-$ & \cellcolor{g}4137 \\
  1000 & 0.2510 $-$ & \cellcolor{g}0.2097 & 0.3246 $-$ & \cellcolor{g}0.2329 & 4130 $-$ & \cellcolor{g}1429 & 45426 $-$ & \cellcolor{g}2857 & 88892 $-$ & \cellcolor{g}3324 \\
  2000 & 0.2824 $-$ & \cellcolor{g}0.2080 & 0.3095 $-$ & \cellcolor{g}0.2177 & 4327 $-$ & \cellcolor{g}1441 & 47597 $-$ & \cellcolor{g}2882 & 91538 $-$ & \cellcolor{g}886.6 \\
  4000 & 0.3027 $-$ & \cellcolor{g}0.2098 & 0.3129 $-$ & \cellcolor{g}0.2131 & 4521 $-$ & \cellcolor{g}1406 & 49731 $-$ & \cellcolor{g}2812 & 92463 $-$ & \cellcolor{g}-3964 \\
  8000 & 0.3075 $-$ & \cellcolor{g}0.2146 & 0.3117 $-$ & \cellcolor{g}0.2159 & 4724 $-$ & \cellcolor{g}1430 & 51964 $-$ & \cellcolor{g}2860 & 90165 $-$ & \cellcolor{g}-13089 \\

  \bhline{1pt}
  Rank & \textit{1.57}$\downarrow$ & \cellcolor{g}\textit{1.43} & \textit{2.00}$\downarrow^{\dag}$ & \cellcolor{g}\textit{1.00} & \textit{2.00}$\downarrow^{\dag}$ & \cellcolor{g}\textit{1.00} & \textit{2.00}$\downarrow^{\dag}$ & \cellcolor{g}\textit{1.00} & \textit{2.00}$\downarrow^{\dag}$ & \cellcolor{g}\textit{1.00} \\
  $+/-/\sim$ & 2/4/1 & - & 0/6/1 & - & 0/7/0 & - & 0/7/0 & - & 0/7/0 & - \\
  $p$-value & 1 & - & 0.0156 & - & 0.0156 & - & 0.0156 & - & 0.0156 & - \\
  \bhline{1pt}
\end{tabular}

  }
\end{table*}

\begin{table*}[h]
  \centering
  \caption{{Detailed Paired Comparison of XCSF and KACS over 30 Matched
  Runs for Each Sample Size and Metric. Notation Follows Table~\ref{tb:supplementary-statistics sup}}}
  \label{tb:statistics samscale sup}
  \renewcommand{\arraystretch}{1.2}
  \resizebox{\width}{!}{
    \small
    \begin{tabular}{cccccc}
  \bhline{1pt}
  $N_S$ & Metric & XCSF mean $\pm$ SD & KACS mean $\pm$ SD & Wilcoxon $p$ & Effect size \\
  \bhline{1pt}
  125 & Training MAE & \cellcolor{g}0.0124 $\pm$ 0.0080 & 0.2077 $\pm$ 0.0775 & 1.86E-09 & -1 \\
  125 & Testing MAE & 0.4156 $\pm$ 0.0814 & \cellcolor{g}0.3798 $\pm$ 0.0953 & 0.1241 & 0.3247 \\
  125 & \#Rules & 3164 $\pm$ 82.06 & \cellcolor{g}1242 $\pm$ 45.36 & 1.82E-06 & 1 \\
  125 & \#Parameters & 34802 $\pm$ 902.7 & \cellcolor{g}2484 $\pm$ 90.71 & 1.86E-09 & 1 \\
  125 & AIC & 68860 $\pm$ 1857 & \cellcolor{g}4652 $\pm$ 176.2 & 1.86E-09 & 1 \\
  \bhline{1pt}
  250 & Training MAE & \cellcolor{g}0.0893 $\pm$ 0.0297 & 0.2070 $\pm$ 0.0464 & 1.86E-09 & -1 \\
  250 & Testing MAE & 0.3964 $\pm$ 0.0717 & \cellcolor{g}0.3308 $\pm$ 0.0760 & 0.001232 & 0.6516 \\
  250 & \#Rules & 3585 $\pm$ 69.45 & \cellcolor{g}1333 $\pm$ 42.34 & 1.82E-06 & 1 \\
  250 & \#Parameters & 39436 $\pm$ 763.9 & \cellcolor{g}2666 $\pm$ 84.68 & 1.86E-09 & 1 \\
  250 & AIC & 78067 $\pm$ 1593 & \cellcolor{g}4722 $\pm$ 196.1 & 1.86E-09 & 1 \\
  \bhline{1pt}
  500 & Training MAE & \cellcolor{g}0.1805 $\pm$ 0.0265 & 0.1903 $\pm$ 0.0378 & 0.1772 & -0.286 \\
  500 & Testing MAE & 0.3484 $\pm$ 0.0471 & \cellcolor{g}0.2246 $\pm$ 0.0515 & 1.86E-09 & 1 \\
  500 & \#Rules & 3924 $\pm$ 63.14 & \cellcolor{g}1353 $\pm$ 39.37 & 1.82E-06 & 1 \\
  500 & \#Parameters & 43166 $\pm$ 694.5 & \cellcolor{g}2707 $\pm$ 78.74 & 1.86E-09 & 1 \\
  500 & AIC & 85158 $\pm$ 1387 & \cellcolor{g}4137 $\pm$ 213.1 & 1.86E-09 & 1 \\
  \bhline{1pt}
  1000 & Training MAE & 0.2510 $\pm$ 0.0167 & \cellcolor{g}0.2097 $\pm$ 0.0395 & 4.42E-06 & 0.8667 \\
  1000 & Testing MAE & 0.3246 $\pm$ 0.0250 & \cellcolor{g}0.2329 $\pm$ 0.0508 & 3.73E-09 & 0.9957 \\
  1000 & \#Rules & 4130 $\pm$ 74.08 & \cellcolor{g}1429 $\pm$ 36.94 & 1.82E-06 & 1 \\
  1000 & \#Parameters & 45426 $\pm$ 814.9 & \cellcolor{g}2857 $\pm$ 73.87 & 1.86E-09 & 1 \\
  1000 & AIC & 88892 $\pm$ 1580 & \cellcolor{g}3324 $\pm$ 349.7 & 1.86E-09 & 1 \\
  \bhline{1pt}
  2000 & Training MAE & 0.2824 $\pm$ 0.0086 & \cellcolor{g}0.2080 $\pm$ 0.0609 & 1.22E-05 & 0.8366 \\
  2000 & Testing MAE & 0.3095 $\pm$ 0.0160 & \cellcolor{g}0.2177 $\pm$ 0.0663 & 3.79E-06 & 0.871 \\
  2000 & \#Rules & 4327 $\pm$ 69.43 & \cellcolor{g}1441 $\pm$ 35.93 & 1.82E-06 & 1 \\
  2000 & \#Parameters & 47597 $\pm$ 763.8 & \cellcolor{g}2882 $\pm$ 71.85 & 1.86E-09 & 1 \\
  2000 & AIC & 91538 $\pm$ 1521 & \cellcolor{g}886.6 $\pm$ 817.7 & 1.86E-09 & 1 \\
  \bhline{1pt}
  4000 & Training MAE & 0.3027 $\pm$ 0.0070 & \cellcolor{g}0.2098 $\pm$ 0.0483 & 1.86E-08 & 0.9785 \\
  4000 & Testing MAE & 0.3129 $\pm$ 0.0147 & \cellcolor{g}0.2131 $\pm$ 0.0469 & 3.73E-09 & 0.9957 \\
  4000 & \#Rules & 4521 $\pm$ 50.57 & \cellcolor{g}1406 $\pm$ 30.42 & 1.82E-06 & 1 \\
  4000 & \#Parameters & 49731 $\pm$ 556.3 & \cellcolor{g}2812 $\pm$ 60.85 & 1.86E-09 & 1 \\
  4000 & AIC & 92463 $\pm$ 1033 & \cellcolor{g}-3964 $\pm$ 1445 & 1.86E-09 & 1 \\
  \bhline{1pt}
  8000 & Training MAE & 0.3075 $\pm$ 0.0059 & \cellcolor{g}0.2146 $\pm$ 0.0468 & 1.30E-08 & 0.9828 \\
  8000 & Testing MAE & 0.3117 $\pm$ 0.0099 & \cellcolor{g}0.2159 $\pm$ 0.0491 & 1.86E-08 & 0.9785 \\
  8000 & \#Rules & 4724 $\pm$ 57.72 & \cellcolor{g}1430 $\pm$ 32.93 & 1.82E-06 & 1 \\
  8000 & \#Parameters & 51964 $\pm$ 634.9 & \cellcolor{g}2860 $\pm$ 65.86 & 1.86E-09 & 1 \\
  8000 & AIC & 90165 $\pm$ 1255 & \cellcolor{g}-13089 $\pm$ 2853 & 1.86E-09 & 1 \\
  \bhline{1pt}
\end{tabular}

  }
\end{table*}

\begin{figure*}[h]
  \centering
  \subfloat[$N_S=125$]{\includegraphics[width=0.24\textwidth]{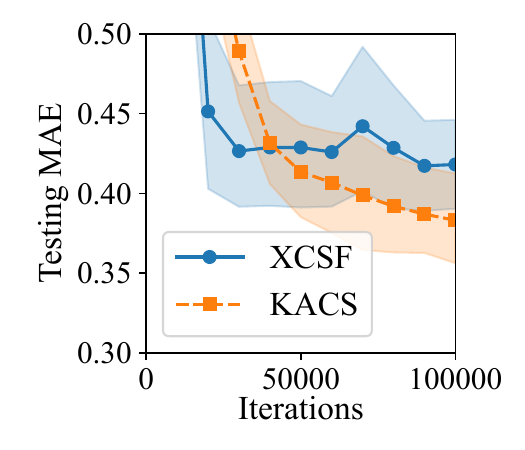} \label{fig:sample size 125 sup}}
  \subfloat[$N_S=250$]{\includegraphics[width=0.24\textwidth]{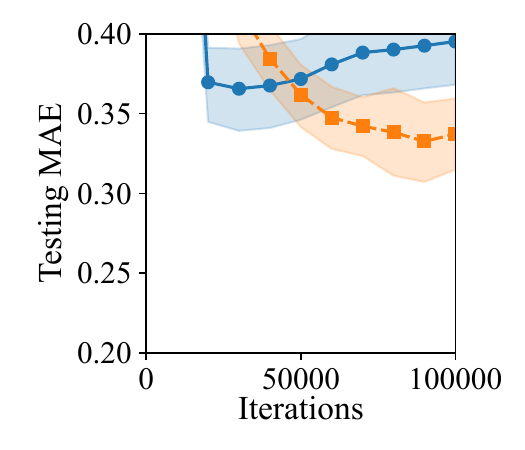} \label{fig:sample size 250 sup}}
  \subfloat[$N_S=500$]{\includegraphics[width=0.24\textwidth]{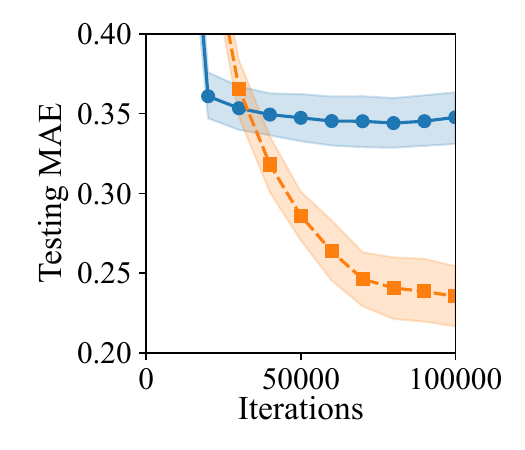} \label{fig:sample size 500 sup}}
  \subfloat[$N_S=1000$]{
    \includegraphics[width=0.24\textwidth]{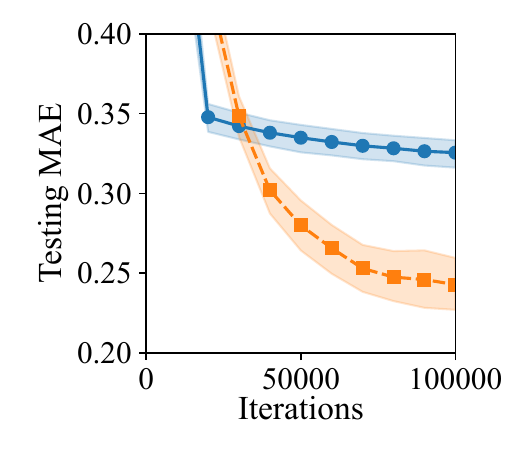}
    \label{fig:sample size 1000 sup}
  }\\
  \subfloat[$N_S=2000$]{\includegraphics[width=0.24\textwidth]{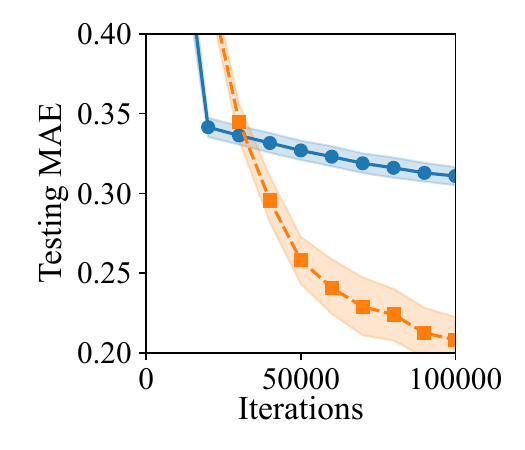} \label{fig:sample size 2000 sup}}
  \subfloat[$N_S=4000$]{\includegraphics[width=0.24\textwidth]{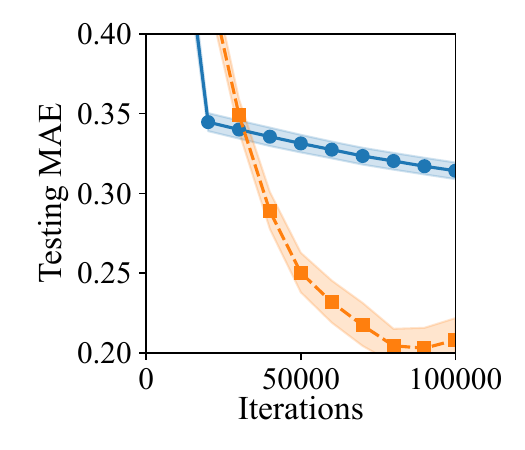} \label{fig:sample size 4000 sup}}
  \subfloat[$N_S=8000$]{\includegraphics[width=0.24\textwidth]{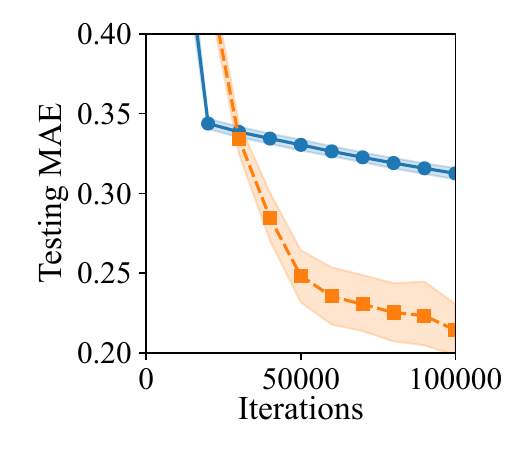} \label{fig:sample size 8000 sup}}
  \caption{Testing MAE learning curves of XCSF and KACS on $f_3$
  under varying sample sizes, with the input dimension fixed at $n=10$.}
  \label{fig:sample scaling mae sup}
\end{figure*}
\clearpage

\subsection{Dimensional Scaling}
\label{ss:Dimensional Scaling sup}
Table~\ref{tb:result dimscale sup} reports detailed results as the input dimension $n$ varies from 2 to 20 with $N_S = 1000$ fixed. {Table~\ref{tb:statistics dimscale sup} reports the corresponding run-level variability, paired Wilcoxon $p$-values, and effect sizes for each dimension and metric.}\label{r2-11-5}

In the practical regime ($n \leq 16$), KACS achieves significantly lower testing MAE than XCSF at all dimensions. The accuracy gap is largest at $n=2$ (0.0883 vs. 0.1582) and remains consistent through $n=14$ (0.2021 vs. 0.3066). Notably, XCSF testing MAE does not grow monotonically with $n$: it peaks around $n=6$ (0.3292) and slightly decreases at higher dimensions, while KACS testing MAE remains stable between 0.20 and 0.25 up to $n=14$.

At $n \geq 18$ under the default budget $N=6400$, KACS degrades sharply: testing MAE rises to 0.3842 at $n=18$ and 1.2500 at $n=20$, at which point XCSF achieves significantly better accuracy. As discussed in Section \ref{sss:scalability_discussion} of the main article, this failure is caused by the cover-delete cycle that occurs when $N$ is too small relative to the $(2n^2+3n+1)$ subproblems. Using the scaled budgets $N=15984$ and $N=19680$ for $n=18$ and $n=20$, respectively, KACS recovers and again achieves significantly lower testing MAE than XCSF (0.2115 vs. 0.2768 at $n=18$; 0.1970 vs. 0.2833 at $n=20$), confirming that the degradation is a capacity issue rather than an inherent limitation of KACS.

Model complexity confirms the structural efficiency of KACS throughout. XCSF parameters grow monotonically from 4504 ($n=2$) to 110,900 ($n=20$), whereas KACS parameters remain between 2372 and 3909 across all dimensions, including the high-dimensional failure cases. The AIC advantage of KACS holds for $n \leq 16$ (e.g., 2120 vs. 138,200 at $n=14$) and reverses only partially at $n \geq 18$ under $N=6400$, where KACS accuracy collapses despite its compact parameter count. At $n=2$, KACS rule count (1955) slightly exceeds that of XCSF (1501), the only exception across all conditions, reflecting the fixed submodel structure of KACS at low dimensions.

\begin{table*}[h]
  \centering
  \caption{Summary of Dimensionality Scalability Analysis. Notation Follows Table \ref{tb:result samscale sup}}
  \label{tb:result dimscale sup}
  \renewcommand{\arraystretch}{1.2}
  \resizebox{\width}{!}{
    \small
    \begin{tabular}{c|cc|cc|cc|cc|cc}
  \bhline{1pt}
  &\multicolumn{4}{c|}{\textsc{Model Accuracy}}
  &\multicolumn{4}{c|}{\textsc{Model Complexity}}
  &\multicolumn{2}{c}{\textsc{Trade-Off}}\\[-0.2ex]\rule{0pt}{2.6ex}
  &\multicolumn{2}{c|}{\text{Training MAE}}
  &\multicolumn{2}{c|}{\text{Testing MAE}}
  &\multicolumn{2}{c|}{\text{\#Rules} $(|\mathcal{P}|)$}
  &\multicolumn{2}{c|}{\text{\#Parameters} $(k)$}
  &\multicolumn{2}{c}{\text{AIC}}\\
  $n$ & XCSF & KACS & XCSF & KACS & XCSF & KACS & XCSF & KACS & XCSF & KACS \\
  \bhline{1pt}
  2 & 0.1459 $-$ & \cellcolor{g}0.0831 & 0.1582 $-$ & \cellcolor{g}0.0883 & \cellcolor{g}1501 $+$ & 1955 & 4504 $-$ & \cellcolor{g}3909 & 6251 $-$ & \cellcolor{g}3792 \\
  4 & 0.2210 $\sim$ & \cellcolor{g}0.2145 & 0.2656 $-$ & \cellcolor{g}0.2345 & 2829 $-$ & \cellcolor{g}1862 & 14145 $-$ & \cellcolor{g}3725 & 26047 $-$ & \cellcolor{g}5125 \\
  6 & 0.2632 $-$ & \cellcolor{g}0.2174 & 0.3292 $-$ & \cellcolor{g}0.2417 & 3423 $-$ & \cellcolor{g}1723 & 23960 $-$ & \cellcolor{g}3447 & 45884 $-$ & \cellcolor{g}4575 \\
  8 & 0.2525 $-$ & \cellcolor{g}0.2076 & 0.3269 $-$ & \cellcolor{g}0.2208 & 3837 $-$ & \cellcolor{g}1540 & 34529 $-$ & \cellcolor{g}3079 & 67024 $-$ & \cellcolor{g}3765 \\
  10 & 0.2510 $-$ & \cellcolor{g}0.2097 & 0.3246 $-$ & \cellcolor{g}0.2329 & 4130 $-$ & \cellcolor{g}1429 & 45426 $-$ & \cellcolor{g}2857 & 88892 $-$ & \cellcolor{g}3324 \\
  12 & 0.2355 $-$ & \cellcolor{g}0.2117 & 0.3224 $-$ & \cellcolor{g}0.2252 & 4439 $-$ & \cellcolor{g}1311 & 57703 $-$ & \cellcolor{g}2623 & 113331 $-$ & \cellcolor{g}2787 \\
  14 & 0.2398 $-$ & \cellcolor{g}0.1870 & 0.3066 $-$ & \cellcolor{g}0.2021 & 4674 $-$ & \cellcolor{g}1186 & 70114 $-$ & \cellcolor{g}2372 & 138194 $-$ & \cellcolor{g}2120 \\
  16 & \cellcolor{g}0.2234 $\sim$ & 0.2359 & 0.2770 $-$ & \cellcolor{g}0.2513 & 4890 $-$ & \cellcolor{g}1209 & 83137 $-$ & \cellcolor{g}2418 & 164127 $-$ & \cellcolor{g}2713 \\
  18 & \cellcolor{g}0.2388 $+$ & 0.3712 & \cellcolor{g}0.2749 $+$ & 0.3842 & 5067 $-$ & \cellcolor{g}1400 & 96268 $-$ & \cellcolor{g}2801 & 190490 $-$ & \cellcolor{g}4335 \\
  20 & \cellcolor{g}0.2395 $+$ & 1.2015 & \cellcolor{g}0.2706 $+$ & 1.2497 & 5283 $-$ & \cellcolor{g}1862 & 110941 $-$ & \cellcolor{g}3725 & 219841 $-$ & \cellcolor{g}8259 \\
  \bhline{1pt}
  Rank & \textit{1.70}$\downarrow$ & \cellcolor{g}\textit{1.30} & \textit{1.80}$\downarrow$ & \cellcolor{g}\textit{1.20} & \textit{1.90}$\downarrow^{\dag}$ & \cellcolor{g}\textit{1.10} & \textit{2.00}$\downarrow^{\dag}$ & \cellcolor{g}\textit{1.00} & \textit{2.00}$\downarrow^{\dag}$ & \cellcolor{g}\textit{1.00} \\
  $+/-/\sim$ & 2/6/2 & - & 2/8/0 & - & 1/9/0 & - & 0/10/0 & - & 0/10/0 & - \\
  $p$-value & 0.557 & - & 0.432 & - & 0.00391 & - & 0.00195 & - & 0.00195 & - \\
  \bhline{1pt}
\end{tabular}%

  }
\end{table*}

\begin{table*}[h]
  \centering
  \caption{{Detailed Paired Comparison of XCSF and KACS over 30 Matched
  Runs for Each Dimension and Metric. Notation Follows Table~\ref{tb:supplementary-statistics sup}}}
  \label{tb:statistics dimscale sup}
  \renewcommand{\arraystretch}{1.2}
  \resizebox{\width}{!}{
    \small
    \begin{tabular}{cccccc}
  \bhline{1pt}
$n$ & Metric & XCSF mean $\pm$ SD & KACS mean $\pm$ SD & Wilcoxon $p$ & Effect size \\
\bhline{1pt}
2 & Training MAE & 0.1459 $\pm$ 0.0528 & \cellcolor{g}0.0831 $\pm$ 0.0203 & 8.33E-07 & 0.9097 \\
2 & Testing MAE & 0.1582 $\pm$ 0.0637 & \cellcolor{g}0.0883 $\pm$ 0.0223 & 6.92E-06 & 0.8538 \\
2 & \#Rules & \cellcolor{g}1501 $\pm$ 220.6 & 1955 $\pm$ 96.36 & 1.82E-06 & -1 \\
2 & \#Parameters & 4504 $\pm$ 661.9 & \cellcolor{g}3909 $\pm$ 192.7 & 4.41E-05 & 0.7935 \\
2 & AIC & 6251 $\pm$ 995.4 & \cellcolor{g}3792 $\pm$ 497.9 & 9.31E-09 & 0.9871 \\
\bhline{1pt}
4 & Training MAE & 0.2210 $\pm$ 0.0209 & \cellcolor{g}0.2145 $\pm$ 0.0226 & 0.4045 & 0.1785 \\
4 & Testing MAE & 0.2656 $\pm$ 0.0310 & \cellcolor{g}0.2345 $\pm$ 0.0301 & 0.0001529 & 0.7462 \\
4 & \#Rules & 2829 $\pm$ 84.79 & \cellcolor{g}1862 $\pm$ 42.24 & 1.82E-06 & 1 \\
4 & \#Parameters & 14145 $\pm$ 423.9 & \cellcolor{g}3725 $\pm$ 84.48 & 1.86E-09 & 1 \\
4 & AIC & 26047 $\pm$ 767.7 & \cellcolor{g}5125 $\pm$ 222.9 & 1.86E-09 & 1 \\
\bhline{1pt}
6 & Training MAE & 0.2632 $\pm$ 0.0194 & \cellcolor{g}0.2174 $\pm$ 0.0288 & 1.64E-07 & 0.9441 \\
6 & Testing MAE & 0.3292 $\pm$ 0.0228 & \cellcolor{g}0.2417 $\pm$ 0.0349 & 1.86E-09 & 1 \\
6 & \#Rules & 3423 $\pm$ 78.47 & \cellcolor{g}1723 $\pm$ 34.24 & 1.82E-06 & 1 \\
6 & \#Parameters & 23960 $\pm$ 549.3 & \cellcolor{g}3447 $\pm$ 68.48 & 1.86E-09 & 1 \\
6 & AIC & 45884 $\pm$ 1045 & \cellcolor{g}4575 $\pm$ 282.3 & 1.86E-09 & 1 \\
\bhline{1pt}
8 & Training MAE & 0.2525 $\pm$ 0.0169 & \cellcolor{g}0.2076 $\pm$ 0.0372 & 9.22E-06 & 0.8452 \\
8 & Testing MAE & 0.3269 $\pm$ 0.0242 & \cellcolor{g}0.2208 $\pm$ 0.0359 & 1.86E-09 & 1 \\
8 & \#Rules & 3837 $\pm$ 73.09 & \cellcolor{g}1540 $\pm$ 41.68 & 1.86E-09 & 1 \\
8 & \#Parameters & 34529 $\pm$ 657.8 & \cellcolor{g}3079 $\pm$ 83.37 & 1.86E-09 & 1 \\
8 & AIC & 67024 $\pm$ 1286 & \cellcolor{g}3765 $\pm$ 300.7 & 1.86E-09 & 1 \\
\bhline{1pt}
10 & Training MAE & 0.2510 $\pm$ 0.0167 & \cellcolor{g}0.2097 $\pm$ 0.0395 & 4.42E-06 & 0.8667 \\
10 & Testing MAE & 0.3246 $\pm$ 0.0250 & \cellcolor{g}0.2329 $\pm$ 0.0508 & 3.73E-09 & 0.9957 \\
10 & \#Rules & 4130 $\pm$ 74.08 & \cellcolor{g}1429 $\pm$ 36.94 & 1.82E-06 & 1 \\
10 & \#Parameters & 45426 $\pm$ 814.9 & \cellcolor{g}2857 $\pm$ 73.87 & 1.86E-09 & 1 \\
10 & AIC & 88892 $\pm$ 1580 & \cellcolor{g}3324 $\pm$ 349.7 & 1.86E-09 & 1 \\
\bhline{1pt}
12 & Training MAE & 0.2355 $\pm$ 0.0185 & \cellcolor{g}0.2117 $\pm$ 0.0928 & 0.01454 & 0.5054 \\
12 & Testing MAE & 0.3224 $\pm$ 0.0222 & \cellcolor{g}0.2252 $\pm$ 0.0927 & 4.42E-06 & 0.8667 \\
12 & \#Rules & 4439 $\pm$ 44.61 & \cellcolor{g}1311 $\pm$ 37.30 & 1.86E-09 & 1 \\
12 & \#Parameters & 57703 $\pm$ 580.0 & \cellcolor{g}2623 $\pm$ 74.60 & 1.86E-09 & 1 \\
12 & AIC & 113331 $\pm$ 1156 & \cellcolor{g}2787 $\pm$ 593.1 & 1.86E-09 & 1 \\
\bhline{1pt}
14 & Training MAE & 0.2398 $\pm$ 0.0175 & \cellcolor{g}0.1870 $\pm$ 0.0530 & 3.24E-06 & 0.8753 \\
14 & Testing MAE & 0.3066 $\pm$ 0.0206 & \cellcolor{g}0.2021 $\pm$ 0.0595 & 1.86E-08 & 0.9785 \\
14 & \#Rules & 4674 $\pm$ 48.48 & \cellcolor{g}1186 $\pm$ 33.29 & 1.86E-09 & 1 \\
14 & \#Parameters & 70114 $\pm$ 727.2 & \cellcolor{g}2372 $\pm$ 66.59 & 1.86E-09 & 1 \\
14 & AIC & 138194 $\pm$ 1437 & \cellcolor{g}2120 $\pm$ 551.8 & 1.86E-09 & 1 \\
\bhline{1pt}
16 & Training MAE & \cellcolor{g}0.2234 $\pm$ 0.0149 & 0.2359 $\pm$ 0.0365 & 0.0961 & -0.3505 \\
16 & Testing MAE & 0.2770 $\pm$ 0.0256 & \cellcolor{g}0.2513 $\pm$ 0.0422 & 0.01454 & 0.5054 \\
16 & \#Rules & 4890 $\pm$ 54.43 & \cellcolor{g}1209 $\pm$ 31.91 & 1.86E-09 & 1 \\
16 & \#Parameters & 83137 $\pm$ 925.3 & \cellcolor{g}2418 $\pm$ 63.83 & 1.86E-09 & 1 \\
16 & AIC & 164127 $\pm$ 1828 & \cellcolor{g}2713 $\pm$ 277.5 & 1.86E-09 & 1 \\
\bhline{1pt}
18 & Training MAE & \cellcolor{g}0.2388 $\pm$ 0.0124 & 0.3712 $\pm$ 0.0537 & 1.86E-09 & -1 \\
18 & Testing MAE & \cellcolor{g}0.2749 $\pm$ 0.0219 & 0.3842 $\pm$ 0.0643 & 3.73E-09 & -0.9957 \\
18 & \#Rules & 5067 $\pm$ 51.70 & \cellcolor{g}1400 $\pm$ 42.87 & 1.82E-06 & 1 \\
18 & \#Parameters & 96268 $\pm$ 982.3 & \cellcolor{g}2801 $\pm$ 85.74 & 1.82E-06 & 1 \\
18 & AIC & 190490 $\pm$ 1946 & \cellcolor{g}4335 $\pm$ 380.9 & 1.86E-09 & 1 \\
\bhline{1pt}
20 & Training MAE & \cellcolor{g}0.2395 $\pm$ 0.0156 & 1.2015 $\pm$ 0.1632 & 1.86E-09 & -1 \\
20 & Testing MAE & \cellcolor{g}0.2706 $\pm$ 0.0322 & 1.2497 $\pm$ 0.2080 & 1.86E-09 & -1 \\
20 & \#Rules & 5283 $\pm$ 78.45 & \cellcolor{g}1862 $\pm$ 37.57 & 1.82E-06 & 1 \\
20 & \#Parameters & 110941 $\pm$ 1647 & \cellcolor{g}3725 $\pm$ 75.14 & 1.86E-09 & 1 \\
20 & AIC & 219841 $\pm$ 3369 & \cellcolor{g}8259 $\pm$ 322.6 & 1.86E-09 & 1 \\
\bhline{1pt}
\end{tabular}

  }
\end{table*}

\clearpage

\section{{Covering Dynamics and Cover-Delete Behavior}}
\label{sec:covering_dynamics_sup}
\label{r2-12-2}

{To examine whether XCSF and KACS repeatedly lose input coverage during training, we recorded the number of match sets that invoked covering. For a recording interval of $L$ iterations, let $C_{\mathrm{X}}$, $C_{\mathrm{K}}$, $C_{\mathrm{in}}$, and $C_{\mathrm{out}}$ denote the numbers of covering events in XCSF, all KACS submodels, the KACS inner submodels, and the KACS outer submodels, respectively. We calculate
  \begin{align}
    r_{\mathrm{XCSF}} &= \frac{C_{\mathrm{X}}}{L}, &
    r_{\mathrm{KACS}} &=
    \frac{C_{\mathrm{K}}}{L(2n^2+3n+1)}, \\
    r_{\mathrm{inner}} &=
    \frac{C_{\mathrm{in}}}{Ln(2n+1)}, &
    r_{\mathrm{outer}} &=
    \frac{C_{\mathrm{out}}}{L(2n+1)}.
    \label{eq:covering_rates_sup}
  \end{align}
  Thus, each rate is the fraction of match-set evaluations that require covering, which accounts for the different numbers of match sets evaluated by XCSF and KACS.}

{Fig.~\ref{fig:all_covering sup} shows that both methods require frequent covering during initialization. Thereafter, the XCSF covering rate rapidly approaches zero and remains negligible on all eight problems. Thus, under the standard experimental settings, XCSF does not regularly exhibit recurring coverage gaps that would indicate a cover-delete cycle. The overall KACS covering rate also decreases sharply, but remains nonzero throughout training. This persistence is most pronounced on the four synthetic problems and CS, while it is smaller on ASN, CCPP, and EEC.

  The near-zero XCSF covering rate is consistent with its high rule density within a single population. As reported in Table~\ref{tb:result} of the main article, XCSF retains 1953--4170 macro-rules, all of which are tested when constructing its single match set. In contrast, KACS distributes 1012--1924 macro-rules among $2n^2+3n+1$ submodels. For example, the $n=10$ synthetic problems have 231 submodels and 1341--1462 KACS macro-rules, corresponding to only about 5.8--6.3 macro-rules per submodel on average. This substantially lower per-submodel rule density makes a KACS match set more likely to become empty. This interpretation concerns antecedent coverage and rule allocation rather than the total number of consequent parameters, because covering is triggered by an empty match set.}

\begin{figure*}[h]
  \centering
  \subfloat[$f_1$: Rastrigin Function]{
    \includegraphics[width=0.24\textwidth]{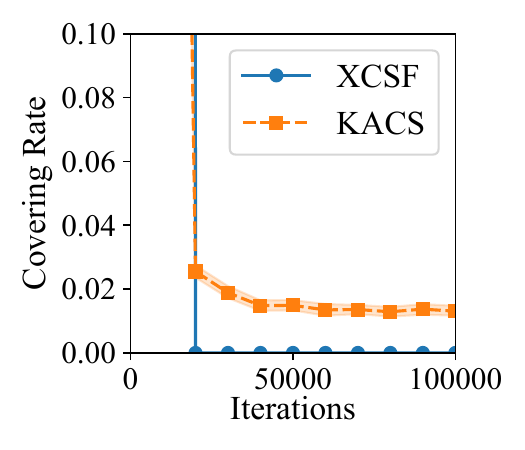}
    \label{fig:f1_covering sup}
  }
  \subfloat[$f_2$: Rosenbrock Function]{\includegraphics[width=0.24\textwidth]{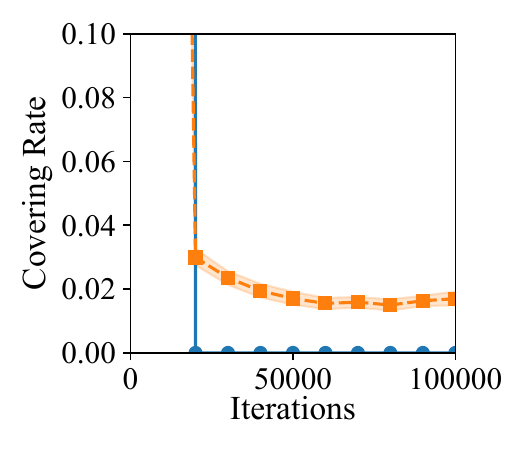} \label{fig:f2_covering sup}}
  \subfloat[$f_3$: Cross Function]{\includegraphics[width=0.24\textwidth]{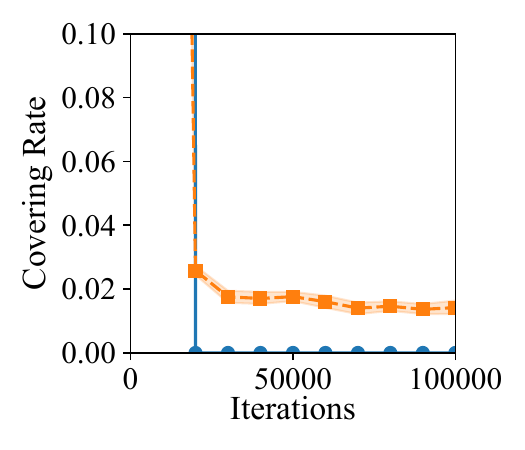} \label{fig:f3_covering sup}}
  \subfloat[$f_4$: Styblinski-Tang Function]{\includegraphics[width=0.24\textwidth]{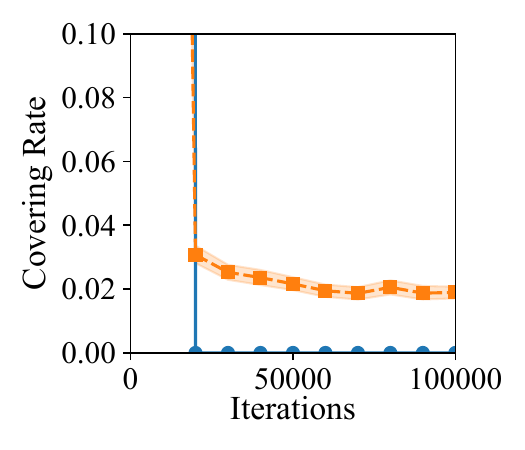} \label{fig:f4_covering sup}}
  \\
  \subfloat[ASN]{
    \includegraphics[width=0.24\textwidth]{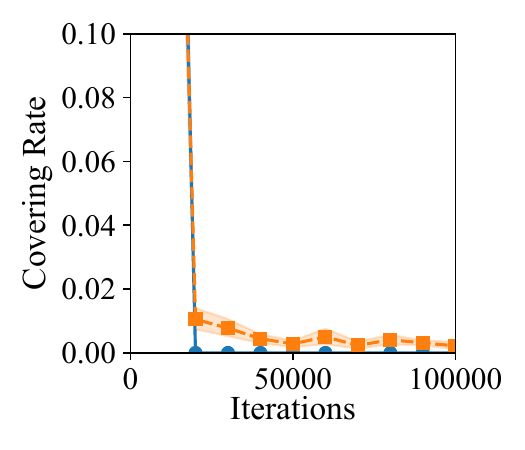}
    \label{fig:asn_covering sup}
  }
  \subfloat[CCPP]{\includegraphics[width=0.24\textwidth]{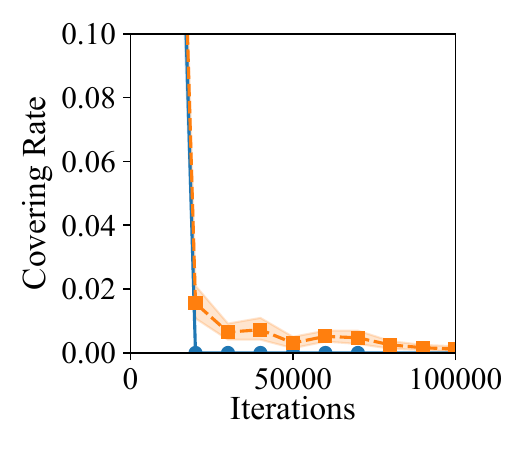} \label{fig:ccpp_covering sup}}
  \subfloat[CS]{\includegraphics[width=0.24\textwidth]{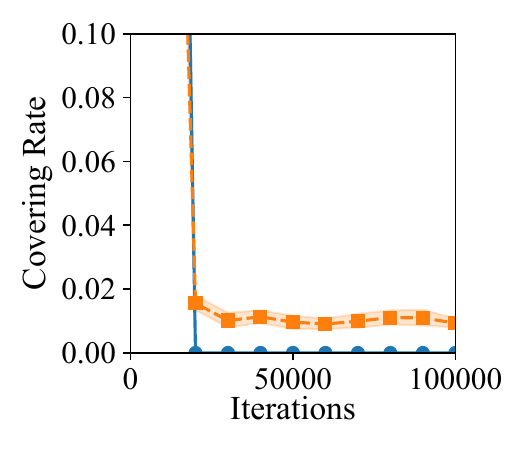} \label{fig:cs_covering sup}}
  \subfloat[EEC]{\includegraphics[width=0.24\textwidth]{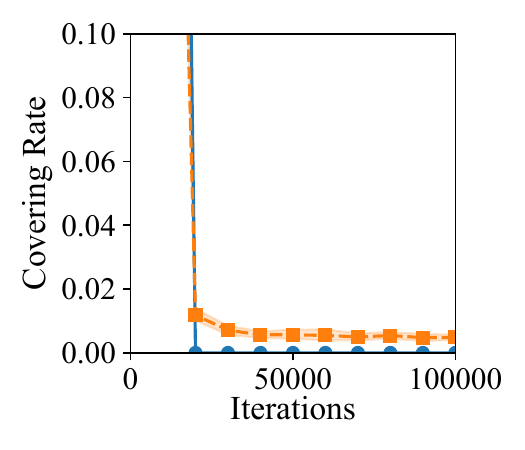} \label{fig:eec_covering sup}}
  \caption{Covering-rate trajectories of XCSF and KACS during training
    on the eight benchmark problems. Rates are normalized by the number
    of match-set evaluations. Curves show the mean over 30 runs, and
  shaded regions denote 95\% confidence intervals.}
  \label{fig:all_covering sup}
\end{figure*}

{Fig.~\ref{fig:all_covering kacs sup} localizes this behavior within the KACS architecture. The covering rate of the inner submodels approaches zero after initialization on every problem, following a pattern similar to XCSF. In contrast, the outer submodels continue to invoke covering throughout training, although the magnitude is problem dependent. This difference is consistent with the distinct input domains of the two submodel types: inner submodels always receive inputs in $[0,1]$, whereas the outer inputs $\hat{z}_q$ change as the inner functions are learned and can span a substantially wider range. Consequently, recurrent coverage gaps are concentrated in the outer submodels rather than being a general property of all KACS rulesets.

  Persistent late-stage covering indicates that stable outer-submodel coverage is not always maintained and is therefore consistent with recurrent cover-delete behavior. Moreover, this behavior does not cause a general accuracy failure under the standard settings, where KACS remains competitive despite the nonzero outer covering rate. On the other hand, persistent outer-submodel covering becomes harmful when the population budget is too small to maintain adequate coverage, as observed in the high-dimensional experiments analyzed in Section~\ref{sss:scalability_discussion} of the main article. In summary, XCSF and the KACS inner submodels maintain stable coverage after initialization, whereas the dynamically changing outer domains make the KACS outer submodels more susceptible to recurrent covering.}

\begin{figure*}[h]
  \centering
  \subfloat[$f_1$: Rastrigin Function]{
    \includegraphics[width=0.24\textwidth]{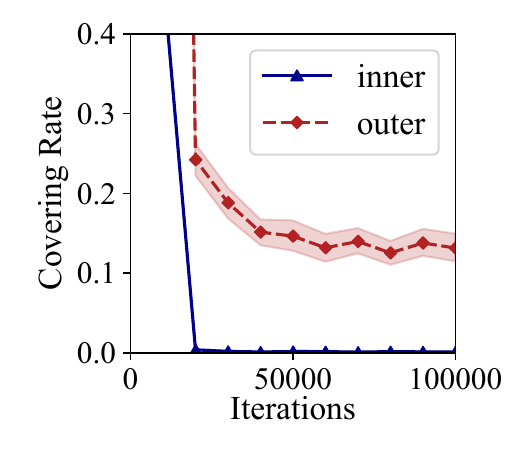}
    \label{fig:f1_covering kacs sup}
  }
  \subfloat[$f_2$: Rosenbrock Function]{\includegraphics[width=0.24\textwidth]{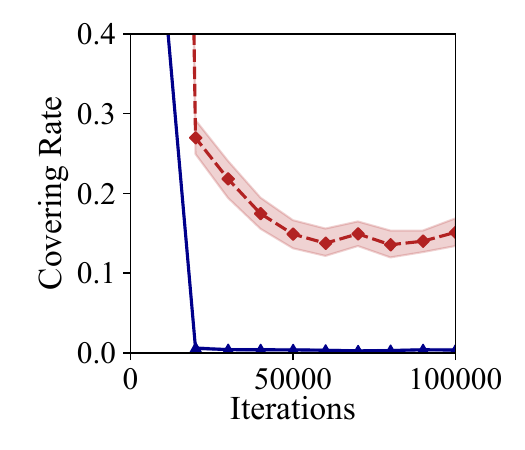} \label{fig:f2_covering kacs sup}}
  \subfloat[$f_3$: Cross Function]{\includegraphics[width=0.24\textwidth]{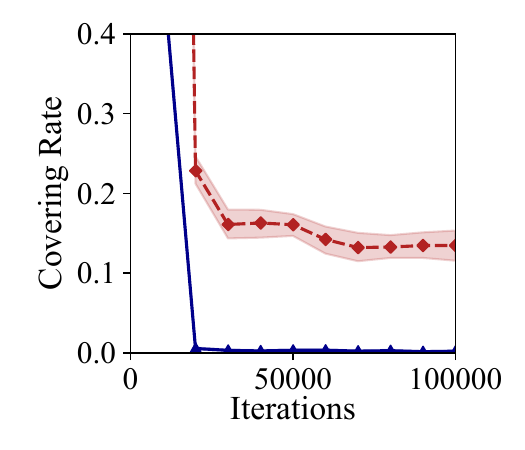} \label{fig:f3_covering kacs sup}}
  \subfloat[$f_4$: Styblinski-Tang Function]{\includegraphics[width=0.24\textwidth]{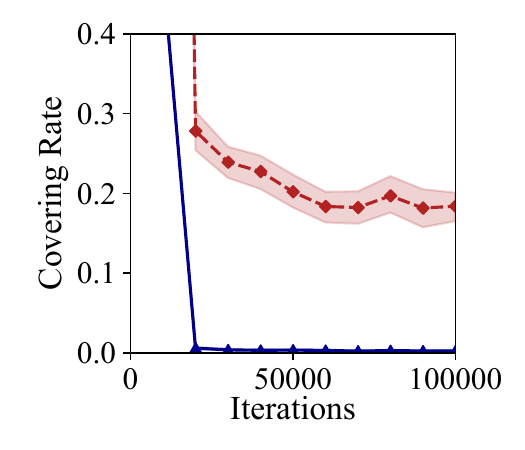} \label{fig:f4_covering kacs sup}}
  \\
  \subfloat[ASN]{
    \includegraphics[width=0.24\textwidth]{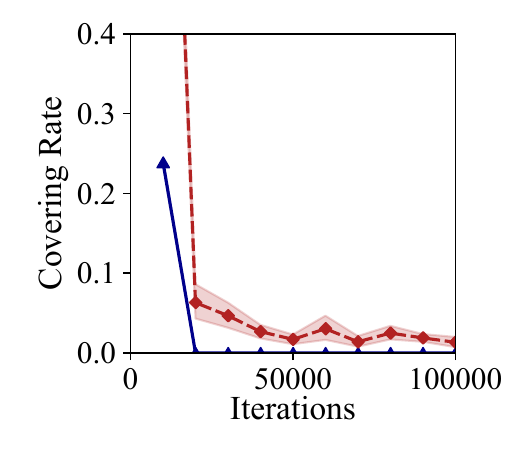}
    \label{fig:asn_covering kacs sup}
  }
  \subfloat[CCPP]{\includegraphics[width=0.24\textwidth]{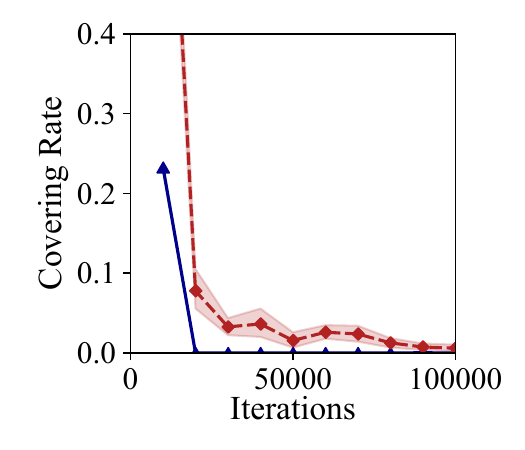} \label{fig:ccpp_covering kacs sup}}
  \subfloat[CS]{\includegraphics[width=0.24\textwidth]{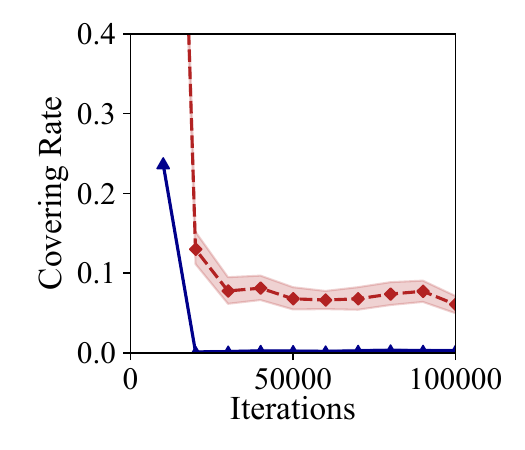} \label{fig:cs_covering kacs sup}}
  \subfloat[EEC]{\includegraphics[width=0.24\textwidth]{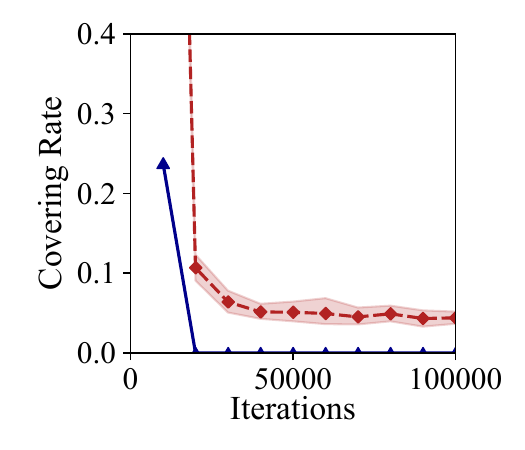} \label{fig:eec_covering kacs sup}}
  \caption{Covering-rate trajectories of the inner and outer KACS
    submodels during training on the eight benchmark problems. Each rate
    is normalized by the number of match-set evaluations for the
    corresponding submodel type. Curves show the mean over 30 runs, and
  shaded regions denote 95\% confidence intervals.}
  \label{fig:all_covering kacs sup}
\end{figure*}


\clearpage
\makeatletter
\let\addcontentsline\arxiv@mainaddcontentsline
\makeatother
\end{document}